%% file: camera_ready_paper.tex
\documentclass[twoside]{article}

\usepackage[accepted]{aistats2023_cameraready}
\usepackage[round]{natbib}

\usepackage{url}
\usepackage{booktabs}
\usepackage{multirow}

\input{macros.tex}

\usepackage{tikz}
\usetikzlibrary{shapes.geometric, arrows, positioning, bayesnet, fit}
\usepackage{subcaption}
\usepackage{adjustbox}

\usepackage{dsfont}
\usepackage{amssymb}
\usepackage{amsthm}
\newtheorem{proposition}[]{Proposition}
\newtheorem{problem}[]{Problem}
\newtheorem{definition}[]{Definition}

\newtheorem{lemma}[]{Lemma}
\usepackage{thm-restate}
\usepackage[shortlabels]{enumitem}

\usepackage[ruled, vlined, linesnumbered, noend]{algorithm2e}
\DontPrintSemicolon
\SetKwInput{KwInput}{Input}
\SetArgSty{textnormal}

\newcommand\restr[2]{{
  \left.\kern-\nulldelimiterspace 
  #1 
  \vphantom{|} 
  \right|_{#2} 
  }}

\usepackage{bm}

\usepackage{filecontents}
\usepackage{multibib}

\newcites{SM}{References}

\begin{document}

%

%

\twocolumn[

\aistatstitle{Compositional Probabilistic and Causal Inference using Tractable Circuit Models}

\aistatsauthor{ Benjie Wang \And Marta Kwiatkowska }

\aistatsaddress{ Department of Computer Science \\ University of Oxford \And Department of Computer Science \\ University of Oxford } ]

\begin{abstract}

   Probabilistic circuits (PCs) are a class of tractable probabilistic models, which admit efficient inference routines depending on their structural properties.
   In this paper, we introduce \textit{md-vtrees}, a novel structural formulation of (marginal) determinism in structured decomposable PCs, which generalizes previously proposed classes such as probabilistic sentential decision diagrams. Crucially, we show how md-vtrees can be used to derive tractability conditions and efficient algorithms for advanced inference queries expressed as \textit{arbitrary compositions} of basic probabilistic operations, such as marginalization, multiplication and reciprocals, in a sound and generalizable manner. In particular, we derive the first polytime algorithms for causal inference queries such as backdoor adjustment on PCs. As a practical instantiation of the framework, we propose MDNets, a novel PC architecture using md-vtrees, and empirically demonstrate their application to causal inference.

\end{abstract}

\input{sections/introduction.tex}

\input{sections/preliminaries.tex}

\input{sections/mdvtree.tex}

\input{sections/mdnet}

\input{sections/operations.tex}

\input{sections/applications}

\input{sections/experiments}

\input{sections/conclusion}

\input{sections/acknowledgements.tex}

\bibliographystyle{apalike}
\bibliography{citations.bib}

\newpage
\clearpage
\newpage

\appendix
\thispagestyle{empty}
\onecolumn

\input{sections/appendix/proofs}

\input{sections/appendix/mdcalculus}

\input{sections/appendix/causal}

\bibliographystyleSM{apalike}

\bibliographySM{citations.bib}

\end{document}

%% file: macros.tex
\newcommand{\edg}{\bm{V}}

\newcommand{\edgm}{V}

\newcommand{\softbold}{\bm}

\newcommand{\vars}{\softbold{V}}

\newcommand{\var}{V}

\newcommand{\exovars}{\softbold{U}}
\newcommand{\exovarsval}{\softbold{u}}

\newcommand{\exoparents}{\softbold{U}}
\newcommand{\edgparents}{Pa}
\newcommand{\dm}{d}

\newcommand{\scm}{\mathcal{M}}
\newcommand{\scmfn}{F}
\newcommand{\scmfns}{\softbold{F}}
\newcommand{\dist}{p}

\newcommand{\graph}{G}

\newcommand{\treat}{\bm{X}}
\newcommand{\treatval}{\bm{x}}
\newcommand{\treatm}{X}
\newcommand{\treatmval}{x}

\newcommand{\outcome}{\bm{Y}}
\newcommand{\outcomeval}{\bm{y}}
\newcommand{\outcomem}{Y}
\newcommand{\outcomemval}{y}

\newcommand{\adjust}{\bm{Z}}
\newcommand{\adjustval}{\bm{z}}
\newcommand{\adjustm}{Z}

\newcommand{\circuit}{C}
\newcommand{\scope}{\phi}

\newcommand{\pcparam}{\theta}

\newcommand{\pcvars}{\bm{V}}
\newcommand{\pcvarsval}{\bm{v}}
\newcommand{\pcvarsingle}{V}
\newcommand{\pcvarsingleval}{v}
\newcommand{\varsubset}{\bm{W}}
\newcommand{\varsubsetval}{\bm{w}}

\newcommand{\pcchild}{ch}

\newcommand{\sumunit}{T}
\newcommand{\sumunits}{\bm{T}}
\newcommand{\productunit}{P}
\newcommand{\productunits}{\bm{P}}
\newcommand{\leafunit}{L}

\newcommand{\genericunit}{N}

\newcommand{\rootunit}{R}

\newcommand{\margdetvars}{\bm{Q}}
\newcommand{\margdetvarsval}{\bm{q}}

\newcommand{\evidence}{\bm{E}}
\newcommand{\evidenceval}{\bm{e}}

\newcommand{\values}{val}

\newcommand{\pcfunc}{p}

\newcommand{\vtree}{v}
\newcommand{\vnodes}{M}
\newcommand{\vedges}{E}
\newcommand{\vlabel}{\psi}
\newcommand{\vnode}{m}

\newcommand{\vlabelno}{U}
\newcommand{\margdetset}{\mathcal{Q}}

\newcommand{\mdvtree}{w}

\newcommand{\vnodepa}{\vnode_{\text{pa}}}
\newcommand{\vnodech}{\vnode_{\text{ch}}}

\newcommand{\parents}{pa}
\newcommand{\children}{ch}

\newcommand{\circuitclass}{\mathcal{C}}

\newcommand{\marg}{\texttt{MARG}}
\newcommand{\inst}{\texttt{INST}}

\newcommand{\product}{\texttt{PROD}}

\newcommand{\logop}{\texttt{LOG}}
\newcommand{\maxop}{\texttt{MAX}}

\newcommand{\powop}{\texttt{POW}}
\newcommand{\power}{\alpha}
\newcommand{\p}{p}

\newcommand{\powerset}{\mathcal{P}}
\newcommand{\supp}{\textnormal{supp}}

\newcommand{\targetmdset}{\bm{S}}

\newcommand{\commonvars}{\bm{C}}
\newcommand{\newcommonvars}{\bm{C}} 
\newcommand{\commonvarsglobal}{\bm{C}_{\text{global}}}

\newcommand{\regionsize}{K}

\newcommand{\vparamfn}{\rho}

%% file: sections/introduction.tex
\section{INTRODUCTION}

Probabilistic circuits (PC) \citep{choi2020probcirc} are a broad family of tractable probabilistic models that are known for their ability to perform \textit{exact} and \textit{efficient} probabilistic inference. 
For example, in contrast to neural probabilistic models such as variational autoencoders (VAE) \citep{kingma2013auto}, generative adversarial networks (GAN) \citep{goodfellow2014gan}, and normalizing flows (NF) \citep{rezende2015nf}, linear-time exact algorithms on PCs are available for important inference tasks such as computing marginal probabilities, or the maximum probability assignment of variables, for certain types of PCs. Meanwhile, probabilistic circuit structures have recently been shown to scale to high-dimensional datasets such as CelebA \citep{peharz2020einet}.

A distinguishing feature of the PC framework is the ability to trade off \textit{expressive efficiency} for \textit{tractability} by imposing various properties on the PC. Broadly, these properties can be divided into \textit{scope properties}, such as decomposability and structured decomposability \citep{pipatsrisawat2008compilation}, and \textit{support properties}, such as determinism, strong determinism \citep{pipatsrisawat2010strdet}, and marginal determinism \citep{choi2020probcirc}. As we impose more properties on a PC, more inference tasks become tractable (i.e. computable in polynomial time), but we also lose some expressive efficiency and generality.

In this paper, we aim to extend the boundaries of advanced inference queries that can be tackled with PCs. In particular, we consider probabilistic inference queries specified as compositions of basic operations such as marginalization, products, and reciprocals, building upon the approach of \citet{vergari2021atlas}. We find that current PC classes are not sufficient to analyze \textit{arbitrary} compositions of these operations, and thus propose a novel class of circuits (md-vtrees \& MDNet) and accompanying rules (MD-calculus) to derive in a sound and generalizable manner tractability conditions and algorithms for such compositions. Exploiting this, we design the first efficient exact algorithms for causal inference on probabilistic circuits.  

Our first contribution is to introduce a unifying formulation of support properties in structured decomposable circuits using \textit{md-vtrees}. We show that md-vtrees generalize previously proposed PC families such as probabilistic sentential decision diagrams (PSDD) \citep{kisapsdd2014} and structured decomposable and deterministic circuits \citep{dang2020strudel, dimauro2021random}. Notably, we also show that PSDDs are not \textit{optimal} in that we can impose weaker support properties while maintaining tractability for the same inference tasks.

Next, as a practical instantiation of the framework, we propose MDNets, a novel architecture for PCs which can be easily configured to conform to any md-vtree. Crucially, this allows us to enforce arbitrary support properties needed for tractable inference, including those not covered by existing PC architectures. We show how to learn MDNets simply and efficiently using randomized structures and parameter learning \citep{peharz2020ratspn, peharz2020einet, dimauro2021random}.

Finally, for inference, we derive a set of rules for analyzing arbitrary compositions of basic operations using md-vtrees, which we call the MD-calculus. In particular, MD-calculus rules can be applied backward through a given composition, to derive sufficient conditions for tractability on the inputs to the query, which we can enforce during learning through our MDNets. As an application, we demonstrate how the MD-calculus can be applied to derive tractability conditions and algorithms for causal inference estimands on PCs, including the backdoor and frontdoor formulae and (an extension of) the napkin formula.

\subsection{Related Work}

\paragraph{Support Properties} 
The property of determinism was first introduced in the context of Boolean circuits, specifically, those in negation normal form \citep{darwiche2001deter, darwiche2002map}, before being naturally extended to arithmetic/probabilistic circuits \citep{darwiche2003differential}. Later, a stronger property known as \textit{strong determinism} was introduced \citep{pipatsrisawat2010strdet, darwiche2011sdd, kisapsdd2014} as a convenient means of enforcing determinism in structured decomposable circuits by tying determinism to the scope decomposition, resulting in the (probabilistic) sentential decision diagram (SDD). \citet{oztok2016sdp} further introduced the notion of \textit{constrained vtrees}, which restricts the structure (scopes, and thus support) of the SDD vtree in order to solve problems on weighted Boolean formulae. Finally, \citet{choi2020probcirc} recently introduced a more general support property called \textit{marginal determinism}, which applies to general probabilistic circuits and is not directly tied to the scope decomposition; our work shows how to construct marginal deterministic circuits, previously considered an intractable task \citep{choi2022mmapexact}. Marginal determinism is sufficient for tractability of some marginal MAP queries \citep{huang2006map}.

\paragraph{Causality and Probabilistic Circuits} The relationship between probabilistic circuits and causality has its roots in the seminal \textit{compilation} methods of \citet{darwiche2003differential}, which described an inference approach for (causal) Bayesian networks that involved compiling their graphs into tractable arithmetic circuits; subsequent work has further examined causality and compiled circuits \citep{butz2020decompilation, wang2021provable, darwiche2021causal, chen2021treewidth}. 
However, obtaining an exact causal interpretation of more general, learned probabilistic circuits has remained an open problem \citep{zhao2015spnbn, papantonis2020causal}. The only practical prior causal inference method for such circuits is the neural parameterization of \citet{zecevic2021interventional}, but this lacks exactness guarantees and is only applicable to fully observed settings. In contrast, we consider exact causal inference using do-calculus, where circuits encode the observed probability distribution.

%% file: sections/preliminaries.tex
\section{PRELIMINARIES}

\paragraph{Notation} We use uppercase to denote a random variable (e.g., $\pcvarsingle$) and lowercase for an instantiation of a variable (e.g., $\pcvarsingleval$). Sets of variables (and their instantiations) are denoted using bold font (e.g., $\pcvars, \pcvarsval$), and we use $\values$ for the set of all instantiations of a set of variables (e.g., $\values(\pcvars)$).

Probabilistic circuits (PC) \citep{choi2020probcirc} are computational graphs which encode a non-negative function over a set of variables; in particular, they are often used to model (possibly unnormalized) probability distributions. 

\begin{definition}[Probabilistic Circuit]

A circuit $\circuit$ over variables $\pcvars$ is a parameterized rooted graph, consisting of three types of nodes $\genericunit$: leaf $\leafunit$, sum $\sumunit$ and product $\productunit$. Leaf nodes $\leafunit$ are leaves of the graph, while each internal node (sum or product) $\genericunit$ has a set of children, denoted $\pcchild(\genericunit)$. Sum nodes have a parameter/weight $\pcparam_i \in \mathbb{R}^{\geq 0}$ associated with each of their children $\genericunit_i$.

Each leaf node $\leafunit$ encodes a non-negative function $\pcfunc_\leafunit: \scope(\leafunit) \to \mathbb{R}^{\geq 0}$ over a subset of variables $\scope(\leafunit) \subseteq \pcvars$, known as its \textnormal{scope}. The function encoded by each internal node $\genericunit$ is then given by:
\begin{equation*}
\pcfunc_\genericunit(\pcvars) := 
\begin{cases}
\prod_{\genericunit_i \in \pcchild(\genericunit)} \pcfunc_{\genericunit_i}(\pcvars) & \text{if $\genericunit$ is a product} \footnotemark\\
\sum_{\genericunit_i \in \pcchild(\genericunit)} \pcparam_i \pcfunc_{\genericunit_i}(\pcvars) & \text{if $\genericunit$ is a sum} \\
\end{cases}
\end{equation*}
The function encoded by the circuit, $\pcfunc_{\circuit}(\pcvars)$, is defined to be the function encoded by its root node $\rootunit$.
The \textit{size} of a circuit, denoted $|\circuit|$, is defined to be the number of edges in the circuit.
\end{definition}
\footnotetext{We assume in this paper that each product node has exactly two children; this does not lose generality as any product node can be converted into a sequence of binary products.}

\begin{definition}[Scope and support of PC node]
The scope of an internal node $\genericunit$ is the set of variables $\pcfunc_\genericunit$ specifies a function over, recursively defined by $\scope(\genericunit) := \bigcup_{\genericunit_i \in \pcchild(\genericunit)} \scope(\genericunit_i)$. The support of any node $\genericunit$ is the set of all instantiations of its scope s.t. $\pcfunc_\genericunit$ is positive, defined as $\supp(\genericunit) := \{\varsubsetval \in \values(\scope(\genericunit)): \pcfunc_\genericunit(\varsubsetval) > 0\}$.
\end{definition}

The tractability of probabilistic circuits depends on the scope and support properties they satisfy. A PC is \textit{decomposable} if the children of a product node have distinct scopes (and thus partition the scope of the product node), and is \textit{smooth} if the children of a sum node have the same scope. Decomposability and smoothness together enable tractable marginal inference; that is, for any subset $\varsubset \subseteq \scope(\genericunit)$ of the scope of a node $\genericunit$, we can compute $\pcfunc_\genericunit(\varsubset)$ efficiently, where $\pcfunc_\genericunit(\varsubset) := \sum_{\scope(\genericunit) \setminus \varsubset} \pcfunc_\genericunit(\scope(\genericunit))$ is the \textit{marginal} of the function. A stronger version of decomposability known as \textit{structured decomposability} \citep{pipatsrisawat2008compilation, kisapsdd2014} requires the scope of product nodes to decompose according to a \textit{vtree}.
Structured decomposability enables efficient computation of additional operations/queries, notably the product of two circuits respecting the same vtree \citep{pipatsrisawat2008compilation, shen2016operations}. 
As for support properties, a PC is \textit{deterministic} if, for every instantiation $\varsubsetval$ of the scope of a sum node, at most one of its children $\genericunit_i$ evaluates to a non-zero value $\pcfunc_{\genericunit_i}(\varsubsetval)$ (equivalently, the supports of the children are distinct). Determinism enables tractability of the MAP inference query, i.e. computing $\max_{\pcvars \setminus \evidence} \pcfunc_{\genericunit}(\pcvars \setminus \evidence, \evidenceval)$ for some instantiation $\evidenceval$ of a set of evidence variables $\evidence \subseteq \pcvars$.

%% file: sections/mdvtree.tex
\section{A UNIFYING FRAMEWORK FOR SUPPORT PROPERTIES IN STRUCTURED DECOMPOSABLE CIRCUITS} \label{sec:mdvtree}

In this section, we describe our md-vtree framework, which integrates support properties into the vtree formulation of structured decomposable circuits. Using this unifying perspective, we derive a trade-off between the generality of md-vtree circuit classes and tractability, and necessary conditions for optimality of this trade-off.

\subsection{Structured Decomposability}

The property of structured decomposability is defined with respect to a variable tree known as a \textit{vtree}. 

\begin{restatable}[Vtree]{definition}{vtreedef}
A vtree $\vtree = (\vnodes, \vedges)$ for a set of variables $\pcvars$ is a rooted binary tree with nodes $\vnodes$ and edges $\vedges$, whose leaves $\vnode$ each correspond to a subset $\scope(\vnode) \subseteq \pcvars$, such that the subsets for all leaves form a partition of $\pcvars$.
\end{restatable}

We define the \textit{scope} of a leaf $\vnode$ to be $\scope(\vnode)$, and the scope of any other node to be $\scope(\vnode) = \cup_{\vnode_i \in ch(\vnode)} \scope(\vnode_i)$. Further, we write $\vtree_{\vnode} = (\vnodes_{\vnode}, \vedges_{\vnode})$ to denote the vtree rooted at $\vnode$.

Intuitively, a vtree specifies how the scope of product nodes decompose in a circuit. However, our definition of structured decomposability differs from typical recursive definitions \citep{pipatsrisawat2008compilation,darwiche2011sdd,shih2019smooth} in that the key condition is on the scopes of the sum/leaf nodes, without directly placing conditions on the product nodes (besides decomposability): 

\begin{restatable}[PC respecting vtree]{definition}{respectdef}
Let $\circuit$ be a PC and $\vtree = (\vnodes, \vedges)$ be a vtree, both over variables $\pcvars$. We say that $\circuit$ respects $\vtree$ if (1) $\circuit$ is smooth\footnote{A structured decomposable circuit can be smoothed in near-linear time \citep{shih2019smooth}.} and decomposable; and (2) for every leaf node and non-trivial\footnote{A sum node is non-trivial iff it has more than one child.} sum node $\genericunit \in \circuit$, there exists a vtree node $\vnode \in \vnodes$ such that $\scope(\genericunit) = \scope(\vnode)$.
\end{restatable}

\begin{restatable}[Structured Decomposability]{definition}{strdecdef}
A PC $\circuit$ is structured decomposable if it respects some vtree $\vtree$.
\end{restatable}

Structured decomposability enables tractable products of circuits respecting compatible vtrees, i.e. those which have the same structure when projected onto their common variables $\commonvars := \pcvars^{(1)} \cap \pcvars^{(2)}$ (see Appendix for details). 
In the next subsection, we will show how support properties, which are also specified on the sum nodes in the circuit, can be neatly integrated into vtrees.

\subsection{Structured Marginal Determinism}

\input{diags/mdvtree_comparison}

Our new definition of structured decomposability based on sum nodes provides the basis for us to specify a novel systematic characterization of support properties in structured decomposable circuits, which we call \textit{structured marginal determinism}. First, we reformulate the definitions of marginal determinism from \citet{choi2020probcirc}. 

\begin{definition}[Restricted Scope]
    For a PC node $\genericunit$ (resp. vtree node $\vnode$), and given a set $\commonvars \subseteq \pcvars$, the restricted scope is defined as $\scope_{\commonvars}(\genericunit) := \scope(\genericunit) \cap \commonvars$ (resp. $\scope_{\commonvars}(\vnode) := \scope(\vnode) \cap \commonvars$).
\end{definition}

\begin{definition}[Marginalized Support]
    For any PC node $\genericunit$ and subset of variables $\margdetvars \subseteq \pcvars$, we define the marginalized support of $\genericunit$ with respect to $\margdetvars$ as 
    $\supp_{\margdetvars}(\genericunit) := \{\margdetvarsval \in \values(\margdetvars): \pcfunc_{\genericunit}(\margdetvarsval) > 0\}$.
\end{definition}

Note that $\margdetvars$ can contain variables outside of $\scope(\genericunit)$; in a slight abuse of notation, we write $\pcfunc_{\genericunit}(\margdetvarsval)$ for $\pcfunc_{\genericunit}(\margdetvarsval \cap \scope(\genericunit))$.

\begin{definition}[Marginal Determinism]
    A sum node is marginal deterministic with respect to a subset $\margdetvars \subseteq \pcvars$ (written $\margdetvars$-deterministic) if the children of the sum node have distinct marginalized support, i.e. $\supp_{\margdetvars}(\genericunit_i) \cap \supp_{\margdetvars}(\genericunit_j) = \emptyset$ for $\genericunit_i, \genericunit_j$ distinct children of $\sumunit$.
\end{definition}

\begin{definition}[Marginal Determinism of PC]
A PC is marginal deterministic with respect to a subset $\margdetvars \subseteq \pcvars$ (written $\margdetvars$-deterministic) if for every sum node $\sumunit$, either:
\begin{itemize}
    \item $\scope(\sumunit)$ does not overlap with $\margdetvars$, i.e. $\scope_{\margdetvars}(\sumunit) = \emptyset$; or
    \item The sum node $\sumunit$ is $\margdetvars$-deterministic.
\end{itemize}
\end{definition}

For example, normal determinism is equivalent to $\circuit$ being marginal deterministic with respect to $\pcvars$. In general, there is no straightforward relation between $\margdetvars$-determinism and $\margdetvars'$-determinism for different sets $\margdetvars, \margdetvars'$. In particular, 
neither determinism (i.e. $\pcvars$-determinism) nor $\margdetvars$-determinism imply each other in general; for example, a circuit can be $\margdetvars$-deterministic but not deterministic if there exist some sum nodes with $\scope_{\margdetvars}(\sumunit)= \emptyset$. Thus, we use $\margdetset(\circuit)$ to denote the set of all sets $\margdetvars \subseteq \pcvars$ such that $\circuit$ is $\margdetvars$-deterministic; this provides a characterization of the support properties of the circuit.

Now, for a given PC $\circuit$, and any sum node $\sumunit$ in that PC, let $\vlabel(\sumunit)$ be the set of all sets $\margdetvars$ such that $\sumunit$ is $\margdetvars$-deterministic; we call this a \textit{labelling function}. Note that the label function $\vlabel$ is a \textit{specification} of marginal determinism for the circuit; that is, it is sufficient to deduce $\margdetset(\circuit)$. We make two observations that allow us to simplify the labelling function, one straightforward, and one more subtle. Firstly, we note that $\margdetvars$-determinism for the circuit imposes the same requirement on all nodes with the same scope; thus we restrict $\vlabel$ to have the same value for all sum nodes with the same scope. For structured decomposable circuits, we can thus write $\vlabel(\vnode)$ as a function of the vtree node $\vnode$.

The second observation is that, under some assumptions, we can actually specify $\vlabel(\vnode)$ using a \textit{single} set $\margdetvars \subseteq \pcvars$.

\begin{restatable}[Conflicting $\margdetvars$-Determinisms for Sum Nodes]{proposition}{propsupport} \label{prop:support}
Let $\circuit$ be a PC, and let $\margdetvars, \margdetvars' \subseteq \pcvars$ such that neither is a subset of the other. Suppose that there exists a non-trivial sum node $\sumunit$ in $\circuit$ that is $\margdetvars$-deterministic and $\margdetvars'$-deterministic, but not $(\margdetvars \cap \margdetvars')$-deterministic. 
Then the circuit rooted at $\sumunit$, $\circuit_\sumunit$, cannot have full support.
\end{restatable}

Proposition \ref{prop:support} says that, if we want $\vlabel(\vnode)$ to contain two sets $\margdetvars, \margdetvars'$ which are not subsets of each other, then this necessarily restricts the support of the circuit. While it can be beneficial to enforce a restricted support on a PC if we have prior knowledge \citep{kisapsdd2014}, it is undesirable in our case where restricting support comes as a \textit{side effect} of enforcing tractability, as this can result in bias when learning. As such, we only consider labellings $\vlabel(\vnode)$ where, for every $\margdetvars, \margdetvars' \in \vlabel(\vnode)$, we have $\margdetvars \subseteq \margdetvars'$ or $\margdetvars' \subseteq \margdetvars$. 

\begin{restatable}[Superset $\margdetvars$-Determinisms for Sum Nodes]{proposition}{propsubset}\label{prop:supset}
Suppose that a sum node $\sumunit$ is $\margdetvars$-deterministic. Then it is also $\margdetvars'$-deterministic for any $\margdetvars \subseteq \margdetvars' \subseteq \pcvars$.
\end{restatable}

Using Proposition \ref{prop:supset}, it now follows that $\vlabel(\vnode)$ must take the form $\{\margdetvars' | \margdetvars \subseteq \margdetvars' \subseteq \pcvars\}$ for some $\margdetvars$. As a result, we can just label our vtree node $\vnode$ with $\margdetvars$, i.e. $\vlabel(\vnode) = \margdetvars$. 

This motivates our characterization of structural marginal determinism based on the concept of a \textit{md-vtree}, which provides a means of specifying the support properties that a structured decomposable circuit satisfies.

\begin{definition} [md-vtree]
A md-vtree $\mdvtree = (\vtree, \vlabel)$ for a set of variables $\pcvars$ consists of a vtree $\vtree = (\vnodes, \vedges)$ over $\pcvars$, together with a labelling function $\vlabel$.

The labelling function maps a vtree node $\vnode \in \vnodes$ to some element in $\powerset(\scope(\vnode)) \cup \{\vlabelno\}$, where $\vlabelno$ is the \textit{universal set}.\footnote{The universal set satisfies, for any set $S$, $\vlabelno \supseteq S$, $\vlabelno \not\subseteq S$ (unless $S$ is $\vlabelno$), $\vlabelno \cap S = S$, and $\vlabelno \cup S = \vlabelno$.}
\end{definition}

\begin{definition} [PC respecting md-vtree]
Let $\circuit$ be a PC and $\mdvtree = (\vtree, \vlabel)$ be a md-vtree, both over variables $\pcvars$. Then we say that $\circuit$ respects $\mdvtree$ if 1) $\circuit$ respects $\vtree$; and 2) for any sum unit $\sumunit \in \circuit$, $\sumunit$ is marginally deterministic with respect to $\vlabel(\vnode)$, where $\vnode$ is the vtree node such that $\scope(\sumunit) = \scope(\vnode)$.
    
We denote the class of circuits respecting $\mdvtree$ by $\circuitclass_{\mdvtree}$.
\end{definition}

Intuitively, md-vtrees capture both structured decomposability, as well as a marginal determinism ``pattern'' that the circuit must follow:

\begin{definition} [Implied $\margdetvars$-Determinisms]
For any set $\margdetvars \subseteq \pcvars$, we say that $\margdetvars$-determinism is implied by a md-vtree $\mdvtree$ if, for every vtree node $\vnode \in \vnodes$ such that $\scope(\vnode) \cap \margdetvars \neq \emptyset$, it is the case that $\margdetvars \supseteq \vlabel(\vnode)$. We write $\margdetset(\mdvtree)$ to denote the set of all sets $\margdetvars$ s.t. $\margdetvars$-determinism is implied by $\mdvtree$. 
\end{definition}

\begin{restatable}[Validity of Implied $\margdetvars$-Determinisms]{proposition}{propimplied} 
For any PC $\circuit$ respecting md-vtree $\mdvtree$, both over $\pcvars$, and any $\margdetvars \subseteq \pcvars$ s.t. $\mdvtree$ implies $\margdetvars$-determinism, it follows that $\circuit$ is $\margdetvars$-deterministic. 
\end{restatable}

Notice that there is a trade-off between generality (expressivity) of the PC class, and the support properties it supports. Increasing the size of the labelling sets will improve the former, but hurt the latter. 

\begin{restatable}[Generality-Tractability Tradeoff]{theorem}{propstrength} \label{prop:strength}
Let $\mdvtree = (\vtree, \vlabel)$ and $\mdvtree' = (\vtree, \vlabel')$ be two md-vtrees, such that $\vlabel'(\vnode) \supseteq \vlabel(\vnode)$ for all $\vnode \in \vnodes$. Then we have that $\margdetset(\mdvtree) \supseteq \margdetset(\mdvtree')$, and $\circuitclass_{\mdvtree} \subseteq \circuitclass_{\mdvtree'}$.
\end{restatable}

It is worth commenting on the two extremes of possible labels; namely, the universal set, and the empty set. The role of the universal set label $\vlabelno$ is to indicate that no $\margdetvars$-determinism properties hold (including normal determinism). On the other hand, the empty set indicates that any sum node $\sumunit$ corresponding to $\vnode$ can only have one child $\genericunit_i$ which is not zero, i.e. $\pcfunc_{\genericunit_i} \equiv 0$; thus the sum node must be trivial. In practice, this means that it must represent a factorized distribution with factors corresponding to the scopes of the children of $\vnode$.

\paragraph{Examples} To the best of our knowledge, the only concrete type of probabilistic circuit proposed in the literature that imposes non-trivial marginal determinism constraints (i.e., not just normal determinism) is the probabilistic sentential decision diagram (PSDD) \citep{kisapsdd2014}.
PSDDs satisfy structured decomposability and a property known as strong determinism. In the language of md-vtrees, this corresponds to requiring that, for any non-leaf vtree node $\vnode$, and children $\vnode_1, \vnode_2$ of $\vnode$, the label $\vlabel(\vnode)$ is either $\scope(\vnode_1)$ or $\scope(\vnode_2)$ (which is then referred to as the \emph{left} child). For example, for the vtree over $\pcvars = \{\pcvarsingle_1, \pcvarsingle_2, \pcvarsingle_3, \pcvarsingle_4\}$, shown in Figure \ref{fig:psdd_example_scope} together with the scopes for each vtree node, the label function $\vlabel^{(psdd)}$ is given on the right in Figure \ref{fig:psdd_example_label}. 

Despite implementing strong determinism, recent work has shown that almost all of the tractable queries and operations that PSDDs support require only structured decomposability and determinism \citep{dang2020strudel}. This raises the question of whether strong determinism adds anything. To analyse this, in Figure \ref{fig:psdd_example_det} we show the labelling function $\vlabel^{(\text{det})}$ which defines a deterministic circuit.
With these representations, we can deduce the $\margdetvars$-determinism properties that any PSDD, or structured decomposable and deterministic circuit, must satisfy, by finding the set $\margdetset(\mdvtree)$ of sets $\margdetvars$ which its md-vtree implies. In this example, by enumerating all sets $\margdetvars \subseteq \pcvars$ and checking the condition, we can see that $\margdetset(\mdvtree^{\text{psdd}}) = \{\{\pcvarsingle_1, \pcvarsingle_2\}, \{\pcvarsingle_1, \pcvarsingle_2, \pcvarsingle_3\}, \{\pcvarsingle_1, \pcvarsingle_2, \pcvarsingle_3, \pcvarsingle_4\}\}$, while $\margdetset(\vtree^{\text{det}}) = \{\{\pcvarsingle_1, \pcvarsingle_2, \pcvarsingle_3, \pcvarsingle_4\}\}$. This shows that PSDDs do have additional $\margdetvars$-determinisms, which means, for example, that they are more tractable with regards to MMAP queries. In fact, in the particular case of (P)SDDs, the implied  $\margdetvars$-determinisms $\margdetset(\mdvtree^{\text{psdd}})$ coincide with the definition of $\margdetvars$-constrained vtrees \citep{oztok2016sdp}.

\subsection{Regular md-vtrees}

Given the trade-off between expressivity and tractability for md-vtrees, one might ask how to choose the labelling function in practice. One reasonable strategy would be to enforce some marginal determinism properties that we require for tractability of some inference task, and optimize for expressivity within this constraint. 

\begin{problem}[Labelling Selection]
Given a vtree $\vtree$, and a set $\targetmdset$ of subsets $\margdetvars \subseteq \pcvars$, choose a labelling function $\vlabel$ such that $\mdvtree = (\vtree, \vlabel)$ implies $\margdetvars$-determinism for all $\margdetvars \in \targetmdset$, i.e. $\margdetset(\mdvtree) \supseteq \targetmdset$, while maximizing expressivity. 
\end{problem}

\input{algs/optimal_labelling}

We propose a simple algorithm to tackle this problem (Algorithm \ref{alg:optlabel}), which directly enforces the necessary labels for each vtree nodes; it can be seen that it is optimally expressive, in the sense that increasing the size of any labelling set will result in some losing $\margdetvars$-determinism for some $\margdetvars \in \targetmdset$. It turns out that such labelling functions have a very specific structure:

\begin{restatable}[Regular md-vtree]{definition}{defnregular}\label{defn:regular}
We say that a md-vtree $\mdvtree = (\vtree, \vlabel)$ is regular if for every non-leaf node $\vnode$, and its children $\vnode_1, \vnode_2$, it holds that either $\vlabel(\vnode) = \vlabel(\vnode_1)$, $\vlabel(\vnode) = \vlabel(\vnode_2)$, or $\vlabel(\vnode) = \vlabel(\vnode_1) \cup \vlabel(\vnode_2)$.
\end{restatable}

\begin{restatable}[Regularity of Algorithm \ref{alg:optlabel}]{proposition}{propalg}
The output of Algorithm \ref{alg:optlabel} is a regular md-vtree.
\end{restatable}

The following Theorem shows that regular md-vtrees are optimal in the sense that for any given marginal determinism requirement $\targetmdset$, (one of) the most expressive md-vtree is always a regular md-vtree. This means that we can restrict our attention to the much smaller space of regular md-vtrees. In particular, the labelling function of a regular md-vtree is entirely determined by the labelling of each leaf node $\vlabel(\vnode_{\textnormal{leaf}})$, and a ternary variable over values $\{f, s, b\}$ for each non-leaf node, indicating whether the label depends on the label of the first child, second child, or both.

\begin{restatable}[Optimal md-vtrees]{theorem}{thmoptimal} \label{thm:regularoptimal}
Let $\mdvtree = (\vtree, \vlabel)$ be any md-vtree. Then there exists a regular md-vtree $\mdvtree' = (\vtree, \vlabel')$ such that $\margdetset(\mdvtree) = \margdetset(\mdvtree')$, and $\circuitclass_{\mdvtree'} \supseteq \circuitclass_{\mdvtree}$.
\end{restatable}

While Algorithm 1 always returns an optimal labelling for a given vtree satisfying the required marginal determinisms, the expressivity of the circuit class may differ depending on the vtree. For example, for some vtrees, Algorithm 1 may output a labelling function such that $\vlabel(\vnode)$ is empty for some vtree nodes $\vnode$, i.e. a factorized distribution. We leave designing optimally expressive vtrees for a given set $\targetmdset$ of marginal determinisms as an open problem.  

\paragraph{Examples} Let us return to the PSDD example from Figure \ref{fig:psdd_example}. It can be easily checked that the corresponding md-vtree $\mdvtree^{\text{psdd}}$ in Figure \ref{fig:psdd_example_label} is not regular. Following Algorithm \ref{alg:optlabel}, we therefore construct in Figure \ref{fig:psdd_example2_optimal} an regular md-vtree $\mdvtree^{\text{opt}}$ that retains the same $\margdetset(\mdvtree^{\text{opt}}) = \margdetset(\mdvtree^{\text{psdd}}) = \{\{\pcvarsingle_1, \pcvarsingle_2\}, \{\pcvarsingle_1, \pcvarsingle_2, \pcvarsingle_3\}, \{\pcvarsingle_1, \pcvarsingle_2, \pcvarsingle_3, \pcvarsingle_4\}\}$. 
We should prefer $\mdvtree^{\text{opt}}$ over $\mdvtree^{\text{psdd}}$ as it imposes less constraints/more circuits respect $\mdvtree^{\text{opt}}$. In other words, PSDDs impose more constraints than they ``need to'' to obtain their marginal determinism properties.

%% file: diags/mdvtree_comparison.tex
\begin{figure*}[t]
    \centering
    \begin{subfigure}{0.24\linewidth}
        \centering
        \scalebox{0.6}{
        \input{diags/psdd_scope}
        }
        \caption{vtree with scope function $\scope$}
        \label{fig:psdd_example_scope}
    \end{subfigure}
    \begin{subfigure}{0.24\linewidth}
        \centering
        \scalebox{0.6}{
        \input{diags/psdd_mdvtree}
        }
        \caption{vtree with label $\vlabel^{(psdd)}$}
        \label{fig:psdd_example_label}
    \end{subfigure}
    \begin{subfigure}{0.24\textwidth}
        \centering
        \scalebox{0.6}{
        \input{diags/psdd_labeldet}
        }
        \caption{vtree with label $\vlabel^{(det)}$}
        \label{fig:psdd_example_det}
    \end{subfigure}
    \begin{subfigure}{0.24\textwidth}
        \centering
        \scalebox{0.6}{
        \input{diags/psdd_labeloptimal}
        }
        \caption{Optimal labels for PSDD}
        \label{fig:psdd_example2_optimal}
    \end{subfigure}
    \caption{Example of md-vtree with scope function, and three different labelling functions.}

    \label{fig:psdd_example}
\end{figure*}
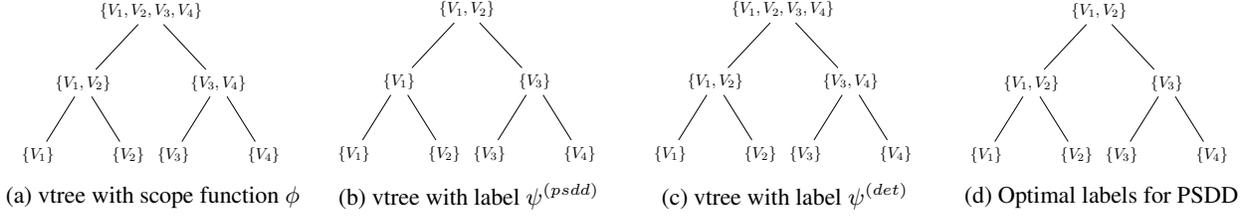

%% file: diags/psdd_scope.tex
\begin{tikzpicture}

\node[] (1234) {$\{\pcvarsingle_1, \pcvarsingle_2, \pcvarsingle_3, \pcvarsingle_4\}$};
\node[below= of 1234, xshift=-1.5cm] (12) {$\{\pcvarsingle_1, \pcvarsingle_2\}$};
\node[below= of 1234, xshift=1.5cm] (34) {$\{\pcvarsingle_3, \pcvarsingle_4\}$};
\node[below= of 12, xshift=-1cm] (1) {$\{\pcvarsingle_1\}$};
\node[below= of 12, xshift=1cm] (2) {$\{\pcvarsingle_2\}$};
\node[below= of 34, xshift=-1cm] (3) {$\{\pcvarsingle_3\}$};
\node[below= of 34, xshift=1cm] (4) {$\{\pcvarsingle_4\}$};

\draw[-] (1234) -- (12);
\draw[-] (1234) -- (34);
\draw[-] (12) -- (1);
\draw[-] (12) -- (2);
\draw[-] (34) -- (3);
\draw[-] (34) -- (4);

\end{tikzpicture}

%% file: diags/psdd_mdvtree.tex
\begin{tikzpicture}

\node[] (1234) {$\{\pcvarsingle_1, \pcvarsingle_2\}$};
\node[below= of 1234, xshift=-1.5cm] (12) {$\{\pcvarsingle_1\}$};
\node[below= of 1234, xshift=1.5cm] (34) {$\{\pcvarsingle_3\}$};
\node[below= of 12, xshift=-1cm] (1) {$\{\pcvarsingle_1\}$};
\node[below= of 12, xshift=1cm] (2) {$\{\pcvarsingle_2\}$};
\node[below= of 34, xshift=-1cm] (3) {$\{\pcvarsingle_3\}$};
\node[below= of 34, xshift=1cm] (4) {$\{\pcvarsingle_4\}$};

\draw[-] (1234) -- (12);
\draw[-] (1234) -- (34);
\draw[-] (12) -- (1);
\draw[-] (12) -- (2);
\draw[-] (34) -- (3);
\draw[-] (34) -- (4);

\end{tikzpicture}

%% file: diags/psdd_labeldet.tex
\begin{tikzpicture}

\node[] (1234) {$\{\pcvarsingle_1, \pcvarsingle_2, \pcvarsingle_3, \pcvarsingle_4\}$};
\node[below= of 1234, xshift=-1.5cm] (12) {$\{\pcvarsingle_1, \pcvarsingle_2\}$};
\node[below= of 1234, xshift=1.5cm] (34) {$\{\pcvarsingle_3, \pcvarsingle_4\}$};
\node[below= of 12, xshift=-1cm] (1) {$\{\pcvarsingle_1\}$};
\node[below= of 12, xshift=1cm] (2) {$\{\pcvarsingle_2\}$};
\node[below= of 34, xshift=-1cm] (3) {$\{\pcvarsingle_3\}$};
\node[below= of 34, xshift=1cm] (4) {$\{\pcvarsingle_4\}$};

\draw[-] (1234) -- (12);
\draw[-] (1234) -- (34);
\draw[-] (12) -- (1);
\draw[-] (12) -- (2);
\draw[-] (34) -- (3);
\draw[-] (34) -- (4);

\end{tikzpicture}

%% file: diags/psdd_labeloptimal.tex
\begin{tikzpicture}

\node[] (1234) {$\{\pcvarsingle_1, \pcvarsingle_2\}$};
\node[below= of 1234, xshift=-1.5cm] (12) {$\{\pcvarsingle_1, \pcvarsingle_2\}$};
\node[below= of 1234, xshift=1.5cm] (34) {$\{\pcvarsingle_3\}$};
\node[below= of 12, xshift=-1cm] (1) {$\{\pcvarsingle_1\}$};
\node[below= of 12, xshift=1cm] (2) {$\{\pcvarsingle_2\}$};
\node[below= of 34, xshift=-1cm] (3) {$\{\pcvarsingle_3\}$};
\node[below= of 34, xshift=1cm] (4) {$\{\pcvarsingle_4\}$};

\draw[-] (1234) -- (12);
\draw[-] (1234) -- (34);
\draw[-] (12) -- (1);
\draw[-] (12) -- (2);
\draw[-] (34) -- (3);
\draw[-] (34) -- (4);

\end{tikzpicture}

%% file: algs/optimal_labelling.tex
\begin{algorithm}[t]
\SetAlgoLined
\KwInput{vtree $\vtree = (\vnodes, \vedges)$, required set of marginal determinisms $\targetmdset$
}
\KwResult{labelling function $\vlabel$ for $\vtree$}

\For{$\vnode \in \vnodes$} { 
    $\vlabel(\vnode) \gets \vlabelno$
    
    \For{$\margdetvars \in \targetmdset$} { \label{algline:qloopbegin}
        \If{$\margdetvars \cap \scope(\vnode) \neq \emptyset$} {
            $\vlabel(\vnode) \gets \vlabel(\vnode) \cap \margdetvars$ \label{algline:qloopend}
        }
    }
    \If{$\vlabel(\vnode) \neq \vlabelno$} { \label{algline:edgecasebegin}
        $\vlabel(\vnode) \gets \vlabel(\vnode) \cap \scope(\vnode)$ \label{algline:edgecaseend}
    }
}

\textbf{Return} $\mdvtree = (\vtree, \vlabel)$

\caption{Optimal Labelling}
\label{alg:optlabel}
\end{algorithm}

%% file: sections/mdnet.tex
\section{MDNETS: A PRACTICAL ARCHITECTURE FOR MD-VTREES}

In this section, we show how to construct and learn a probabilistic circuit that respects a particular md-vtree. It is worth noting that, as special cases of md-vtrees, we can use existing architectures and learning algorithms for PCs such as PSDDs \citep{liang2017learnpsdd} and structured decomposable and deterministic circuits \citep{dang2020strudel, dimauro2021random}. However, we have seen that PSDDs are not optimally expressive, and to enforce tractability, we may need to target md-vtrees which do not fall into these existing categories, such as those generated from Algorithm \ref{alg:optlabel}. We thus propose a novel PC architecture, MDNet, which enforces a given regular md-vtree \textit{by design}.

The key component of MDNets is the \textit{node group}, which is a vector of sum nodes with the same scope (i.e. corresponding to the same vtree node $\vnode$) with the property that the nodes in the group have disjoint marginalized support $\supp_{\vlabel(\vnode)}(\genericunit)$.
Intuitively, the sum nodes in a group provide a partition of the domain of $\vlabel(\vnode)$, which we use as an invariant in order to enforce the required marginal determinisms throughout the circuit.
More formally, suppose that we have a non-leaf vtree node $\vnode$, and let $\vnode_l, \vnode_r$ be its children. We refer to all sum nodes corresponding to a vtree node as being a \textit{layer}, and assign $G$ groups $\sumunits_1, ... \sumunits_G$ to the layer for vtree node $\vnode$, and similarly $\sumunits_1^{(l)}, ... \sumunits^{(l)}_{G_l}$, $\sumunits^{(r)}_1, ... \sumunits^{(r)}_{G_r}$ to the layers for $\vnode_l, \vnode_r$. For regular md-vtrees, the label $\vlabel(\vnode)$ is either equal to $\vlabel(\vnode_l)$ or $\vlabel(\vnode_r)$, or is their union. We handle these cases separately, as \textit{mixing} and \textit{synthesizing} layers.

\paragraph{Mixing Layer} If $\vlabel(\vnode) = \vlabel(\vnode_l)$ or $\vlabel(\vnode_r)$, we implement a \textit{mixing} layer. W.l.o.g. we assume $\vlabel(\vnode) = \vlabel(\vnode_l)$. For each group $\sumunits_i = (\sumunit_{i, 1}, ..., \sumunit_{i, k})$, and for each sum node $\sumunit_{i, k}$ in the group, we assign a set of product nodes $\productunits_{i, k}$ to $\sumunit_{i, k}$. Each product node has two children; the left child being a node from $\sumunits_1^{(l)}, ... \sumunits^{(l)}_{G_l}$ and the right child being 
a node from $\sumunits^{(r)}_1, ... \sumunits^{(r)}_{G_r}$. We place the following restriction: the left children of the product nodes $\cup_k \productunits_{i, k}$ must all be distinct, and all come from a single group. This ensures that all of the sum nodes are $\vlabel(\vnode) = \vlabel(\vnode_l)$-deterministic, and further that each group $\sumunits_i$ satisfies the invariant.

\paragraph{Synthesizing Layer} If $\vlabel(\vnode) = \vlabel(\vnode_l) \cup \vlabel(\vnode_r)$, we assign product nodes $\productunits_{i, k}$ to each sum node $\sumunit_{i, k}$ as before, but with a different restriction; we now require that both the left children and right children of the product nodes $\cup_k \productunits_{i, k}$ are nodes coming from a single group from their respective layers, and that each product node has a unique combination of children from these groups. 

Intuitively, mixing layers achieve their marginal determinism by ``copying'' the marginal determinism of one of their child layers, while mixing over groups in the other child layer. On the other hand, synthesizing layers enforce marginal determinism by combining, or synthesizing, the marginal determinism properties of both of their children. 

For simplicity, we propose to learn MDNets exploiting recent advancements in random structures for PC learning \citep{peharz2020einet, peharz2020ratspn, dimauro2021random}: in particular, we propose to choose the MDNet structure randomly within the constraints, and then learn the parameters using standard MLE estimation if the md-vtree implies ($\pcvars$-)determinism, or use EM otherwise \citep{peharz2015foundations}.

%% file: sections/operations.tex
\section{COMPOSITIONAL INFERENCE USING STRUCTURED MARGINAL DETERMINISM} \label{sec:inference}

In this section, we will describe a methodology that exploits our md-vtree framework as a language for deriving tractability conditions for arbitrary compositions of basic operations on probabilistic circuits. In particular, we build upon the work of \citet{vergari2021atlas}, showing how to extend their analysis to compositional queries which include marginalization (integration) operations at arbitrary points in the pipeline, and allow for maximization queries (e.g. marginal MAP).

\subsection{Support Properties in Compositional Inference}

\input{figs/ops_table}

In Table \ref{tbl:basic_ops}, we define the basic probabilistic inference operations, including marginalization, products, instantiation, powers, maximization, and logarithms, along with the properties (\textit{requirements}) under which there exist efficient (polytime) algorithms for computing them on PCs \citep{vergari2021atlas}; note that we assume decomposability and smoothness by default. These operations produce a circuit encoding the specified function (or scalar in the case of $\maxop$). We use $\restr{}{\supp(\circuit)}$ to denote the \textit{restricted} power/logarithm, as these functions are not defined at $0$:
\begin{equation*}
\restr{f(\pcfunc_{\circuit}(\pcvars))}{\supp(\circuit)} := 
\begin{cases}
    f(\pcfunc_{\circuit}(\pcvars)) & \textnormal{if } \pcfunc_{\circuit}(\pcvars) > 0 \\
    0 & \textnormal{otherwise} \\
\end{cases}
\end{equation*}

We refer to the basic operations in the bottom half of the table as \textit{deterministic} operations, as their tractability depends on the input circuit being deterministic; they are NP-complete ($\maxop$) or \#P-hard ($\powop, \logop$) otherwise \citep{choi2017determinism, vergari2021atlas}. 

Many complex inference queries can be expressed as compositions of these basic operations; we show a selection of examples in Table \ref{tbl:examples}. We can interpret such compositions as  \textit{pipelines}, or computational graphs, which specify an algorithm for computing the query that uses the efficient algorithms for each basic operation. To show tractability of a pipeline for given input circuits, one needs to show that the inputs to intermediate operations satisfy the requirement for tractability of that operation. For this purpose, \citet{vergari2021atlas} derive \textit{input-output conditions} for basic operations, which specify a pair of properties such that the output of the operation is guaranteed to satisfy the output property if the input satisfies the input property. This allows properties to be soundly propagated through the pipeline, from the input circuits.

However, there remains a significant unresolved challenge for analyzing general compositions of basic operations; namely, analyzing how operations affect \textit{support properties} of the circuit beyond just determinism.
For example, the marginal MAP problem (MMAP) \citep{huang2006map, choi2020probcirc} in Table \ref{tbl:examples} is a canonical inference task that can be decomposed into a composition of a $\marg$ operation $\sum_{\pcvars \setminus \varsubset} \pcfunc_\circuit(\varsubset, \pcvars \setminus \varsubset)$, and a $\maxop$ operation $\max_{\varsubsetval} \pcfunc_{\circuit}(\varsubset)$. The $\maxop$ operation $\max_{\varsubsetval} \pcfunc_\circuit(\varsubset)$ is known to be tractable for any deterministic input circuit. Unfortunately, however, ensuring that the output of $\marg(\circuit; \pcvars \setminus \varsubset)$ is deterministic is known to be NP-hard \citep{shen2016operations}, even if $\circuit$ is deterministic. For similar reasons, any compositional inference task in which a deterministic operation appears after a marginalization operation cannot be currently analyzed, notable examples of which we show in the lower half of Table \ref{tbl:examples}.

\input{figs/queries_table}

\subsection{Operations on md-vtrees} \label{sec:forward}

To tackle these challenges, we apply our md-vtree framework as a unified language for scope and support in structured decomposable circuits. The first problem that we would like to address is that of determining whether a pipeline is tractable for given input md-vtrees:

\begin{problem}[Forward Problem]
Given a query expressed as a pipeline of basic operations, and md-vtree(s) that the input circuits respect, determine if the pipeline is tractable.
\end{problem}

To solve this problem, we propose to derive input-output conditions in terms of md-vtrees. In other words, if we have input circuit(s) that respect some given md-vtree(s), can we obtain a md-vtree that the output of a basic operation is guaranteed to respect? 
In the Appendix, we detail a set of algorithms for each of the operations in Table \ref{tbl:basic_ops}, which take as input md-vtree(s) and return an output md-vtree that provides exactly such a guarantee, based upon the corresponding algorithms on circuits. Then, given any pipeline, and input md-vtree(s), we can determine if the compositional query is tractable, simply by propagating md-vtree(s) \textit{forward} through the pipeline, and checking that the input md-vtree(s) to any intermediate operation satisfy the requirements in Table \ref{tbl:basic_ops}. Importantly, this can be done \textit{without doing the computation of the query pipeline itself}; all of our algorithms run in polytime in the number of variables $|\pcvars|$, which is much smaller than the circuits themselves.

\subsection{The MD-calculus} \label{sec:backward}

\input{figs/md_table}

The forward problem allows us to reason about tractability if we already have the input circuits. However, when learning  circuits from data, we have the freedom to choose the md-vtree(s) in order to enable tractable inference:

\begin{problem}[Backward Problem]
Given any query expressed as a pipeline of basic operations, derive input md-vtree(s) such that the pipeline is tractable.
\end{problem}

To this end, in Table \ref{tbl:mdcalc} we show a set of input-output conditions called the \textit{MD-calculus}. These are sufficient conditions on the input(s) to an operation to guarantee that the output is $\margdetvars$-deterministic. The MD-calculus forms a set of rules that we can apply backwards from deterministic operations (which require $\pcvars$-determinism), in order to determine a sufficient set of marginal determinisms $\targetmdset$ for each intermediate circuit. Finally, we can enforce those marginal determinisms on the input md-vtree(s) using Algorithm \ref{alg:optlabel}. We defer proofs of the MD-calculus rules to the Appendix.

\begin{restatable}[MD-calculus]{theorem}{thmmdcalc}
The conditions in Table \ref{tbl:mdcalc} hold.
\end{restatable}

We show examples of compositional queries in the bottom half of Table \ref{tbl:examples}, with the marginal determinism condition on the input circuit $\circuit$ derived using this approach. For MMAP, for the $\maxop$ operation, we require $\marg(\circuit; \pcvars \setminus \varsubset)$ to be $\varsubset$-deterministic. Using the MD-calculus rule for $\marg$, we can see that it is sufficient for $\circuit$ to also be $\varsubset$-deterministic. For mutual information, applying a similar approach we obtain that $\circuit$ should be both $\bm{X}$-deterministic and $\bm{Y}$-deterministic, but we have seen in Proposition \ref{prop:support} that this is not possible without restricting support.

%% file: figs/ops_table.tex
\begin{table}[]
\centering
\begin{tabular}{@{}lllll@{}}
\toprule
\textbf{Operation} & \textbf{Requirements} & \textbf{Output Encodes} \\ \midrule
     $\marg(\circuit; \varsubset)  $  & -       &      $\sum_{\varsubset}\pcfunc_\circuit(\pcvars)$    \\
     $\inst(\circuit; \varsubsetval) $    &      -          &   $ \pcfunc_\circuit(\varsubsetval, \pcvars \setminus \varsubset)$        \\ \midrule
    $\product(\circuit_1, \circuit_2)$   & Cmp. Vtrees       &  $\pcfunc_{\circuit_1}(\pcvars) \times \pcfunc_{\circuit_2}(\pcvars)$       \\  
          \midrule
        $\powop(\circuit; \power)$         & Det            &       $\restr{\pcfunc_{\circuit}(\pcvars)^\power}{\supp(\circuit)}$      \\
        $\maxop(\circuit)$            & Det           &       $\max_{\pcvars}\pcfunc_{\circuit}(\pcvars)$        \\
        $\logop(\circuit)$             & Det     &     $ \restr{\log \pcfunc_{\circuit}(\pcvars)}{\supp(\circuit)}$   \\ \bottomrule
\end{tabular} 
\caption{Definitions of basic operations.}
\label{tbl:basic_ops}
\end{table}

%% file: figs/queries_table.tex
\begin{table*}[t]
\centering
\begin{tabular}{@{}lllll@{}}
\toprule
\textbf{Task} & \textbf{Computation} & \textbf{Operations} & \textbf{Condition} \\ \midrule
      Cross-Entropy     &  $- \sum_{\pcvars}\pcfunc_{\circuit^{(1)}}(\pcvars) \log (\pcfunc_{\circuit^{(2)}}(\pcvars)) $                    &    $\marg, \product, \logop$          & $\circuit^{(1)}, \circuit^{(2)}$ Cmp.; $\circuit^{(2)}$ Det.  \\
      Var. Elim.          &   $\sum_{\varsubset} \pcfunc_{\circuit^{(1)}}(\pcvars) \pcfunc_{\circuit^{(2)}}(\pcvars)$                &    $\marg, \product$          & $\circuit^{(1)}, \circuit^{(2)}$ Cmp.          \\ \midrule
      MMAP          &      $\max_{\varsubsetval} \sum_{\pcvars \setminus \varsubset}\pcfunc_{\circuit}(\varsubsetval, \pcvars \setminus \varsubset)$                 &     $\marg, \maxop$         &  $\circuit$ Mdet. wrt. $\varsubset$       \\ 
      Mut. Inf. & $\pcfunc_{\circuit}(\bm{X}, \bm{Y}) \log \frac{\pcfunc_{\circuit}(\bm{X}, \bm{Y})}{\pcfunc_{\circuit}(\bm{X}) \pcfunc_{\circuit}(\bm{Y})}$ & $\marg, \product, \logop, \powop$ & - \\
      BD Adj.         &   $\sum_{\adjust} \pcfunc_{\circuit}(\outcome|\treat, \adjust) \pcfunc_{\circuit}(\adjust)$                   &    $\marg, \product, \powop$          &  $\circuit$ Mdet. wrt. $\treat \cup \adjust$; Str. Dec.      \\ 
      \bottomrule
\end{tabular}
\caption{Examples of complex inference tasks expressed as compositions of basic operations.}
\label{tbl:examples}
\end{table*}

%% file: figs/md_table.tex
\begin{table*}[t]
\centering
\begin{tabular}{@{}lllll@{}}
\toprule
\textbf{Operation} & \textbf{Requirement} & \textbf{Input Condition} & \textbf{Output Condition}  \\ \midrule
     $\marg(\circuit; \varsubset)  $            &        -              &      $\margdetvars$-det        &    $\margdetvars$-det                 \\
     $\inst(\circuit; \varsubsetval) $         &       -               &     $\exists \varsubset' \subseteq \varsubset: (\margdetvars \cup \varsubset')\textnormal{-det} $        &   $\margdetvars$-det                    \\ \midrule
    \multirow{3}{*}{$\product(\circuit^{(1)}, \circuit^{(2)})$}     &   $\circuit^{(1)}, \circuit^{(2)}$  respect             &  $\exists \margdetvars^{(1)}, \margdetvars^{(2)}: \margdetvars^{(1)}\textnormal{-det}, \margdetvars^{(2)}\textnormal{-det}$, and:
    
         &      \multirow{3}{*}{$\margdetvars$-det}            \\ 
             &       compatible  vtrees               &   \textbullet\ Either (a) $\margdetvars \subseteq \pcvars^{(1)} \cap \pcvars^{(2)}$ and $\margdetvars^{(1)} = \margdetvars^{(2)} = \margdetvars$;          &                 &                     \\ 
             & & \textbullet\ Or (b) $\margdetvars^{(1)}, \margdetvars^{(2)} \supseteq \pcvars^{(1)} \cap \pcvars^{(2)}$ and $\margdetvars = \margdetvars^{(1)} \cup \margdetvars^{(2)}$ & & \\ \midrule
        $\powop(\circuit; \power)$           &      Det                &       $\margdetvars$-det         &      $\margdetvars$-det             \\
        $\maxop(\circuit)$           &      Det                &       N/A        &      N/A (scalar output)                 \\
        $\logop(\circuit)$          &      Det               &        -        &       -                 \\ \bottomrule
\end{tabular} 
\caption{MD-calculus: sufficient input-output conditions for each basic operation}
\label{tbl:mdcalc}
\end{table*}

%% file: sections/applications.tex
\section{APPLICATION: CAUSAL INFERENCE}

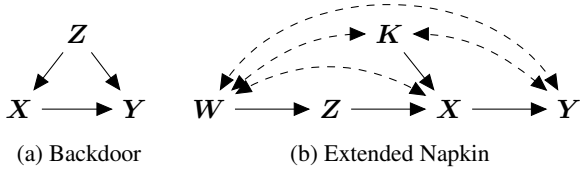
\begin{figure}
\begin{subfigure}[b]{0.3\linewidth}
\centering
\input{diags/backdoor_graph_compact}
\caption{Backdoor}
\label{fig:backdoor}
\end{subfigure}
\begin{subfigure}[b]{0.68\linewidth}
\centering
\input{diags/extended_napkin}
\caption{Extended Napkin}
\label{fig:extnapkin}
\end{subfigure}
\caption{Examples of causal diagrams}
\label{fig:causal_diagrams}
\end{figure}

In this section, we use our md-vtree framework to analyse tractability conditions for \textit{exact} causal inference for PCs. The typical setup in causal inference is that we have access to an observed distribution $\p(\edg)$, and some domain assumptions, often conveniently expressed using a \textit{causal diagram} \citep{pearl09causality}, and we are interested in computing some \textit{interventional distribution} $p_{\treat}(\outcome)$, where $\treat, \outcome \subseteq \edg$ are disjoint subsets of the observed variables.
$p_{\treat}(\outcome)$ is said to be \textit{identifiable} if the assumptions are sufficient for it to be uniquely determined; and if it is identifiable, a causal formula/estimand can be obtained using the \textit{do-calculus} \citep{pearl1995do, shpitser2006complete}. 

\subsection{Hardness of Backdoor Adjustment on Circuits}

We assume that the observed data distribution is modelled by a PC $\circuit$ (perhaps learned from data), and we wish to compute some causal query. One of the most common cases where the interventional distribution is identifiable is when there exists a valid \textit{backdoor adjustment set} $\adjust \subseteq \edg \setminus (\treat \cup \outcome)$ (also known as the conditional exchangability/ignorability assumption), an example of which is illustrated in Figure \ref{fig:backdoor}. Whenever such a set exists, the interventional distribution $\pcfunc_{\circuit, \treat}(\outcome)$ is given by the backdoor adjustment formula\footnote{As commonly done in causal inference we assume positivity, i.e. $\pcfunc_{\circuit}(\pcvars) \geq 0$ such that the conditional is well defined.}:
\begin{align*} \label{eqn:backdoor}
     \sum_{\adjust} \pcfunc_{\circuit}(\adjust) \pcfunc_{\circuit}(\outcome|\treat, \adjust) = \sum_{\adjust} \pcfunc_{\circuit}(\adjust) \frac{\pcfunc_{\circuit}(\outcome, \treat, \adjust)}{\pcfunc_{\circuit}(\treat, \adjust)}
\end{align*}
Unfortunately, we now show that the backdoor adjustment is not tractable for existing classes of probabilistic circuits:

\begin{restatable}[Hardness of Backdoor Query]{theorem}{backdoorhard}
\label{thm:backdoorhard}
The backdoor query for decomposable and smooth PCs is \textsc{\#P}-hard, even if the PC is structured decomposable and deterministic.
\end{restatable}

We can also view this result as placing a theoretical barrier on interpreting probabilistic circuits as causal models \citep{zhao2015spnbn, papantonis2020causal}; namely that, even if such an interpretation exists, it is not possible to tractably perform causal inference with the causal PC. Thus, whether we interpret a PC as itself expressing causality, or merely a model of the observational probability distribution, new PC classes are required. In the following, we take the latter perspective, employing our md-vtrees.

\subsection{MD-calculus for Causal Formulae}

We now employ the MD-calculus to derive tractability conditions for the backdoor query, given $\circuit(\pcvars)$ as input. Note that there is only one deterministic operation, namely the reciprocal $\powop(\cdot; -1)$. The input to this operation is $\circuit(\treat, \adjust)$, which we thus require to be deterministic, that is, $(\treat \cup \adjust)$-deterministic. We then have that $\circuit(\treat, \adjust) = \marg(\circuit(\pcvars); \pcvars \setminus (\treat \cup \adjust))$, so, using the MD-calculus rule for $\marg$, it is sufficient for $\circuit(\pcvars)$ to be $(\treat \cup \adjust)$-deterministic. Given a vtree, we can then use Algorithm \ref{alg:optlabel} to derive an md-vtree for the input circuit $\circuit(\pcvars)$.

It is worth noting that, unlike typical causal inference approaches which derive a scalar $\pcfunc_{\circuit, \treatval}(\outcomeval)$ (or some causal effect) for a particular intervention $\treat = \treatval$, the output of the backdoor query here is a \textit{circuit} over variables $\treat, \outcome$ encoding $\pcfunc_{\circuit, \treat}(\outcome)$. This means that we can do further downstream reasoning over different values of $\treat, \outcome$.

As a second example, consider the \textit{extended napkin} causal diagram in Figure \ref{fig:extnapkin}. We can use the do-calculus to derive the following expression for $\pcfunc_{\circuit, \treat}(\outcome)$ in this case \footnote{Here, $\adjustval$ can be any instantiation of $\adjust$, and $f(\circuit, \bm{K})$ is a expression over $\circuit, \bm{K}$ that we omit here for clarity.}:
\begin{align*}
    \centering
    \sum_{\bm{K}} &f(\circuit, \bm{K}) \frac{\sum_{\bm{W}} \pcfunc_{\circuit}(\treat, \outcome| \bm{K}, \adjustval, \bm{W}) \pcfunc_{\circuit}(\bm{W}, \bm{K})}{\sum_{\bm{W}} \pcfunc_{\circuit}(\treat| \bm{K}, \adjustval, \bm{W}) \pcfunc_{\circuit}(\bm{W}, \bm{K})}
\end{align*}

For this formula, we can try to apply the MD-calculus as usual. The denominator (input to the $\powop(\cdot; -1)$ operation) is a function of $\{\treat, \bm{K}\}$; thus, passing through the marginalization operation, we require $\product(\circuit(\treat| \bm{K}, \adjustval, \bm{W}) ,\circuit(\bm{W}, \bm{K}))$ to be $(\treat \cup \bm{K})$-deterministic. However, if we then look at the conditions for the $\product$ operation, we see this cannot be achieved. In particular, $\circuit(\treat| \bm{K}, \adjustval, \bm{W})$ and $\circuit(\bm{W}, \bm{K})$ are circuits over $\pcvars^{(1)} = \{\treat, \bm{K}, \bm{W}\}$ and $\pcvars^{(2)} = \{\bm{K}, \bm{W}\}$ respectively, and there are no possible $\margdetvars^{(1)}, \margdetvars^{(2)}$ such that condition (a) or (b) holds.

However, we \textit{can} derive a tractable pipeline for $\pcfunc_{\circuit, \treatval}(\outcome)$ where we commit to a specific intervention $\treat = \treatval$. In this case, the new denominator is a function of just $\bm{K}$, and so we require $\product(\circuit(\treatval| \bm{K}, \adjustval, \bm{W}), \circuit(\bm{W}, \bm{K}))$  to be $(\treat \cup \bm{K}) \setminus (\treat) = \bm{K}$-deterministic. This can be achieved by choosing $\margdetvars^{(1)} = \margdetvars^{(2)} = \{\bm{K}\}$, which satisfies condition (a) for $\product$. 
By analyzing the rest of the formula, we can derive a set $\targetmdset$ of marginal determinisms for $\circuit(\pcvars)$ that is sufficient for the extended napkin query, and apply Algorithm \ref{alg:optlabel}; we defer the full derivations (and analysis of the frontdoor formula) to the Appendix.

%% file: diags/backdoor_graph_compact.tex
\begin{tikzpicture}

\node[] (X) {$\treat$};
\node[right=0.2cm of X, yshift=1cm] (Z) {$\adjust$};
\node[right=0.2cm of Z, yshift=-1cm] (Y) {$\outcome$};

\edge {Z} {X};
\edge {Z} {Y};
\edge {X} {Y};

\end{tikzpicture}

%% file: diags/extended_napkin.tex
\begin{tikzpicture}
\node[] (W) {$\bm{W}$};
\node[right=of W] (Z) {$\adjust$};
\node[right=0.2cm of Z, yshift=1cm] (K) {$\bm{K}$};
\node[right=of Z] (X) {$\treat$};
\node[right=of X] (Y) {$\outcome$};

\edge {W} {Z};
\edge {Z} {X};
\edge {K} {X};
\edge {X} {Y};

\draw[<->, dashed] (W) to [out=60, in=120] (Y);
\draw[<->, dashed] (W) to [out=30, in=150] (X);
\draw[<->, dashed] (W) to [out=40, in=180](K);
\draw[<->, dashed] (K) to [out=0, in=140] (Y);
\end{tikzpicture}

%% file: sections/experiments.tex
\section{EMPIRICAL EVALUATION}

\begin{table}[]
\begin{tabular}{@{}lccccc@{}}
\toprule
\multirow{2}{*}{\textbf{Dataset}} & \multirow{2}{*}{$|\bm{Z}|$} & \multicolumn{2}{c}{\textbf{Error}} & \multicolumn{2}{c}{\textbf{Time}} \\
                              &    & MD  & Counting & MD  & Counting \\ \midrule
Asia                   &   4       &  0.269             &  \textbf{0.0143}       &    \textbf{0.5}         &      1.0                          \\
Sachs                   &   4       &    0.180           &  \textbf{0.0219}       &     1.7      &     \textbf{ 0.7}                          \\
Child                  &   13    &     \textbf{0.0802}          &  0.135       &      1.9       &      \textbf{0.9}                          \\
Win95pts                   &   59     &        \textbf{0.0044 }       &  0.0511       &      3.2       &      \textbf{0.9}                          \\
Andes                     &    202          &   \textbf{0.0382}       &       0.0982      &           7.9    &          \textbf{1.3}            \\ \bottomrule
\end{tabular}
\caption{Backdoor Estimation (averaged over 10 runs)}
\label{tbl:results}
\end{table}

In this section, we empirically evaluate our tractable algorithm for backdoor adjustment derived using MD-calculus. We generate datasets by sampling $1000$ datapoints from the (discrete variable) Bayesian network (BN) models in the \textit{bnrepository} \citep{bnrepo}, and learn a MDNet over all variables $\pcvars$ from data. For each Bayesian network (causal graph), we select a single treatment variable $\treatm$ and single outcome variable $\outcomem$, as well as a set of variables $\adjust$ forming a valid backdoor adjustment set for $(\treatm, \outcomem)$, and seek to estimate $\sum_{\adjust} p(\outcomem|\treatm, \adjust) p(\adjust)$. We manually select a vtree which splits the scope into $(\treatm \cup \adjust)$ and $\outcomem$ at the root, and thereafter generates the rest of the vtree randomly. Given a vtree, we use Algorithm \ref{alg:optlabel} to generate a labelling, i.e. regular md-vtree. The required tractability properties for backdoor adjustment are then enforced through the structure of the corresponding MDNet. 

The results are shown in Table \ref{tbl:results}; for comparison, we show also results for the \textit{counting} approach, where we estimate $p(\outcomem|\treatm, \adjust)$ as $\frac{N_{\outcomem, \treatm, \adjust}}{N_{\treatm, \adjust}}$, where $N$ refers to the number of datapoints with the subscripted assignment of variables (this is set to 0 if $N_{\treatm, \adjust} = 0$).
It can be seen that, while the counting approach is generally more robust in lower dimensions, the advantage in terms of learning a full model becomes apparent with the higher-dimensional adjustment sets, as shown by the lower error on the Win95pts and Andes datasets. Remarkably, as the size of the adjustment set $\bm{Z}$ increases, the time taken for the algorithm based on the MD-calculus increases only approximately \textit{linearly} in the dimension.  This illustrates the attractiveness of tractable probabilistic modelling, in that we can systematically control the computational cost of exact inference by restricting the size of the PC model.

%% file: sections/conclusion.tex
\section{CONCLUSION}

In summary, we introduced the md-vtree framework for support properties in structured decomposable circuits, and showed how it can be employed for reasoning about tractability of compositional inference queries using our MD-calculus rules. Our unifying framework naturally provides insight into the properties of previously proposed PC classes such as PSDDs, as well as inspiring our newly designed MDNet architecture. Nonetheless, there remain a number of interesting challenges for future work. For example, for more challenging datasets, how can we also learn the vtree and/or MDNet structure from data while maintaining tractability? For what other inference queries can we derive tractability conditions using the MD-calculus? We hope that our work will help lay the theoretical foundations for tackling these questions.

%% file: sections/acknowledgements.tex
\subsubsection*{Acknowledgements}

This project was funded by the ERC under the European Union’s Horizon 2020 research and innovation programme (FUN2MODEL, grant agreement No.834115).

%% file: sections/appendix/proofs.tex
\section{Md-vtree Proofs} \label{apx:proofs}

In this section, we provide proofs of the results in Section \ref{sec:mdvtree} regarding md-vtrees.

\subsection{$\margdetvars$-determinism Results}
We begin with the results regarding (structured) marginal determinism.

\propsupport*

\begin{proof}
Since the sum node $\sumunit$ is non-trivial, it has at least two children. Let $\genericunit_1, \genericunit_2$ be two distinct children of $\sumunit$. Let $I_1, I_2$ be the sets of values $\margdetvarsval$ of $\margdetvars$ such that $\dist_{\genericunit_1}(\margdetvarsval) > 0, \dist_{\genericunit_2}(\margdetvarsval) > 0$ respectively. Define $I'_1, I'_2$ similarly for $\margdetvars'$. Note that, by $\margdetvars$-determinism and $\margdetvars'$-determinism, $I_1, I_2$ are disjoint, and similarly $I'_1, I'_2$ are disjoint.

Now, we claim that there exists values $\margdetvarsval \in I_1$ and $\margdetvarsval' \in I_2'$ such that they agree over the intersection $\margdetvars \cap \margdetvars'$. If not, then $\dist_{\genericunit_1}$ and $\dist_{\genericunit_2}$ are non-zero for disjoint subsets of values of $(\margdetvars \cap \margdetvars')$, which implies $(\margdetvars \cap \margdetvars')$-determinism, which is a contradiction of the assumption of the Proposition. Now consider the value $\margdetvarsval \cup \margdetvarsval'$ of $\margdetvars \cup \margdetvars'$. $\dist_{\genericunit_1}(\margdetvarsval \cup \margdetvarsval') = 0$ since $\dist_{\genericunit_1}(\margdetvarsval') = 0$ (by the disjointness of $I'_1, I'_2$), and similarly $\dist_{\genericunit_2}(\margdetvarsval \cup \margdetvarsval') = 0$ since $\dist_{\genericunit_2}(\margdetvarsval) = 0$. For any other child $\genericunit_3$ of $\sumunit$, we have that $I_1$ and $I_3$ are disjoint by $\margdetvars$-determinism, so $\dist_{\genericunit_3}(\margdetvarsval) = 0$ and we get $\dist_{\genericunit_3}(\margdetvarsval \cup \margdetvarsval') = 0$. Putting it all together, $\dist_{\sumunit}(\margdetvarsval \cup \margdetvarsval') = 0$ and thus the circuit $\circuit_{\sumunit}$ does not have full support.

\end{proof}

An important corollary of this result is that any \emph{circuit} that is $\margdetvars$-deterministic and $\margdetvars'$-deterministic cannot have full support, as the root sum node $\rootunit$ of the circuit must be $\margdetvars$-deterministic and $\margdetvars'$-deterministic. 

\propsubset*

\begin{proof}
By definition, a sum node $\sumunit$ is $\margdetvars$-deterministic if for any instantiation $\margdetvarsval$ of $\margdetvars$, at most one of its children $\genericunit_i$ evaluate to a nonzero output under $\margdetvarsval$. If $\margdetvars' \supseteq \margdetvars$, then any instantiation $\margdetvarsval'$ of $\margdetvars'$ will imply a specific instantiation of $\margdetvarsval$, and so at most one of the children of $\sumunit$ evaluate to a nonzero output under $\margdetvarsval'$. More formally, $\dist_{\genericunit_i}(\margdetvarsval) = \sum_{\margdetvars' \setminus \margdetvars} \dist_{\genericunit_i}(\margdetvarsval, \margdetvars' \setminus \margdetvars) = 0 \implies \dist_{\genericunit_i}(\margdetvarsval') = 0$.
\end{proof}

\propimplied*

\begin{proof}
Since $\circuit$ respects $\mdvtree$, every sum unit $\sumunit \in \circuit$ has scope $\scope(\sumunit) = \scope(\vnode)$ for some $\vnode \in \vnodes$, and $\sumunit$ is $\vlabel(\vnode)$-deterministic. Further, since $\mdvtree$ implies $\margdetvars$-determinism, we have that  $\scope(\vnode) \cap \margdetvars = \emptyset$, or else $\margdetvars \supseteq \vlabel(\vnode)$. Combining these statements, we see that for all sum units $\sumunit \in \circuit$, either $\scope(\sumunit) \cap \margdetvars = \emptyset$, or else $\sumunit$ is $\vlabel(\vnode)$-deterministic and thus (by Proposition \ref{prop:supset}) $\margdetvars$-deterministic. This shows that $\circuit$ is $\margdetvars$-deterministic.
\end{proof}

The following theorem justifies the intuition that having smaller labels $\vlabel(\vnode)$ corresponds to a stronger restriction on the circuit, such that less circuits respect the md-vtree, but more marginal determinisms are implied:

\propstrength*

\begin{proof}
For the first part, suppose $\margdetvars \in \margdetset(\mdvtree')$. Then for all $\vnode \in \vnodes$, it holds that $\scope(\vnode) \cap \margdetvars = \emptyset$, or else $\margdetvars \supseteq \vlabel'(\vnode)$. Since $\vlabel'(\vnode) \supseteq \vlabel(\vnode)$, it holds that $\scope(\vnode) \cap \margdetvars = \emptyset$, or else $\margdetvars \supseteq \vlabel(\vnode)$ also, so $\margdetvars \in \margdetset(\mdvtree)$. This shows that $\margdetset(\mdvtree) \supseteq \margdetset(\mdvtree')$.

For the second part, suppose that $\circuit \in \circuitclass_{\mdvtree}$. Then, for any sum unit $\sumunit \in \circuit$, there is an $\vnode \in \vnodes$ that $\sumunit$ is marginally deterministic w.r.t. $\vlabel(\vnode)$. As $\vlabel'(\vnode) \supseteq \vlabel(\vnode)$, this means that $\sumunit$ is also marginally deterministic w.r.t. $\vlabel'(\vnode)$. Thus $\circuit$ respects $\mdvtree'$ also, i.e. $\circuit \in \circuitclass_{\mdvtree'}$.
\end{proof}

\subsection{Regular and Optimal md-vtrees}

We now move to the results regarding regular and optimal md-vtrees. Recall that Algorithm \ref{alg:optlabel} is designed to return a labelling of a given vtree that satisfies some desired set of marginal determinisms. We begin with a formal correctness proof of Algorithm \ref{alg:optlabel}, i.e. that it does indeed return an md-vtree which implies the given set $\targetmdset$ of marginal determinisms:

\begin{proposition}[Correctness of Algorithm 1]
    For any input vtree $\vtree = (\vnodes, \vedges)$ and set of marginal determinisms $\targetmdset$, Algorithm \ref{alg:optlabel} returns an md-vtree that implies $\margdetvars$-determinism for all $\margdetvars \in \targetmdset$.
\end{proposition}
\begin{proof}
    We need to show that, for each $\vnode \in \vnodes$ and $\margdetvars \in \targetmdset$, if $\scope(\vnode) \cap \margdetvars \neq \emptyset$, then $\margdetvars \supseteq \vlabel(\vnode)$.

    This can be seen by inspection of the Algorithm. For every $\vnode \in \vnodes$, and every $\margdetvars \in \targetmdset$, if $\scope(\vnode) \cap \margdetvars \neq \emptyset$, then in line \ref{algline:qloopend} of the Algorithm, we intersect the label with the $\margdetvars$. At each iteration of the loop in lines \ref{algline:qloopbegin}-\ref{algline:qloopend}, no elements can be added to $\vlabel(\vnode)$ as we are taking an intersection, and in lines \ref{algline:edgecasebegin}-\ref{algline:edgecaseend} we also take an intersection (with the scope). Thus, it follows that for all $\margdetvars \in \targetmdset$ such that $\scope(\vnode) \cap \margdetvars \neq \emptyset$, the label at the end of the outer loop will satisfy $\margdetvars \supseteq \vlabel(\vnode)$. 
\end{proof}

Now, we move on to properties of Algorithm \ref{alg:optlabel}. We begin by showing that the output md-vtree is regular; that is, the labelling function satisfies a certain structure as defined in Definition \ref{defn:regular} and reproduced below.

\defnregular*

\propalg*

\begin{proof}
By examining the Algorithm, we can see that the  label of each vtree node node is given by:
\begin{equation}
\vlabel(\vnode) = 
\begin{cases}
    \vlabelno \;\;\;\;\;\;\;\;\;\;\;\;\;\;\;\;\;\;\;\;\;\;\;\;\;\;\;\;\;\;\;\;\;\;\;\;\;\;\;\;\;\;\;\; \textnormal{if } \margdetvars \cap \scope(\vnode)= \emptyset \;\; \forall \margdetvars \in \targetmdset \\
    \scope(\vnode) \cap (\bigcap_{\margdetvars \in \targetmdset: \margdetvars \cap \scope(\vnode) \neq \emptyset} \margdetvars ) \;\;\;\;\;\; \textnormal{otherwise} 
\end{cases}
\end{equation}

For any non-leaf vtree-node $\vnode$, let its children be $\vnode_1, \vnode_2$, and note that by definition, $\scope(\vnode) = \scope(\vnode_1) \cup \scope(\vnode_2)$ with $\scope(\vnode_1), \scope(\vnode_2)$ disjoint. Firstly, if $\vlabel(\vnode) = \vlabelno$, this means that no $\margdetvars \in \targetmdset$ has non-empty intersection with $\scope(\vnode)$, and since the scopes of the children are subsets of $\scope(\vnode)$, we must also have $\vlabel(\vnode_1) = \vlabel(\vnode_2) = \vlabelno$ and the regularity condition is satisfied.

Otherwise, we have that $\vlabel(\vnode) = \scope(\vnode) \cap (\bigcap_{\margdetvars \in \targetmdset_{\vnode}} \margdetvars)$, where have written $\targetmdset_{\vnode} := \{\margdetvars \in \targetmdset: \margdetvars \cap \scope(\vnode) \neq \emptyset\}$ to denote the subset of $\targetmdset$ that has non-empty intersection of $\scope(\vnode)$. Note that $\targetmdset_{\vnode}$ is non-empty. Since $\scope(\vnode) = \scope(\vnode_1) \cup \scope(\vnode_2)$, it must be the case that every $\margdetvars \in \targetmdset_{\vnode}$ also intersects with at least one of $\scope(\vnode_1)$, $\scope(\vnode_2)$. More formally, $\targetmdset_{\vnode} = \targetmdset_{\vnode_{1}} \cup \targetmdset_{\vnode_{2}}$. We consider three cases separately:

\begin{itemize}
    \item If $\targetmdset_{\vnode_{2}}$ is empty, then we must have $\targetmdset_{\vnode} = \targetmdset_{\vnode_{1}}$. Thus $\targetmdset_{\vnode_{1}}$ is non-empty, meaning that $\vlabel(\vnode_1) = \scope(\vnode_1) \cap (\bigcap_{\margdetvars \in \targetmdset_{\vnode_{1}}} \margdetvars)$ in the definition. We can then derive that $\vlabel(\vnode_1) = \scope(\vnode_1) \cap (\bigcap_{\margdetvars \in \targetmdset_{\vnode_{1}}} \margdetvars) = \scope(\vnode_1) \cap (\bigcap_{\margdetvars \in \targetmdset_{\vnode}} \margdetvars) = \scope(\vnode) \cap (\bigcap_{\margdetvars \in \targetmdset_{\vnode}} \margdetvars ) = \vlabel(\vnode_1)$, where the second-to-last equality follows since all $\margdetvars \in \targetmdset_{\vnode}$ have empty intersection with $\scope(\vnode_2)$. This satisfies the regularity condition.
    \item If $\targetmdset_{\vnode_{1}}$ is empty, by similar reasoning we have that $\vlabel(\vnode_2) = \vlabel(\vnode)$.
    \item If neither $\targetmdset_{\vnode_{1}}, \targetmdset_{\vnode_{2}}$ is empty, then we have that $\vlabel(\vnode) = \scope(\vnode) \cap (\bigcap_{\margdetvars \in \targetmdset_{\vnode}} \margdetvars) = \left(\scope(\vnode_1) \cap (\bigcap_{\margdetvars \in \targetmdset_{\vnode}} \margdetvars)\right) \cup \left(\scope(\vnode_2) \cap (\bigcap_{\margdetvars \in \targetmdset_{\vnode}} \margdetvars)\right) = \left(\scope(\vnode_1) \cap (\bigcap_{\margdetvars \in \targetmdset_{\vnode_{1}}} \margdetvars)\right) \cup \left(\scope(\vnode_2) \cap (\bigcap_{\margdetvars \in \targetmdset_{\vnode_{2}}} \margdetvars)\right)$. Since both $\targetmdset_{\vnode_{1}}, \targetmdset_{\vnode_{2}}$ are non-empty, we have that $\vlabel(\vnode_1) = \scope(\vnode_1) \cap (\bigcap_{\margdetvars \in \targetmdset_{\vnode_{1}}} \margdetvars)$ and $\scope(\vnode_2) \cap (\bigcap_{\margdetvars \in \targetmdset_{\vnode_{2}}} \margdetvars)$ in the definition. Combining, we have that $\vlabel(\vnode) = \vlabel(\vnode_1) \cup \vlabel(\vnode_2)$, which satisfies the regularity condition.
\end{itemize} 

Thus we have shown that the output md-vtree is regular.

\end{proof}

In the following theorem, we show that the significance of regularity, in that regular md-vtrees are ``optimally expressive'' among all md-vtrees with the same vtree.

\thmoptimal*

We will prove this Theorem explicitly by constructing a regular md-vtree with these properties. To do this, we prove two Lemmas which define operations which do not change $\margdetset(\mdvtree)$, while keeping the same or increasing the set $\circuitclass_{\mdvtree}$; the result of iterative application of the two operations being a regular md-vtree.

\begin{definition}[Expand Child Labels] \label{defn:ecl}
    Given a md-vtree $\mdvtree = (\vtree, \vlabel)$, and any vtree nodes $\vnode_{\text{pa}^*}, \vnode_{\text{ch}^*}$ such that $\vnode_{\text{ch}^*}$ is a child of $\vnode_{\text{pa}^*}$, the operation $\texttt{ECL}(\mdvtree, \vnode_{\text{pa}^*}, \vnode_{\text{ch}^*})$ returns a new md-vtree $\mdvtree' = (\vtree, \vlabel')$, defined as follows:
    \begin{equation}
    \vlabel'(\vnode) = 
        \begin{cases}
            \vlabel(\vnode_{\text{ch}^*}) \cup (\vlabel(\vnode_{\text{pa}^*}) \cap \scope(\vnode_{\text{ch}^*}) ) & \textnormal{if } \vnode = \vnode_{ch}\\
            \vlabel(\vnode) & \textnormal{otherwise}
        \end{cases}
    \end{equation}
\end{definition}

\begin{lemma} \label{lem:pasubsetch}
 The output $\mdvtree' = \texttt{ECL}(\mdvtree, \vnode_{\text{pa}^*}, \vnode_{\text{ch}^*})$ satisfies $\margdetset(\mdvtree') = \margdetset(\mdvtree)$, and $\circuitclass_{\mdvtree} \subseteq \circuitclass_{\mdvtree'}$.
\end{lemma}

\begin{proof}
The only difference between $\mdvtree$ and $\mdvtree'$ is the label of $\vnode_{\text{ch}^*}$. Suppose that $\margdetvars \in \margdetset(\mdvtree)$, then we have that either $\scope(\vnode_{\text{ch}^*}) \cap \margdetvars = \emptyset$, or else $\margdetvars \supseteq \vlabel(\vnode_{\text{ch}^*})$. In the former case, since the vtrees and thus scopes are the same between $\mdvtree, \mdvtree$, it follows that $\margdetvars \in \margdetset(\mdvtree')$ also. In the latter case, since the scope of the parent $\scope(\vnode_{\text{pa}^*}) \supseteq \scope(\vnode_{\text{ch}^*})$, $\margdetvars$ overlaps with the parent scope as well, implying that that $\margdetvars \supseteq \vlabel(\vnode_{\text{pa}^*})$. Thus we have that $\margdetvars \supseteq \vlabel(\vnode_{\text{ch}^*}) \cup \vlabel(\vnode_{\text{pa}^*}) \supseteq  \vlabel(\vnode_{\text{ch}^*}) \cup (\vlabel(\vnode_{\text{pa}^*}) \cap \scope(\vnode_{\text{ch}^*}) ) = \vlabel'(\vnode_{\text{ch}^*})$. Thus, $\margdetvars \in \margdetset(\mdvtree')$ also. That is, $\margdetset(\mdvtree) \subseteq \margdetset(\mdvtree')$. 

To complete the result, note that $\vlabel(\vnode) \subseteq \vlabel'(\vnode)$ for all vtree nodes $\vnode$. Thus by Theorem \ref{prop:strength}, it follows that $\margdetset(\mdvtree) \supseteq \margdetset(\mdvtree')$ and $\circuitclass_{\mdvtree} \subseteq \circuitclass_{\mdvtree'}$. Combining with the paragraph above we have shown that $\margdetset(\mdvtree) = \margdetset(\mdvtree')$.
\end{proof}

Intuitively, this operation "pushes down" elements of $\vlabel(\vnode_{\text{pa}^*})$ to its children. If we apply this operation to all pairs of parent/child vtree nodes $(\vnode_{\text{pa}}, \vnode_{\text{ch}})$, then it can be seen that the new labels will have the property that \emph{all elements of the parent label that are contained in the scope of a child, will be in the label of that child}. More formally,
$\vlabel'(\vnode_{\text{pa}}) \cap \scope(\vnode_{\text{ch}}) = \vlabel'(\vnode_{\text{pa}}) \cap \vlabel'(\vnode_{\text{ch}})$.
This is the starting point for the next operation:

\begin{definition}[Expand Parent Labels] \label{defn:epl}
Let $\mdvtree = (\vtree, \vlabel)$ be a md-vtree such that $\vlabel(\vnode_{\text{pa}}) \cap \scope(\vnode_{\text{ch}}) = \vlabel(\vnode_{\text{pa}}) \cap \vlabel(\vnode_{\text{ch}})$ holds for all pairs of parents $\vnode_{\text{pa}}$ and children $\vnode_{\text{ch}}$. Then, given any non-leaf vtree node $\vnode_{\text{pa}^*}$, the operation $\texttt{EPL}(\mdvtree, \vnode_{\text{pa}^*})$ returns a new md-vtree $\mdvtree' = (\vtree, \vlabel')$, defined as follows:
\begin{equation}
    \vlabel'(\vnode) = 
        \begin{cases}
             \bigcup_{\vnode_{\text{ch}^*} \in \vnodes_{\text{active}}} \vlabel(\vnode_{\text{ch}^*}) & \textnormal{if } \vnode = \vnode_{pa}\\
            \vlabel(\vnode) & \textnormal{otherwise}
        \end{cases}
    \end{equation}
where we define $\vnodes_{\text{active}} = \{\vnode_{\text{ch}^*} | \vnode_{\text{ch}^*} \in \vnodes_{\text{ch}^*}, \vlabel(\vnode_{\text{ch}^*}) \cap \vlabel(\vnode_{\text{pa}^*}) \neq \emptyset\}$ to be the set of all children whose labellings have non-empty intersection with the labelling of the parent.
\end{definition}

\begin{lemma} \label{lem:paunionch}
 The output $\mdvtree' = \texttt{EPL}(\mdvtree, \vnode_{\text{pa}^*})$ satisfies $\margdetset(\mdvtree') = \margdetset(\mdvtree)$, and $\circuitclass_{\mdvtree} \subseteq \circuitclass_{\mdvtree'}$.

Further, the property that $\vlabel'(\vnode_{\text{pa}}) \cap \scope(\vnode_{\text{ch}}) = \vlabel'(\vnode_{\text{pa}}) \cap \vlabel'(\vnode_{\text{ch}})$ holds for all pairs of parents $\vnode_{\text{pa}}$ and children $\vnode_{\text{ch}}$ in $\mdvtree'$ (i.e. is maintained in $\mdvtree'$).
\end{lemma}

\begin{proof}
Firstly, we show that $\vlabel(\vnode_{\text{pa}^*}) \subseteq \vlabel'(\vnode_{\text{pa}^*})$. This follows by taking a union over children of both sides of the assumption $\vlabel(\vnode_{\text{pa}^*}) \cap \scope(\vnode_{\text{ch}^*}) = \vlabel(\vnode_{\text{pa}^*}) \cap \vlabel(\vnode_{\text{ch}^*})$, where the LHS becomes $\bigcup_{\vnode_{\text{ch}^*} \in \vnodes_{\text{ch}^*}} \left(\vlabel(\vnode_{\text{pa}^*}) \cap \scope(\vnode_{\text{ch}^*}) \right) 
= \vlabel(\vnode_{\text{pa}^*}) \cap \bigcup_{\vnode_{\text{ch}^*} \in \vnodes_{\text{ch}^*}} \scope(\vnode_{\text{ch}^*}) = \vlabel(\vnode_{\text{pa}^*}) \cap \scope(\vnode_{\text{pa}^*}) = \vlabel(\vnode_{\text{pa}^*})$, and the RHS becomes $\bigcup_{\vnode_{\text{ch}^*} \in \vnodes_{\text{ch}^*}} \left(\vlabel(\vnode_{\text{pa}^*}) \cap \vlabel(\vnode_{\text{ch}^*})\right) = \bigcup_{\vnode_{\text{ch}^*} \in \vnodes_{\text{active}}} \left(\vlabel(\vnode_{\text{pa}^*}) \cap \vlabel(\vnode_{\text{ch}^*})\right) = \vlabel(\vnode_{\text{pa}^*}) \cap \bigcup_{\vnode_{\text{ch}^*} \in \vnodes_{\text{active}}}  \vlabel(\vnode_{\text{ch}^*}) \subseteq \vlabel'(\vnode_{\text{pa}^*})$. Thus
$\vlabel(\vnode) \subseteq \vlabel'(\vnode)$ for all vtree nodes $\vnode$, and by Theorem \ref{prop:strength}, it follows that $\margdetset(\mdvtree) \supseteq \margdetset(\mdvtree')$ and $\circuitclass_{\mdvtree} \subseteq \circuitclass_{\mdvtree'}$.  

Now suppose $\margdetvars \in \margdetset(\mdvtree)$. We consider two cases. Firstly, if $\scope(\vnode_{\text{pa}^*}) \cap \margdetvars = \emptyset$, then since vtrees and scopes are the same between $\mdvtree, \mdvtree'$, we have $\margdetvars \in \margdetset(\mdvtree')$. Otherwise, we have  $\scope(\vnode_{\text{pa}^*}) \cap \margdetvars \neq \emptyset$ and $\margdetvars \supseteq \vlabel(\vnode_{\text{pa}^*})$. Now, for those children in $\vnodes_{\text{active}}$, we have that $\vlabel(\vnode_{\text{ch}^*}) \cap \vlabel(\vnode_{\text{pa}^*}) \neq \emptyset$ and so since $\margdetvars \supseteq \vlabel(\vnode_{\text{pa}^*})$ and $\scope(\vnode_{\text{ch}^*}) \supseteq \vlabel(\vnode_{\text{ch}^*})$, we have $\margdetvars \cap \scope(\vnode_{\text{ch}^*}) \neq \emptyset$. The marginal determinism property on these children then implies that $\margdetvars \supseteq \vlabel(\vnode_{\text{ch}^*})$ for all $\vnode_{\text{ch}^*} \in \vnodes_{\text{active}}$; and so $\margdetvars \supseteq \bigcup_{\vnode_{\text{ch}^*} \in \vnodes_{\text{active}}} \vlabel(\vnode_{\text{ch}^*}) = \vlabel'(\vnode_{\text{pa}^*})$. This shows that $\margdetvars \in \margdetset(\mdvtree')$ also. This gives $\margdetset(\mdvtree) \subseteq \margdetset(\mdvtree')$, and combined with the previous result, $\margdetset(\mdvtree) = \margdetset(\mdvtree')$.

Finally, we show that the property that $\vlabel'(\vnode_{\text{pa}}) \cap \scope(\vnode_{\text{ch}}) = \vlabel'(\vnode_{\text{pa}}) \cap \vlabel'(\vnode_{\text{ch}})$ holds for all pairs of parents $\vnode_{\text{pa}}$ and children $\vnode_{\text{ch}}$ in $\mdvtree'$. The only label which has changed is that of $\vnode_{\text{pa}}$, so we need only consider the pairs $(\vnode_{\text{pa}}, \vnode_{\text{ch}})$ with either (a) $\vnode_{\text{pa}} = \vnode_{\text{pa}^*}$ and $\vnode_{\text{ch}}$ is a child of $\vnode_{\text{pa}^*}$ or (b) $\vnode_{\text{ch}} = \vnode_{\text{pa}^*}$ and $\vnode_{\text{pa}}$ is the parent of $\vnode_{\text{pa}^*}$. 
\begin{itemize}
    \item In case (a), by definition we have that $\vlabel'(\vnode_{\text{pa}}) = \vlabel'(\vnode_{\text{pa}^*}) = \bigcup_{\vnode_{\text{ch}^*} \in \vnodes_{\text{active}}} \vlabel(\vnode_{\text{ch}^*})$, and $\vlabel'(\vnode_{\text{ch}}) = \vlabel'(\vnode_{\text{ch}})$. If $\vnodech$ is an active child of $\vnodepa^*$, then we have that the LHS of the property $\vlabel'(\vnode_{\text{pa}}) \cap \scope(\vnode_{\text{ch}}) = \bigcup_{\vnode_{\text{ch}^*} \in \vnodes_{\text{active}}} \vlabel(\vnode_{\text{ch}^*}) \cap \scope(\vnode_{\text{ch}}) = \vlabel(\vnodech)$, and the RHS of the property $\vlabel'(\vnode_{\text{pa}}) \cap \vlabel'(\vnode_{\text{ch}}) = \bigcup_{\vnode_{\text{ch}^*} \in \vnodes_{\text{active}}} \vlabel(\vnode_{\text{ch}^*}) \cap \vlabel'(\vnode_{\text{ch}}) = \vlabel(\vnodech)$. If $\vnodech$ is not an active child of $\vnodepa^*$, then both sides of the property correspond to the empty set. 
    \item In case (b), by definition we have $\vlabel'(\vnodepa) = \vlabel(\vnodepa)$, and $\vlabel'(\vnodech) = \vlabel'(\vnode_{\text{pa}^*})$. We have shown above that $\vlabel(\vnode_{\text{pa}^*}) \subseteq \vlabel'(\vnode_{\text{pa}^*})$, so $\vlabel(\vnodech) \subseteq \vlabel'(\vnodech)$.  By the precondition for applying the \texttt{EPL} operation, we have that $\vlabel(\vnode_{\text{pa}}) \cap \scope(\vnode_{\text{ch}}) = \vlabel(\vnode_{\text{pa}}) \cap \vlabel(\vnode_{\text{ch}})$. Substituting, we get $\vlabel'(\vnode_{\text{pa}}) \cap \scope(\vnode_{\text{ch}}) = \vlabel'(\vnode_{\text{pa}}) \cap \vlabel(\vnode_{\text{ch}}) \subseteq \vlabel'(\vnode_{\text{pa}}) \cap \vlabel'(\vnode_{\text{ch}})$. The other direction $\vlabel'(\vnode_{\text{pa}}) \cap \scope(\vnode_{\text{ch}}) \supseteq \vlabel'(\vnode_{\text{pa}}) \cap \vlabel'(\vnode_{\text{ch}})$ is immediate as the label of a node is contained in its scope. 
\end{itemize}

Thus, we have shown that $\vlabel'(\vnode_{\text{pa}}) \cap \scope(\vnode_{\text{ch}}) = \vlabel'(\vnode_{\text{pa}}) \cap \vlabel'(\vnode_{\text{ch}})$ holds for all pairs of parents $\vnode_{\text{pa}}$ and children $\vnode_{\text{ch}}$ in $\mdvtree'$, concluding the proof.
\end{proof}

Intuitively, this operation "pulls up" elements $\vlabel(\vnode_{\text{ch}})$ of the \emph{active} children to the parent.  After applying this operation to all nodes, we obtain a regular md-vtree, which has the same set $\margdetset$ of marginal determinisms, and is at least as expressive. More formally:

\begin{proof} 
(of Theorem) Starting from $\mdvtree$, apply the \texttt{ECL} operation to each pair of parent and child nodes, in a topological order starting from the root. For each $\vnode_{\text{pa}}, \vnode_{\text{ch}}$ pair, we have that $\vlabel'(\vnode_{\text{pa}}) = \vlabel(\vnode_{\text{pa}})$ and $\vlabel'(\vnode_{\text{ch}}) = \vlabel(\vnode_{\text{ch}}) \cup (\vlabel(\vnode_{\text{pa}}) \cap \scope(\vnode_{\text{ch}}) )$ by definition of the operation. Then $\vlabel'(\vnode_{\text{pa}}) \cap \vlabel'(\vnode_{\text{ch}}) = \vlabel(\vnode_{\text{pa}}) \cap \left(\vlabel(\vnode_{\text{ch}}) \cup (\vlabel(\vnode_{\text{pa}}) \cap \scope(\vnode_{\text{ch}}) )\right) = (\vlabel(\vnode_{\text{pa}}) \cap \vlabel(\vnode_{\text{ch}})) \cup (\vlabel(\vnode_{\text{pa}}) \cap \scope(\vnode_{\text{ch}}) ) = \vlabel(\vnode_{\text{pa}}) \cap \scope(\vnode_{\text{ch}}) = \vlabel'(\vnode_{\text{pa}}) \cap \scope(\vnode_{\text{ch}})$, which is the required property for applying the \texttt{EPL} operation. As we proceed in a topological order, and the operation only modifies the label of the child, it follows that the property $\vlabel'(\vnode_{\text{pa}}) \cap \vlabel'(\vnode_{\text{ch}}) = \vlabel'(\vnode_{\text{pa}}) \cap \scope(\vnode_{\text{ch}})$ holds for all parent/children pairs at the end.

This allows us to apply the \texttt{EPL} operation. We apply this operation to every non-leaf node, in a reverse topological order from the leaves to the root. The precondition for applying the operation holds at all points due to the result of Lemma \ref{lem:paunionch}. This operation only modifies the label of the parent, and so after we have modified all the labels, we have the property that $\vlabel'(\vnode_{\text{pa}}) = \bigcup_{\vnode_{\text{ch}} \in \vnodes_{\text{active}}} \vlabel(\vnode_{\text{ch}})$ for every non-leaf node $\vnode_{\text{pa}}$. That is, it satisfies the conditions to be a regular md-vtree, i.e. $\vlabel'(\vnode_{\text{pa}}) = \vlabel'(\vnode_1), \vlabel'(\vnode_2) \text{, or } \vlabel'(\vnode_1) \cup \vlabel'(\vnode_2)$, where $\vnode_1, \vnode_2$ are the children of $\vnode_{\text{pa}}$.

\end{proof}

%% file: sections/appendix/mdcalculus.tex
\section{Operations and MD-Calculus} \label{apx:mdcalculus} 

In this section, we provide further details on the results in Section \ref{sec:inference} regarding inference on md-vtrees. First, for the \textit{forward} problem, we describe algorithms for soundly propagating the md-vtree forward under each basic circuit operation, as mentioned in Section \ref{sec:forward}. Then, for the \textit{backward} problem, we analyze these algorithms to prove the MD-calculus for propagating marginal determinisms backwards through operations. 

\input{figs/md_table_repeated}

\subsection{Algorithms and the Forward Problem}

For each of the basic operations, there exist efficient (polynomial time) algorithms for computing them on probabilistic circuits satisfying the requirement column in Table \ref{tbl:mdcalc_repeated} \citepSM{choi2020probcirc,vergari2021atlas}.
In this section, we will also describe, for each basic operation, an algorithm for computing the operation on md-vtrees $\mdvtree$, that is a sound abstraction of the corresponding algorithm on circuits. By \textit{sound}, we mean that, given an input md-vtree $\mdvtree$ and the output of the md-vtree algorithm $\mdvtree'$, it is guaranteed for any input PC respecting $\mdvtree$, the output of the corresponding PC algorithm will respect $\mdvtree'$. 

The construction of these md-vtrees algorithms is based upon the corresponding PC algorithm. Thus, we present the algorithms as applying to both the md-vtree and PC. For convenience, we assume that the PC satisfies the following condition, which we call \textit{exactly respecting} a md-vtree:

\begin{definition}[PC exactly respecting md-vtree]
A PC $\circuit$ exactly respects a md-vtree $\mdvtree$ if (1) it respects $\mdvtree$ and (2) the children of any sum node $\sumunit$ corresponding to a non-leaf vtree node $\vnode$ are all product nodes $\productunit$, where $\productunit$ has two children which are sum nodes, each corresponding to a child of $\vnode$.
\end{definition}
 PC architectures are typically designed with these alternating sum and product nodes, where the product nodes are binary; for example, both MDNets and PSDDs satisfy this property. Further, any PC which respects a md-vtree can be transformed into an equivalent PC which exactly respects the md-vtree, as follows. For every sum node $\sumunit$ which has a sum node child $\sumunit'$, we can directly attach the children of the $\sumunit'$ to $\sumunit$ (with the appropriate combination of weights). Then, for every product node $\productunit$ which has a product node child $\productunit'$, we can replace $\productunit'$ with a new single-child sum node $\sumunit'$, which has $\productunit'$ as its child. The resulting circuit still encodes the same function, and has the same marginal determinisms as the original circuit.

Given that a PC exactly respects a md-vtree, for each non-leaf vtree node $\vnode$, we can represent the corresponding PC layer simply as a vector of sum nodes $\sumunits_{\vnode}$ with length $\regionsize_{\vnode} := |\sumunits_{\vnode}| $, and a weight/parameter matrix $\pcparam_\vnode$ with shape $(\regionsize_{\vnode}, \regionsize_{\vnode_1}, \regionsize_{\vnode_2})$, with the semantics that $\dist_{\sumunit_{\vnode, i}}(\pcvars) = \sum_{jk} \pcparam_{\vnode, ijk} \dist_{\sumunit_{\vnode_1, j}}(\pcvars) \dist_{\sumunit_{\vnode_2, k}}(\pcvars)$ (where $\vnode_1, \vnode_2$ are the children of $\vnode$). Note that for any pair of sum nodes $\sumunit_{\vnode_1, j}, \sumunit_{\vnode_2, k}$ for which there isn't a product in the PC connected to $\sumunit_{\vnode, i}$, we can simply set the weight $\pcparam_{\vnode, ijk}$ to zero\footnote{While this is sufficient to represent any PC exactly respecting an md-vtree, it may be inefficient to represent $\pcparam_{\vnode, ijk}$ as a dense matrix if the connections in the PC are sparse, i.e. $\pcparam_{\vnode, ijk} = 0$ for many $i,j,k$. In the evaluation of a sum vector as a function of its child sum vectors, we only require the sum $\sum_{jk} \pcparam_{\vnode, ijk} \pcfunc_{\sumunit_{\vnode_1, j}} \pcfunc_{\sumunit_{\vnode_2, k}}$ to be computed, so this can be implemented in a sparse manner if that is more appropriate. Similar reasoning applies to the product algorithm.}. For leaf vtree nodes, the corresponding layer can consist of both sum and leaf nodes (with the sum nodes being mixtures over the leaf nodes, e.g. $0.7 \mathds{1}_{X = 0} + 0.3 \mathds{1}_{X = 1}$). In this case, we represent the sum and leaf nodes as a vector $\sumunits_{\vnode}$ with length $\regionsize_{\vnode} := |\sumunits_{\vnode}|$, and $\pcparam_\vnode$, and a weight matrix $\pcparam_{\vnode}$, with $\pcparam_{\vnode, ij} > 0$ iff $\sumunit_{\vnode, j}$ is a leaf node that is a child of sum node $\sumunit_{\vnode, i}$.

This characterization of a PC as a pair $\vparamfn(\vnode) := (\sumunits_{\vnode}, \pcparam_\vnode)$ for each vtree node
, which we call the \emph{parameter function}, allows us to efficiently describe algorithms for the basic operations. Thus, in the algorithms below we will represent $\circuit$ exactly respecting some md-vtree using the triple $\circuit = (\vtree, \vlabel, \vparamfn)$, where $\vtree$ is the vtree, $\vlabel$ is the labelling function, and $\vparamfn$ the parameter function.

\paragraph{$\marg(\cdot; \varsubset)$} The marginalization algorithm is depicted in Algorithm \ref{alg:marg}. For the marginalization operation, we can take advantage of the fact that marginalization commutes with both product and sum nodes in a decomposable and smooth PC (which is the basis of tractable marginal inference).
\begin{equation} \sum_{\varsubset} \dist_{\productunit}(\scope(\productunit)) = \sum_{\varsubset} \dist_{\genericunit_1}(\scope(\genericunit_1)) \dist_{\genericunit_2}(\scope(\genericunit_2)) = (\sum_{\varsubset} \dist_{\genericunit_1}(\scope(\genericunit_1))) (\sum_{\varsubset} \dist_{\genericunit_2}(\scope(\genericunit_2)))\end{equation}
\begin{equation} \sum_{\varsubset} \dist_{\sumunit}(\scope(\sumunit)) = \sum_{\varsubset} \sum_{\genericunit_i \in \children(\sumunit)} \pcparam_i \dist_{\genericunit_i}(\scope(\genericunit_i)) = \sum_{\genericunit_i \in \children(\sumunit)} \pcparam_i (\sum_{\varsubset} \dist_{\genericunit_i}(\scope(\genericunit_i)))\end{equation}
where the last equality on the first line holds because $\scope(\genericunit_1) \cap \scope(\genericunit_2) = \emptyset$ by decomposability. This means that, in order to marginalize a circuit, we simply need to marginalize the leaf nodes. In Algorithm \ref{alg:marg}, we show the (recursive) procedure of marginalizing a circuit represented as $(\vtree, \vlabel, \vparamfn)$. In lines \ref{algline:margleafstart}-\ref{algline:margleafend} we marginalize out $\varsubset$ from the leaf nodes in the PC, and in lines \ref{algline:margnonleafstart}-\ref{algline:margnonleafend}, we handle non-leaf vtree nodes simply by copying the existing circuit. To update the md-vtree, in line \ref{algline:margscope}, we update the scope of the vtree node, removing the marginalized variables, and in lines \ref{algline:marglabelstart}-\ref{algline:marglabelend}, we assign a label to the new vtree node $\vnode'$.  

The new label is justified as follows. Suppose we have a sum node $\sumunit \in \sumunit_{\vnode}$, with children $\genericunit_1, ..., \genericunit_n$; by definition, $\sumunit$ is marginally deterministic with respect to $\vlabel(\vnode)$. After marginalization, the function encoded by each child $\genericunit_i'$ satisfies $\dist_{\genericunit_i'}(\margdetvarsval) = \sum_{\varsubset} \dist_{\genericunit_i}(\margdetvarsval)$ for any value $\margdetvarsval$ of $\vlabel(\vnode)$ by definition. Now:
\begin{itemize}
    \item If $\vlabel(\vnode) \cap \varsubset = \emptyset$, then this is just proportional to $\dist_{\genericunit_i}(\margdetvarsval)$ and so the marginalized support will remain the same for each child, and $\sumunit'$ will maintain $\vlabel(\vnode)$-determinism.
    \item On the other hand, if $\vlabel(\vnode) \cap \varsubset \neq \emptyset$, then we do not have any such guarantee; in fact, we cannot be sure that $\sumunit'$ will be $\margdetvars$-deterministic for any $\margdetvars$, so we assign the universal set.
\end{itemize} 

\input{algs/marg}

\paragraph{$\inst(\cdot; \varsubsetval)$} For the instantiation operation, we have Algorithm \ref{alg:inst}. At first glance, this seems to be very similar to the marginalization operation; it changes the scope in the same way, and the changes to the circuit can be implemented through the leaf nodes. However, the crucial difference is in the label function. 

The new label of $\vlabel'(\vnode) = \vlabel(\vnode) \setminus \varsubset$ is justified as follows. Suppose that we have a sum node $\sumunit \in \sumunit_{\vnode}$, with children $\genericunit_1, ..., \genericunit_n$, with $\sumunit$  marginally deterministic with respect to $\vlabel(\vnode)$. After instantiation (of $\varsubset$ with the value $\varsubsetval$), the function encoded by each child $\genericunit_i'$ satisfies $\dist_{\genericunit_i'}(\margdetvarsval \setminus \varsubset) = \dist_{\genericunit_i}(\varsubsetval, \margdetvarsval \setminus \varsubset)$, for any value $\margdetvarsval$ of $\vlabel(\vnode)$ by definition\footnote{Note that we write $\margdetvarsval \setminus \varsubset$ to represent the value of $\vlabel(\vnode) \setminus \varsubset$ given by $\margdetvarsval$ restricted to this variable set.}.
 
Now, we claim that $\genericunit_i'$ is $(\vlabel(\vnode) \setminus \varsubset)$-deterministic, i.e. $\genericunit_i', \genericunit_j'$ have distinct marginalized support $\supp_{\vlabel(\vnode) \setminus \varsubset}(\genericunit_i')$, $\supp_{\vlabel(\vnode) \setminus \varsubset}(\genericunit_j')$ for $i \neq j$. Suppose for contradiction there exists a value $\margdetvarsval^* \setminus \varsubset$ of $(\vlabel(\vnode) \setminus \varsubset)$ such that 
$\dist_{\genericunit_i'}(\margdetvarsval^* \setminus \varsubset) > 0$ and $\dist_{\genericunit_j'}(\margdetvarsval^* \setminus \varsubset) > 0$. Then we have
\begin{align}
    \dist_{\genericunit_i'}(\margdetvarsval^* \setminus \varsubset) > 0 \textnormal{ and } \dist_{\genericunit_j'}(\margdetvarsval^* \setminus \varsubset) > 0 \\
    \dist_{\genericunit_i}(\varsubsetval, \margdetvarsval^* \setminus \varsubset) > 0 \textnormal{ and } \dist_{\genericunit_j}(\varsubsetval, \margdetvarsval^* \setminus \varsubset) > 0 \\
    \sum_{\varsubset \setminus \margdetvars} \dist_{\genericunit_i}(\varsubset \setminus \margdetvars, \varsubsetval \cap \margdetvars, \margdetvarsval^* \setminus \varsubset) > 0 \textnormal{ and } \sum_{\varsubset \setminus \margdetvars} \dist_{\genericunit_j}(\varsubset \setminus \margdetvars, \varsubsetval \cap \margdetvars,  \margdetvarsval^* \setminus \varsubset) > 0 \\
    \dist_{\genericunit_i}(\varsubsetval \cap \margdetvars, \margdetvarsval^* \setminus \varsubset) > 0 \textnormal{ and } \ \dist_{\genericunit_j}(\varsubsetval \cap \margdetvars, \margdetvarsval^* \setminus \varsubset) > 0
\end{align}

The second line follows by definition of the instantiation algorithm, the third line is a sum of non-negative terms including a positive term from the previous line (when $\varsubset \setminus \margdetvars = \varsubsetval \setminus \margdetvars$), and the fourth line rewrites the sum. Now we have a value $\margdetvarsval := (\varsubsetval \cap \margdetvars, \margdetvarsval^* \setminus \varsubset)$ of $\vlabel(\vnode)$, such that $\dist_{\genericunit_i}(\margdetvarsval) > 0$ and $\dist_{\genericunit_j}(\margdetvarsval) > 0$, which is a contradiction as $\sumunit$ is $\vlabel(\vnode)$-deterministic.

\input{algs/inst}

\paragraph{$\product(\cdot, \cdot)$} Now, we consider the product of two circuits exactly respecting \textit{compatible} vtrees. 

\begin{definition}[Vtree Compatibility]
    Let $\vtree^{(1)} = (\vnodes^{(1)}, \vedges^{(1)}, \scope^{(1)})$ and $ \vtree^{(2)} = (\vnodes^{(2)}, \vedges^{(2)}, \scope^{(2)})$ be two vtrees, with root nodes $\vnode^{(1)}, \vnode^{(2)}$ respectively. Define $\commonvars := \scope^{(1)}(\vnode^{(1)}) \cup \scope^{(2)}(\vnode^{(2)})$ to be the common variables.
    Then we say that $\vtree^{(1)}, \vtree^{(2)}$ are compatible if any of the following hold:
    \begin{enumerate}
        \item There are no common variables, $\commonvars = \emptyset$;
        \item Both of $\vnode^{(1)}, \vnode^{(2)}$ are leaf vtree nodes;
        \item One of the root nodes has the same restricted scope on $\commonvars$ as one of the children of the other root node, \textbf{and} the vtrees rooted at these nodes are compatible. For example, $\scope^{(1)}_{\commonvars}(\vnode^{(2)}) = \scope^{(1)}_{\commonvars}(\vnode^{(1)}_r)$, and $\vtree^{(2)}_{\vnode^{(2)}}$ and $\vtree^{(1)}_{\vnode^{(1)}_r}$ are compatible.
        \item The children of the root nodes have matching restricted scopes, \textbf{and} are compatible. For example, $\scope^{(1)}_{\commonvars}(\vnode^{(1)}_l) = \scope^{(2)}_{\commonvars}(\vnode^{(2)}_l)$ and $\vtree^{(1)}_{\vnode^{(1)}_l}, \vtree^{(2)}_{\vnode^{(2)}_l}$ are compatible, and $\scope^{(1)}_{\commonvars}(\vnode^{(1)}_r) = \scope^{(2)}_{\commonvars}(\vnode^{(2)}_r)$ and $\vtree^{(1)}_{\vnode^{(1)}_r}, \vtree^{(2)}_{\vnode^{(2)}_r}$ are compatible.
    \end{enumerate}
\end{definition}

This recursive definition allows for products of circuits not necessarily respecting the same vtree, but merely vtrees which ``essentially have the same structure'' over the shared variables. 
Intuitively, there are four cases that allow us to maintain (structured) decomposability in the output circuit, illustrated in Figure \ref{fig:compat}. The first two are base cases where the product is directly tractable: namely, when the root vtree nodes have disjoint scopes, or when they are both leaves. Note, in particular, that the product of a leaf region and a non-leaf region that have overlapping variables is considered intractable here (unless condition 3. holds). The last two are cases where we can recursively call the product algorithm on the children of the root vtree nodes.  Each of the four cases above correspond to a slightly different algorithm for computing the product of the corresponding PC sum nodes, which we will explain next.\footnote{The notion of compatibility between vtrees is somewhat similar to the notion of compatibility between circuits \citep{vergari2021atlas}, but acts at the level of groups of circuit nodes with the same scope (i.e. a vtree node) rather than individual circuit nodes. }

\input{figs/compatibility_diagram}

The product algorithm, depicted in Algorithm \ref{alg:prod}, (recursively) constructs a circuit $\circuit = (\vtree', \vlabel', \vparamfn')$ that is the product of the input circuits $\circuit^{(1)} = (\vtree^{(1)}, \vlabel^{(1)}, \vparamfn^{(1)})$ and $\circuit^{(2)} = (\vtree^{(2)}, \vlabel^{(2)}, \vparamfn^{(2)})$ respectively. In particular, at each recursive step, it computes the root node $\vnode'$ of the new vtree, its label $\vlabel'(\vnode')$, and the parameter function value $\vparamfn'(\vnode') = (\sumunits_{\vnode'}, \pcparam_{\vnode'})$ of that node\footnote{Consider the root nodes of the input vtrees $\vnode^{(1)}, \vnode^{(2)}$, and their parameter function values $\vparamfn^{(1)}(\vnode^{(1)}) = (\sumunits_{\vnode^{(1)}}, \pcparam_{\vnode^{(1)}}), \vparamfn^{(2)}(\vnode^{(2)}) = (\sumunits_{\vnode^{(2)}}, \pcparam_{\vnode^{(2)}})$.  In every case, $\sumunits_{\vnode'}$ will contain one node $\sumunit_{\vnode', i^{(1)}, i^{(2)}}$ corresponding to every \emph{pair of nodes} $\sumunit_{\vnode^{(1)}, i^{(1)}} \in \sumunits_{\vnode^{(1)}}$, $\sumunit_{\vnode^{(2)}, i^{(2)}} \in \sumunits_{\vnode^{(2)}}$; we thus notate it with two dimensions. Correspondingly, in general, $\pcparam_{\vnode'}$ will contain one weight $\pcparam_{\vnode', i^{(1)}i^{(2)}j^{(1)}j^{(2)}k^{(1)}k^{(2)}}$ for every \emph{combination of nodes} $\sumunit_{\vnode', i^{(1)}, i^{(2)}} \in \sumunits_{\vnode'}, \sumunit_{\vnode'_l, j^{(1)}, j^{(2)}} \in \sumunits_{\vnode'_l}, \sumunit_{\vnode'_r, k^{(1)}, k^{(2)}} \in \sumunits_{\vnode'_r}$, where $\vnode'_l$ and $\vnode'_r$ are the left and right children of $\vnode'$ in the \emph{new} md-vtree.}. 
We consider each of the four compatibility cases separately:

\begin{enumerate}
    \item Firstly, if there are no common variables, i.e. $\commonvars = \emptyset$, then we can simply introduce product nodes for each pair of sum nodes, while maintaining decomposability, as in Figure \ref{fig:pmd_1}.
    \begin{itemize}
        \item \emph{Vtree node:} We create a vtree node $\vnode'$ with $(\vnode'_l = \vnode^{(1)}$, $\vnode'_r = \vnode^{(2)})$ as children.
        \item \emph{Parameter function:} The parameter function value for the new node $\vparamfn'(\vnode') = (\sumunits_{\vnode'}, \pcparam_{\vnode'})$ is given as follows. For each pair of sum nodes $\sumunit_{\vnode^{(1)}, j} \in \sumunits_{\vnode^{(1)}}, \sumunit_{\vnode^{(2)}, k} \in \sumunits_{\vnode^{(2)}}$, we create a sum node $\sumunit_{\vnode', i^{(1)}i^{(2)}}$, representing the product of $\sumunit_{\vnode^{(1)}, j}, \sumunit_{\vnode^{(2)}, k}$. To achieve this, the weights $\pcparam_{\vnode', i^{(1)}i^{(2)}jk}$ are defined to be $1$ if $i^{(1)}=j$ and $i^{(2)}=k$, and $0$ otherwise. Note that $j, k$ only have a single index as the children $\vnode'_l, \vnode'_r$ correspond to the input circuits, which only have a single dimension.
        \item \emph{Md-vtree labelling:} The label is set to be $\vlabel'(\vnode') := \emptyset$; this is since all sum nodes only effectively have a single child, so they are trivially $\margdetvars$-deterministic for any $\margdetvars$. An example can be seen in Figure \ref{fig:pmd_1}.
    \end{itemize}
    \item Secondly, if both $\vnode^{(1)}, \vnode^{(2)}$ are leaves, then the PC nodes corresponding to these vtree nodes are also either leaves, or simple mixtures (sum nodes) of leaves. To compute the product of two sum nodes, we expand all combinations of the children of the sum nodes, as shown in Figure \ref{fig:pmd_3}.
        \begin{itemize}
            \item \emph{Vtree node:} We create a leaf vtree node $\vnode'$.
            \item \emph{Parameter function:} The parameter function value for the new node $\vparamfn'(\vnode') = (\sumunits_{\vnode'}, \pcparam_{\vnode'})$ is given as follows. For each pair of nodes $\genericunit_{\vnode^{(1)}, j} \in \sumunits_{\vnode^{(1)}}, \genericunit_{\vnode^{(2)}, k} \in \sumunits_{\vnode^{(2)}}$, we create a sum/leaf node $\genericunit_{\vnode', i^{(1)}i^{(2)}}$, representing the product of $\genericunit_{\vnode^{(1)}, j}, \genericunit_{\vnode^{(2)}, k}$. The weights are defined as $\pcparam_{\vnode', i^{(1)}, i^{(2)}, j^{(1)}, j^{(2)}} := \pcparam_{\vnode', i^{(1)}, j^{(1)}} \pcparam_{\vnode', i^{(2)}, j^{(2)}}$.
            \item \emph{Md-vtree labelling:} The label is set to be $\vlabel(\vnode') := \vlabel^{(1)}(\vnode^{(1)}) \cup \vlabel^{(2)}(\vnode^{(2)})$. This is best seen with an example; in Figure \ref{fig:pmd_3}, we see an example of the product of two sum nodes with leaf node children, where one node is $A$-deterministic and the other is $B$-deterministic. The resulting sum node in the output circuit has children corresponding to (the product of) each combination of the children of the original two sum nodes; as a result, each child of the output sum node corresponds to a different value of $(A, B)$, and so the output sum node is $\{A, B\}$-deterministic.
        \end{itemize}
    \item Thirdly, if one of the nodes has the same restricted scope as a child of the other node, e.g. $\scope^{(1)}_{\commonvars}(\vnode^{(2)}) = \scope^{(1)}_{\commonvars}(\vnode^{(1)}_r)$, we ``defer'' the product as shown in Figure \ref{fig:pmd_2}. 
    \begin{itemize}
        \item \emph{Vtree node:} We create a vtree node $\vnode'$ with $(\vnode'_l = \vnode^{(1)}_l$, $\vnode'_r = \product(\vtree^{(1)}_{\vnode^{(1)}_r}, \vtree^{(2)}_{\vnode^{(2)}}))$ as children.
        \item \emph{Parameter function:} The parameter function value for the new node $\vparamfn'(\vnode') = (\sumunits_{\vnode'}, \pcparam_{\vnode'})$ is given as follows. For every pair of sum nodes $\sumunit_{\vnode^{(1)}, i^{(1)}} \in \sumunits_{\vnode^{(1)}}$, $\sumunit_{\vnode^{(2)}, i^{(2)}} \in \sumunits_{\vnode^{(2)}}$, we create a sum node $\sumunit_{\vnode', i^{(1)}i^{(2)}}$. The weights are defined as $\pcparam_{\vnode', i^{(1)}i^{(2)}j^{(1)}k^{(1)}k^{(2)}} := \pcparam_{\vnode^{(1)}, i^{(1)}j^{(1)}k^{(1)}} \mathds{1}_{i^{(2)} = k^{(2)}}$. Note that the index $j$ only has a single index as the left child $\vnode'_l = \vnode^{(1)}_l$, and so the sum nodes $\sumunits_{\vnode'_l}$ are copies of the sum nodes from $\sumunits_{\vnode^{(1)}_l}$.
        \item \emph{Md-vtree labelling:} The label is set to be $\vlabel'(\vnode') := \vlabel^{(1)}(\vnode^{(1)})$, as the marginalized support of the children of the output sum nodes is a subset of the marginalized support of the corresponding sum node from $\vnode^{(1)}$, as can be seen in Figure \ref{fig:pmd_2}.
    \end{itemize}
    \item Finally, in any other case, the children of $\vnode^{(1)}, \vnode^{(2)}$ have matching restricted scopes, e.g. $\scope^{(1)}_{\newcommonvars_{12}}(\vnode^{(1)}_l) = \scope^{(2)}_{\newcommonvars_{12}}(\vnode^{(2)}_l)$ and $\scope^{(1)}_{\commonvars}(\vnode^{(1)}_r) = \scope^{(2)}_{\commonvars}(\vnode^{(2)}_r)$, we expand all combinations:
    \begin{itemize}
        \item \emph{Vtree node:} We create a vtree node $\vnode'$ with $(\vnode'_l = \product(\vnode^{(1)}_l, \vnode^{(2)}_l)$, $\vnode'_r = \product(\vnode^{(1)}_r, \vnode^{(2)}_r))$ as children.
        \item \emph{Parameter function:} The parameter function value for the new node $\vparamfn'(\vnode') = (\sumunits_{\vnode'}, \pcparam_{\vnode'})$ is given as follows. For every pair of sum nodes $\sumunit_{\vnode^{(1)}, i^{(1)}} \in \sumunits_{\vnode^{(1)}}$, $\sumunit_{\vnode^{(2)}, i^{(2)}} \in \sumunits_{\vnode^{(2)}}$, we create a sum node $\sumunit_{\vnode', i^{(1)}i^{(2)}}$. The weights are defined as $\pcparam_{\vnode', i^{(1)}i^{(2)}j^{(1)}j^{(2)}k^{(1)}k^{(2)}} := \pcparam_{\vnode^{(1)}, i^{(1)}j^{(1)}k^{(1)}} \pcparam_{\vnode^{(1)}, i^{(2)}j^{(2)}k^{(2)}}$.
        \item \emph{Md-vtree labelling:} The label is set to be $\vlabel(\vnode') := \vlabel^{(1)}(\vnode^{(1)}) \cup \vlabel^{(2)}(\vnode^{(2)})$, for similar reasons to the product of two leaf vtree nodes above.
    \end{itemize}
\end{enumerate}

\input{figs/product_md_diagram}

With this, we have shown how to each recursive step of the product algorithm. Now, taking a step back, we consider the entire run of the recursive algorithm. Starting from md-vtrees $\vtree^{(1)}, \vtree^{(2)}$ over variables $\pcvars^{(1)}, \pcvars^{(2)}$, with common variables $\commonvarsglobal := \pcvars^{(1)} \cap \pcvars^{(2)}$, in each recursive call, we reduce the common variables, until either we reach two vtree nodes that do not have common variables, or we reach leaf vtree nodes. One property of the algorithm, which will be important for the proof of the MD-calculus rule below, is that at each recursive step of the product algorithm, the two input vtree nodes have the same restricted scope over $\commonvarsglobal$. 

\begin{proposition}\label{prop:algprodinv}
    At each recursive step of Algorithm \ref{alg:prod}, we have that $\scope^{(1)}_{\commonvarsglobal}(\vnode^{(1)}) = \scope^{(2)}_{\commonvarsglobal}(\vnode^{(2)})$.
\end{proposition}

\begin{proof}
    Proof is by inspection; in each recursive case (3, 4), we have that $\commonvars = \scope^{(1)}_{\commonvarsglobal}(\vnode^{(1)}) = \scope^{(2)}_{\commonvarsglobal}(\vnode^{(2)})$, and match up the common variables among the recursive call(s). 
\end{proof}

\input{algs/prod}

\paragraph{$\powop(\cdot; \power)$} For the power operation, we have Algorithm \ref{alg:powop}. This algorithm simply inverts all the weights/parameters of the circuit, as well as replacing the leaves with their reciprocals. Provided that the input circuit is deterministic, the output circuit faithfully represents the reciprocal of the input circuit \citep{vergari2021atlas}. As the transformation is simply numerical (i.e. not affecting the support of any node), the labels of all nodes remain the same.

\paragraph{$\maxop(\cdot; \power)$} This operation returns a scalar.

\paragraph{$\logop(\cdot)$} This operation returns a circuit which respects the same vtree, but does not have (marginal) determinism \citep{vergari2021atlas}.

\input{algs/pow}

\subsection{MD-calculus and the Backward Problem}

The algorithms for each of the basic operations above allow us to derive a md-vtree for the output circuit, given the md-vtree that the input circuit respects. The MD-calculus in Table \ref{tbl:mdcalc} (repeated for convenience in Table \ref{tbl:mdcalc_repeated}) turns these results into a series of straightforward rules that can easily be applied to derive sufficient (but possibly not necessary) conditions for tractability of compositions of operations.

\thmmdcalc*
\begin{proof} To state the result more formally, we claim that if the input circuit(s) respect md-vtrees(s) implying the input condition, then the result of the operation applied to the input circuit(s) will respect a md-vtree implying the output condition.

\paragraph{$\marg(\cdot; \varsubset)$}  For the marginalization operation, the output md-vtree is over variables $\vars \setminus \varsubset$. Thus, let $\margdetvars$ be any subset of $\vars \setminus \varsubset$.
\begin{itemize}
    \item \emph{Input Condition:} The input condition requires that the input md-vtree $\mdvtree$ implies $\margdetvars$-determinism; that is, for every vtree node $\vnode$, either $\scope_{\margdetvars}(\vnode) = \emptyset$, or else $\margdetvars \supseteq \vlabel(\vnode)$.
    \item \emph{Algorithm:} In Algorithm \ref{alg:marg}, every vtree node $\vnode'$ in the output md-vtree corresponds to a vtree node $\vnode'$ in the input md-vtree, such that $\scope'(\vnode') = \scope(\vnode) \setminus \varsubset$, and $\vlabel'(\vnode') = \vlabel(\vnode)$ if $\vlabel(\vnode) \cap \varsubset = \emptyset$, or $\vlabel'(\vnode')= \vlabelno$ otherwise.
    \item \emph{Proof for Output Condition:} For each vtree node $\vnode'$, if  $\scope'_{\margdetvars}(\vnode') \neq \emptyset$, then, we have that:
    \begin{align*}
        &\scope_{\margdetvars}(\vnode) \setminus \varsubset \neq \emptyset \tag*{(by effect of algorithm)}\\
        \implies &\scope_{\margdetvars}(\vnode) \neq \emptyset \tag*{(weaker statement)}\\
        \implies &\vlabel(\vnode) \subseteq \margdetvars \tag*{(by input condition)} \\
        \implies &\vlabel'(\vnode') \subseteq \margdetvars 
    \end{align*}
    The last line follows since $\margdetvars \cap \varsubset = \emptyset$, so $\vlabel(\vnode) \cap \varsubset = \emptyset$, and so we are in the algorithm case where the label is "copied". Thus, we have shown that the output md-vtree implies $\margdetvars$-determinism, as required.
\end{itemize}

\paragraph{$\inst(\cdot; \varsubsetval)$} For the instantiation operation, the output md-vtree is over variables $\vars \setminus \varsubset$. Thus, let $\margdetvars$ be any subset of $\vars \setminus \varsubset$.
\begin{itemize}
    \item \emph{Input Condition:} The input condition requires that the input md-vtree $\mdvtree$ implies $(\margdetvars \cup \varsubset')$-determinism for some $\varsubset' \subseteq \varsubset$; that is, for every vtree node $\vnode$, either $\scope_{\margdetvars \cup \varsubset'}(\vnode) = \emptyset$, or else $\margdetvars \cup \varsubset' \supseteq \vlabel(\vnode)$.
    \item \emph{Algorithm:} In Algorithm \ref{alg:inst}, every vtree node $\vnode'$ in the output md-vtree corresponds to a vtree node $\vnode'$ in the input md-vtree, such that $\scope'(\vnode') = \scope(\vnode) \setminus \varsubset$, and $\vlabel'(\vnode') = \vlabel(\vnode) \setminus \varsubset$.
    \item \emph{Proof for Output Condition:} For each vtree node $\vnode'$, if  $\scope'_{\margdetvars}(\vnode') \neq \emptyset$, then, we have that:
    \begin{align*}
        &\scope_{\margdetvars}(\vnode) \setminus \varsubset \neq \emptyset \tag*{(by effect of algorithm)}\\
        \implies &\scope_{\margdetvars}(\vnode) \neq \emptyset \tag*{(weaker statement)}\\
        \implies &\scope_{\margdetvars \cup \varsubset'}(\vnode) \neq \emptyset \tag*{(weaker statement)}\\
        \implies &\vlabel(\vnode) \subseteq \margdetvars \cup \varsubset' \tag*{(by input condition)} \\
        \implies &\vlabel'(\vnode') \subseteq \margdetvars 
    \end{align*}
    Here, the last line follows since the new label $\vlabel'(\vnode') = \vlabel(\vnode) \setminus \varsubset$ removes all elements of $\varsubset$, and thus $\varsubset'$, from $\vlabel(\vnode)$. Thus, we have shown that the output md-vtree implies $\margdetvars$-determinism, as required.
\end{itemize}

\paragraph{$\product(\cdot, \cdot)$} For the product operation, the output md-vtree is over variables $\vars^{(1)} \cup \vars^{(2)}$. Thus, let $\margdetvars$ be any subset of $\vars^{(1)} \cup \vars^{(2)}$.

\begin{itemize}
    \item \emph{Input Condition:} The input condition requires that the first input md-vtree $\mdvtree^{(1)}$ implies $\margdetvars^{(1)}$-determinism, and the second input md-vtree $\mdvtree^{(2)}$ implies $\margdetvars^{(2)}$-determinism, where \textit{one} of the following holds:
    \begin{enumerate}[(a)]
        \item $\margdetvars \subseteq \pcvars^{(1)} \cap \pcvars^{(2)}$ and $\margdetvars^{(1)} = \margdetvars^{(2)} = \margdetvars$;
        \item $\margdetvars^{(1)}, \margdetvars^{(2)} \supseteq \pcvars^{(1)} \cap \pcvars^{(2)}$ and $\margdetvars = \margdetvars^{(1)} \cup \margdetvars^{(2)}$
    \end{enumerate}
    \item \emph{Algorithm:} In Algorithm \ref{alg:prod}, every vtree node $\vnode'$ in the output md-vtree corresponds to a pair $\vnode^{(1)}, \vnode^{(2)}$ in the input md-vtrees respectively, such that $\scope'(\vnode') = \scope^{(1)}(\vnode^{(1)}) \cup \scope^{(2)}(\vnode^{(2)})$. There are four cases of the algorithm to consider, in which the label is:
    \begin{enumerate}
        \item $\vlabel'(\vnode') = \emptyset$.
        \item $\vlabel'(\vnode') = \vlabel^{(1)}(\vnode^{(1)}) \cup \vlabel^{(2)}(\vnode^{(2)})$
        \item $\vlabel'(\vnode') = \vlabel^{(1)}(\vnode^{(1)})$
        \item $\vlabel'(\vnode') = \vlabel^{(1)}(\vnode^{(1)}) \cup \vlabel^{(2)}(\vnode^{(2)})$
    \end{enumerate}
    \item \emph{Proof for Output Condition:} We need to show that for each case 1-3 of the algorithm, and for either input condition (a), (b), that the condition for implied $\margdetvars$-determinism holds on $\vnode'$; that is, if $\scope'_{\margdetvars}(\vnode') \neq \emptyset$, then $\vlabel'(\vnode') \subseteq \margdetvars$. Assuming that $\scope'_{\margdetvars}(\vnode') \neq \emptyset$, we have that 
    \begin{align*}
        &\scope'_{\margdetvars}(\vnode') \neq \emptyset \\
        \implies &\scope'(\vnode') \cap \margdetvars \neq \emptyset \tag*{(by definition of restricted scope)} \\
        \implies &(\scope^{(1)}(\vnode^{(1)}) \cup \scope^{(2)}(\vnode^{(2)})) \cap \margdetvars \neq \emptyset \tag*{(by effect of algorithm)} \\
        \implies &\scope^{(1)}_{\margdetvars}(\vnode^{(1)}) \cup \scope^{(2)}_{\margdetvars}(\vnode^{(2)}) \neq \emptyset \tag*{(rewriting)}
    \end{align*}
    However, this does not in general imply that $\scope^{(1)}_{\margdetvars^{(1)}}(\vnode^{(1)}) \neq \emptyset$ or $\scope^{(2)}_{\margdetvars^{(2)}}(\vnode^{(2)}) \neq \emptyset$. Thus, we look at the special cases defined by (a) and (b), and the algorithm variations 1, 2, 3, 4.
    \begin{itemize}[align=left]
        \item [(a1, a2, a3, a4)] In case (a), we have $\margdetvars^{(1)} = \margdetvars^{(2)} = \margdetvars \subseteq \commonvarsglobal$. 
        \begin{align*}
            &\scope^{(1)}_{\commonvarsglobal}(\vnode^{(1)}) = \scope^{(2)}_{\commonvarsglobal}(\vnode^{(2)}) \tag*{(by Proposition \ref{prop:algprodinv})} \\
            \implies &\scope^{(1)}_{\commonvarsglobal}(\vnode^{(1)}) \cap \margdetvars = \scope^{(2)}_{\commonvarsglobal}(\vnode^{(2)}) \cap \margdetvars \\
            \implies &\scope^{(1)}_{\commonvarsglobal \cap \margdetvars}(\vnode^{(1)}) = \scope^{(2)}_{\commonvarsglobal \cap \margdetvars}(\vnode^{(2)}) \\
            \implies &\scope^{(1)}_{\margdetvars}(\vnode^{(1)}) = \scope^{(2)}_{\margdetvars}(\vnode^{(2)}) \tag{as $\margdetvars \subseteq \commonvarsglobal$}\\
            \implies &\scope^{(1)}_{\margdetvars}(\vnode^{(1)}) \neq \emptyset, \scope^{(2)}_{\margdetvars}(\vnode^{(2)}) \neq \emptyset \tag{as $\scope^{(1)}_{\margdetvars}(\vnode^{(1)}) \cup \scope^{(2)}_{\margdetvars}(\vnode^{(2)}) \neq \emptyset$}\\
            \implies &\scope^{(1)}_{\margdetvars^{(1)}}(\vnode^{(1)}) \neq \emptyset, \scope^{(2)}_{\margdetvars^{(2)}}(\vnode^{(2)}) \neq \emptyset \tag{as $\margdetvars^{(1)} = \margdetvars^{(2)} = \margdetvars$}\\
        \end{align*}
        Thus, we have that $\margdetvars \supseteq \vlabel^{(1)}(\vnode^{(1)})$ and $\margdetvars \supseteq \vlabel^{(2)}(\vnode^{(2)})$, and so $\margdetvars \supseteq \vlabel^{(1)}(\vnode^{(1)}) \cup \vlabel^{(2)}(\vnode^{(2)})$. Finally, in each of the cases 1-4, we have $\margdetvars \supseteq \vlabel'(\vnode')$, so the output md-vtree implies $\margdetvars$-determinism as required.
        \item [(b1)] In case (b), we need to consider the cases of the algorithm separately. In case 1, $\vlabel'(\vnode') \subseteq \margdetvars$ holds trivially as $\vlabel'(\vnode') = \emptyset$, so we are done.
        \item [(b2, b3, b4)] We have that $\margdetvars^{(1)}, \margdetvars^{(2)} \supseteq \commonvarsglobal$ and $\margdetvars = \margdetvars^{(1)} \cup \margdetvars^{(2)}$. The key observation is that, as we are not in case 1 of the Algorithm, $\commonvars = \scope^{(1)}(\vnode^{(1)}) \cap \scope^{(2)}(\vnode^{(2)})$ must be non-empty.
        \begin{align*}
            &\commonvars \neq \emptyset\\
            \implies &\scope^{(1)}_{\commonvars}(\vnode^{(1)}) \neq \emptyset, \scope^{(2)}_{\commonvars}(\vnode^{(2)}) \neq \emptyset \tag{by definition of $\commonvars$} \\
            \implies &\scope^{(1)}_{\margdetvars^{(1)}}(\vnode^{(1)}) \neq \emptyset, \scope^{(2)}_{\margdetvars^{(1)}}(\vnode^{(2)}) \neq \emptyset \tag{as $\margdetvars^{(1)}, \margdetvars^{(2)} \supseteq \commonvarsglobal \supseteq \commonvars$} \\
        \end{align*}
        Thus, we have that $\margdetvars^{(1)} \supseteq \vlabel^{(1)}(\vnode^{(1)})$ and $\margdetvars^{(2)} \supseteq \vlabel^{(2)}(\vnode^{(2)})$, and so $\margdetvars = \margdetvars^{(1)} \cup \margdetvars^{(2)} \supseteq \vlabel^{(1)}(\vnode^{(1)}) \cup \vlabel^{(2)}(\vnode^{(2)})$. In each of the cases 2-4, we have $\margdetvars \supseteq \vlabel'(\vnode')$, so the output md-vtree implies $\margdetvars$-determinism as required.
    \end{itemize}
\end{itemize}

\paragraph{$\powop$} Since the reciprocal algorithm retains the same labelling function in the output as the input, it follows that a $\margdetvars$-deterministic input circuit will result in a $\margdetvars$-deterministic output circuit.

\paragraph{$\maxop$} This operation returns a scalar.

\paragraph{$\logop$} This operation does not have any marginal determinism conditions.

\end{proof}

%% file: figs/md_table_repeated.tex
\begin{table*}[t]
\centering
\begin{tabular}{@{}lllll@{}}
\toprule
\textbf{Operation} & \textbf{Requirement} & \textbf{Input Condition} & \textbf{Output Condition}  \\ \midrule
     $\marg(\circuit; \varsubset)  $            &        -              &      $\margdetvars$-det        &    $\margdetvars$-det                 \\
     $\inst(\circuit; \varsubsetval) $         &       -               &     $\exists \varsubset' \subseteq \varsubset: (\margdetvars \cup \varsubset')\textnormal{-det} $        &   $\margdetvars$-det                    \\ \midrule
    \multirow{3}{*}{$\product(\circuit^{(1)}, \circuit^{(2)})$}     &   $\circuit^{(1)}, \circuit^{(2)}$  respect             &  $\exists \margdetvars^{(1)}, \margdetvars^{(2)}: \margdetvars^{(1)}\textnormal{-det}, \margdetvars^{(2)}\textnormal{-det}$, and:
    
         &      \multirow{3}{*}{$\margdetvars$-det}            \\ 
             &       compatible  vtrees               &   \textbullet\ Either (a) $\margdetvars \subseteq \pcvars^{(1)} \cap \pcvars^{(2)}$ and $\margdetvars^{(1)} = \margdetvars^{(2)} = \margdetvars$;          &                 &                     \\ 
             & & \textbullet\ Or (b) $\margdetvars^{(1)}, \margdetvars^{(2)} \supseteq \pcvars^{(1)} \cap \pcvars^{(2)}$ and $\margdetvars = \margdetvars^{(1)} \cup \margdetvars^{(2)}$ & & \\ \midrule
        $\powop(\circuit; \power)$           &      Det                &       $\margdetvars$-det         &      $\margdetvars$-det             \\
        $\maxop(\circuit)$           &      Det                &       N/A        &      N/A (scalar output)                 \\
        $\logop(\circuit)$          &      Det               &        -        &       -                 \\ \bottomrule
\end{tabular} 
\caption{MD-calculus: sufficient input-output conditions for each basic operation}
\label{tbl:mdcalc_repeated}
\end{table*}

%% file: algs/marg.tex
\begin{algorithm}[t]
\SetAlgoLined
\KwInput{Input circuit $\circuit = (\vtree = (\vnodes, \vedges, \scope), \vlabel, \vparamfn)$; set of variables to be marginalized $\varsubset$
}
\KwResult{Output circuit $\circuit' = (\vtree', \vlabel', \vparamfn')$}

$\vnode \gets \texttt{root}(\vtree)$;

$\vnode' \gets \texttt{newnode}()$;
 
\If(\tcp*[f]{Update vtree structure and parameter function (leaf)}){$\vnode$ is leaf} {  \label{algline:margleafstart}
    $\vtree' \gets \texttt{createvtree}(\vnode')$; \tcp*{create vtree with single node}


    $\vparamfn'(\vnode') \gets (\marg(\leafunit; \varsubset) \textbf{ for } \leafunit \in \sumunits_{\vnode}, \pcparam_{\vnode})$; \label{algline:margleafend} \tcp*{marginalize leaf PC nodes}
    
        
}
\Else(\tcp*[f]{Update vtree structure and parameter function (non-leaf)}){ \label{algline:margnonleafstart}
    $\vnode_l, \vnode_r \gets \texttt{children}(\vnode)$; 
    
    $\vtree'_l, \vlabel'_l, \vparamfn'_l \gets \marg((\vtree_{\vnode_l}, \vlabel, \vparamfn), \varsubset)$; 
    
    $\vtree'_r, \vlabel'_r, \vparamfn'_r \gets \marg((\vtree_{\vnode_r}, \vlabel, \vparamfn), \varsubset)$; 

    $\vtree', \vlabel', \vparamfn' \gets \vtree'_l \cup \vtree'_r, \vlabel'_l \cup \vlabel'_r, \vparamfn'_l \cup \vparamfn'_r$; \tcp*{combine the vtrees/labelling fn/param fn}

    $\vtree' \gets \texttt{addnode}(\vtree'; \vnode'); \vtree' \gets \texttt{addchildren}(\vtree'; \vnode', \texttt{root}(\vtree'_l), \texttt{root}(\vtree'_r))$;
    
    
    $\vparamfn'(\vnode') \gets \vparamfn(\vnode)$; \label{algline:margnonleafend}
}

$\scope'(\vnode') \gets \scope(\vnode) \setminus \varsubset$;  \tcp*{Update scope function} \label{algline:margscope}

\If(\tcp*[f]{Update labelling function}){$\vlabel(\vnode) \cap \varsubset = \emptyset$} {  \label{algline:marglabelstart} 
    $\vlabel'(\vnode') \gets \vlabel(\vnode)$;
}
\Else{
    $\vlabel'(\vnode') \gets \vlabelno$; \label{algline:marglabelend}
}

\textbf{Return} $(\vtree', \vlabel', \vparamfn')$

\caption{$\marg(\circuit, \varsubset)$}
\label{alg:marg}
\end{algorithm}

%% file: algs/inst.tex
\begin{algorithm}[t]
\SetAlgoLined
\KwInput{Input circuit $\circuit = (\vtree = (\vnodes, \vedges, \scope), \vlabel, \vparamfn)$; instantiation $\varsubsetval$ of some subset of variables $\varsubset$
}
\KwResult{Output circuit $\circuit' = (\vtree', \vlabel', \vparamfn')$}

$\vnode \gets \texttt{root}(\vtree)$;

$\vnode' \gets \texttt{newnode}()$;
 
\If(\tcp*[f]{Update vtree structure and parameter function (leaf)}){$\vnode$ is leaf} {  \label{algline:instleafstart}
    $\vtree' \gets \texttt{createvtree}(\vnode')$; \tcp*{create vtree with single node}


    $\vparamfn'(\vnode') \gets (\inst(\leafunit; \varsubsetval) \textbf{ for } \leafunit \in \sumunits_{\vnode}, \pcparam_{\vnode})$; \label{algline:instleafend} \tcp*{instantiate leaf PC nodes}
    
        
}
\Else(\tcp*[f]{Update vtree structure and parameter function (non-leaf)}){ \label{algline:instnonleafstart}
    $\vnode_l, \vnode_r \gets \texttt{children}(\vnode)$; 
    
    $\vtree'_l, \vlabel'_l, \vparamfn'_l \gets \inst((\vtree_{\vnode_l}, \vlabel, \vparamfn), \varsubsetval)$; 
    
    $\vtree'_r, \vlabel'_r, \vparamfn'_r \gets \inst((\vtree_{\vnode_r}, \vlabel, \vparamfn), \varsubsetval)$; 

    $\vtree', \vlabel', \vparamfn' \gets \vtree'_l \cup \vtree'_r, \vlabel'_l \cup \vlabel'_r, \vparamfn'_l \cup \vparamfn'_r$; \tcp*{combine the vtrees/labelling fn/param fn}

    $\vtree' \gets \texttt{addnode}(\vtree'; \vnode'); \vtree' \gets \texttt{addchildren}(\vtree'; \vnode', \texttt{root}(\vtree'_l), \texttt{root}(\vtree'_r))$;
    
    
    $\vparamfn'(\vnode') \gets \vparamfn(\vnode)$; \label{algline:instnonleafend}
}

$\scope'(\vnode') \gets \scope(\vnode) \setminus \varsubset$; \tcp*{Update vtree scope function}

$\vlabel'(\vnode') \gets \vlabel(\vnode) \setminus \varsubset$; \tcp*{Update labelling function}

\textbf{Return} $(\vtree', \vlabel', \vparamfn')$

\caption{$\inst(\circuit, \varsubsetval)$}
\label{alg:inst}
\end{algorithm}

%% file: figs/compatibility_diagram.tex
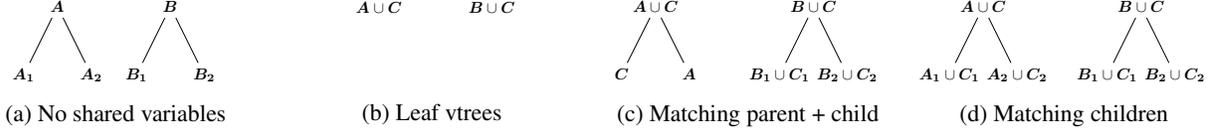
\begin{figure*}[t]
    \centering
    \begin{subfigure}{0.24\linewidth}
        \centering
        \scalebox{0.6}{
        \input{figs/compatibility/c1}
        }
        \caption{No shared variables}
        \label{fig:compat_1}
    \end{subfigure}
    \begin{subfigure}{0.24\linewidth}
        \centering
        \scalebox{0.6}{
        \input{figs/compatibility/c2}
        }
        \caption{Leaf vtrees}
        \label{fig:compat_2}
    \end{subfigure}
    \begin{subfigure}{0.24\textwidth}
        \centering
        \scalebox{0.6}{
        \input{figs/compatibility/c3}
        }
        \caption{Matching parent + child}
        \label{fig:compat_3}
    \end{subfigure}
    \begin{subfigure}{0.24\textwidth}
        \centering
        \scalebox{0.6}{
        \input{figs/compatibility/c4}
        }
        \caption{Matching children}
        \label{fig:compat_4}
    \end{subfigure}
    \caption{Examples of (possibly) compatible vtrees, where $\bm{A} \cap \bm{B} = \emptyset$, and $\bm{C}$ are the shared variables}
    \label{fig:compat}
\end{figure*}

%% file: figs/compatibility/c1.tex
\begin{tikzpicture}

\node[] (A) {$\bm{A}$};
\node[below= of A, xshift=-0.75cm] (A1) {$\bm{A_1}$};
\node[below= of A, xshift=0.75cm] (A2) {$\bm{A_2}$};
\node[xshift = 2.5cm] (B) {$\bm{B}$};
\node[below= of B, xshift=-0.75cm] (B1) {$\bm{B_1}$};
\node[below= of B, xshift=0.75cm] (B2) {$\bm{B_2}$};

\draw[-] (A) -- (A1);
\draw[-] (A) -- (A2);
\draw[-] (B) -- (B1);
\draw[-] (B) -- (B2);

\end{tikzpicture}

%% file: figs/compatibility/c2.tex
\begin{tikzpicture}

\node[] (A) {$\bm{A} \cup \bm{C}$};
\node[xshift = 2.5cm] (B) {$\bm{B} \cup \bm{C}$};

\node[below= of A, xshift=-0.75cm] (A1) {\vphantom{d}}; 
\node[below= of A, xshift=0.75cm] (A2) {};

\end{tikzpicture}

%% file: figs/compatibility/c3.tex
\begin{tikzpicture}

\node[] (A) {$\bm{A} \cup \bm{C}$};
\node[below=of  A, xshift=-0.75cm] (A1) {$\bm{C}$};
\node[below=of  A, xshift=0.75cm] (A2) {$\bm{A}$};
\node[xshift = 3.5cm] (B) {$\bm{B} \cup \bm{C}$};
\node[below= of B, xshift=-0.75cm] (B1) {$\bm{B_1} \cup \bm{C_1}$};
\node[below= of B, xshift=0.75cm] (B2) {$\bm{B_2} \cup \bm{C_2}$};

\draw[-] (A) -- (A1);
\draw[-] (A) -- (A2);
\draw[-] (B) -- (B1);
\draw[-] (B) -- (B2);

\end{tikzpicture}

%% file: figs/compatibility/c4.tex
\begin{tikzpicture}

\node[] (A) {$\bm{A} \cup \bm{C}$};
\node[below= of A, xshift=-0.75cm] (A1) {$\bm{A_1} \cup \bm{C_1}$};
\node[below= of A, xshift=0.75cm] (A2) {$\bm{A_2} \cup \bm{C_2}$};
\node[xshift = 3.5cm] (B) {$\bm{B} \cup \bm{C}$};
\node[below= of B, xshift=-0.75cm] (B1) {$\bm{B_1} \cup \bm{C_1}$};
\node[below= of B, xshift=0.75cm] (B2) {$\bm{B_2} \cup \bm{C_2}$};

\draw[-] (A) -- (A1);
\draw[-] (A) -- (A2);
\draw[-] (B) -- (B1);
\draw[-] (B) -- (B2);

\end{tikzpicture}

%% file: figs/product_md_diagram.tex
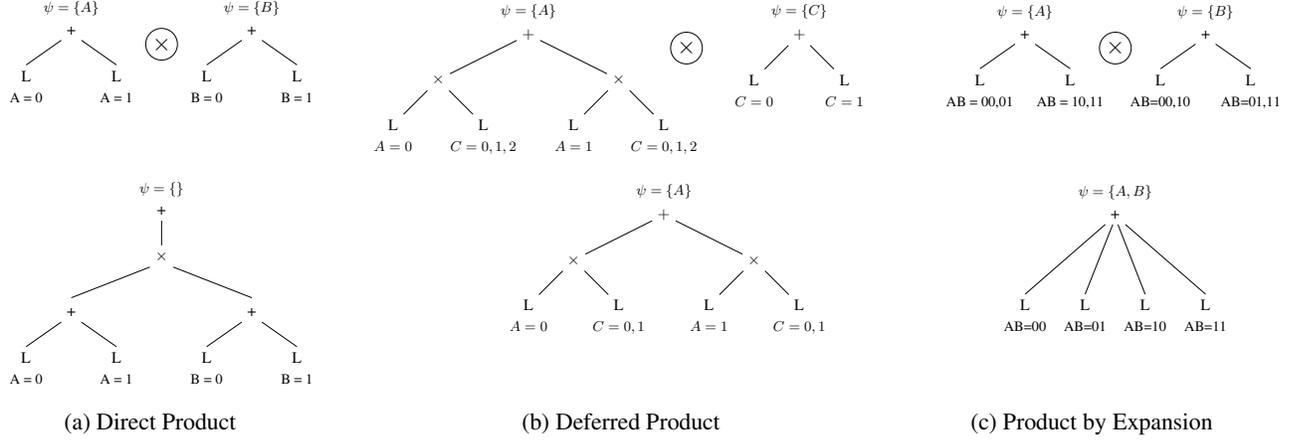
\begin{figure}
    \centering
    \begin{subfigure}[t]{0.24\linewidth}
        \centering
        \scalebox{0.6}{
        \input{figs/product_md/direct_product}
        }
        \caption{Direct Product}
        \label{fig:pmd_1}
    \end{subfigure}
    \begin{subfigure}[t]{0.48\linewidth}
        \centering
        \scalebox{0.6}{
        \input{figs/product_md/stuff}
        }
        \caption{Deferred Product}
        \label{fig:pmd_2}
        \begin{minipage}{.1cm}
            \vfill
            \end{minipage}
    \end{subfigure}
    \begin{subfigure}[t]{0.24\linewidth}
        \centering
        \scalebox{0.6}{
        \input{figs/product_md/match}
        }
        \caption{Product by Expansion}
        \label{fig:pmd_3}
        \begin{minipage}{.1cm}
            \vfill
            \end{minipage}
    \end{subfigure}
    \caption{Examples of product of the two sum nodes on the top half, with the result shown in the bottom half. The root sum node is labelled with the corresponding vtree node label, while the leaves are labelled with their support.}
    \label{fig:product_md}
\end{figure}

%% file: figs/product_md/direct_product.tex
\begin{tikzpicture}
  \node[label=above:{$\vlabel = \{A\}$}] (plus1) at (0,0) {+};
  \node[label=below:{A = 0}] (A0) at (-1,-1) {L};
  \node[label=below:{A = 1}] (A1) at (1,-1) {L};
  \draw (plus1.south west) -- (A0.north);
  \draw (plus1.south east) -- (A1.north);

\node[latent] (C3) at (2, -0.3) {\scalebox{1.5}{\textbf{$\times$}}};

  \node[label=above:{$\vlabel = \{B\}$}] (plus2) at (4,0) {+};
  \node[label=below:{B = 0}] (B0) at (3,-1) {L};
  \node[label=below:{B = 1}] (B1) at (5,-1) {L};
  \draw (plus2.south west) -- (B0.north);
  \draw (plus2.south east) -- (B1.north);

  \node[label=above:{$\vlabel = \{\}$}] (plus3) at (2,-4) {+};
  \node[] (X) at (2,-5) {$\times$};
  \node (left) at (0,-7) {\begin{tikzpicture}
    \node[] (plus1) at (0,0) {+};
    \node[label=below:{A = 0}] (A0) at (-1,-1) {L};
    \node[label=below:{A = 1}] (A1) at (1,-1) {L};
    \draw (plus1.south west) -- (A0.north);
    \draw (plus1.south east) -- (A1.north);
  \end{tikzpicture}};
  \node (right) at (4,-7) {\begin{tikzpicture}
    \node[] (plus2) at (0,0) {+};
    \node[label=below:{B = 0}] (B0) at (-1,-1) {L};
    \node[label=below:{B = 1}] (B1) at (1,-1) {L};
    \draw (plus2.south west) -- (B0.north);
    \draw (plus2.south east) -- (B1.north);
  \end{tikzpicture}};
  \draw (plus3.south) -- (X.north);
  \draw (X.south west) -- (left.north);
  \draw (X.south east) -- (right.north);

\end{tikzpicture}

%% file: figs/product_md/stuff.tex
\begin{tikzpicture}
  \node[label=above:{$\vlabel=\{A\}$}] (A) at (0,0) {$+$};
  \node (X1) at (-2,-1) {$\times$};
  \node (X2) at (2,-1) {$\times$};
  \node[label=below:{$A=0$}] (Y1a) at (-3,-2) {L};
  \node[label=below:{$C=0,1,2$}] (Y1b) at (-1,-2) {L};
  \node[label=below:{$A=1$}] (Y2a) at (1,-2) {L};
  \node[label=below:{$C=0,1,2$}] (Y2b) at (3,-2) {L};

  \node[latent] (C3) at (3.5, -0.3) {\scalebox{1.5}{\textbf{$\times$}}};
  
  \draw (A) -- (X1) node[midway, left] {};
  \draw (A) -- (X2) node[midway, right] {};
  \draw (X1) -- (Y1a) node[midway, left] {};
  \draw (X1) -- (Y1b) node[midway, right] {};
  \draw (X2) -- (Y2a) node[midway, left] {};
  \draw (X2) -- (Y2b) node[midway, right] {};
  
  \node[label=above:{$\vlabel=\{C\}$}] (C) at (6,0) {$+$};
  \node[label=below:{$C=0$}] (L1) at (5,-1) {L};
  \node[label=below:{$C=1$}] (L2) at (7,-1) {L};
  
  \draw (C) -- (L1) node[midway, left] {};
  \draw (C) -- (L2) node[midway, right] {};

  \node[label=above:{$\vlabel=\{A\}$}] (Ap) at (3,-4) {$+$};
  \node (X1p) at (1,-5) {$\times$};
  \node (X2p) at (5,-5) {$\times$};
  \node[label=below:{$A=0$}] (Y1ap) at (0,-6) {L};
  \node[label=below:{$C=0,1$}] (Y1bp) at (2,-6) {L};
  \node[label=below:{$A=1$}] (Y2ap) at (4,-6) {L};
  \node[label=below:{$C=0,1$}] (Y2bp) at (6,-6) {L};
  
  \draw (Ap) -- (X1p) node[midway, left] {};
  \draw (Ap) -- (X2p) node[midway, right] {};
  \draw (X1p) -- (Y1ap) node[midway, left] {};
  \draw (X1p) -- (Y1bp) node[midway, right] {};
  \draw (X2p) -- (Y2ap) node[midway, left] {};
  \draw (X2p) -- (Y2bp) node[midway, right] {};
  
  \node[] (dummy) at (0, -7.75) {\vphantom{L}};
\end{tikzpicture}

%% file: figs/product_md/match.tex
\begin{tikzpicture}
  \node[label=above:{$\vlabel = \{A\}$}] (plus1) at (0,0) {+};
  \node[label=below:{AB = 00,01}] (A0) at (-1,-1) {L};
  \node[label=below:{AB = 10,11}] (A1) at (1,-1) {L};
  \draw (plus1.south west) -- (A0.north);
  \draw (plus1.south east) -- (A1.north);

  \node[latent] (C3) at (2, -0.3) {\scalebox{1.5}{\textbf{$\times$}}};

  \node[label=above:{$\vlabel = \{B\}$}] (plus2) at (4,0) {+};
  \node[label=below:{AB=00,10}] (B0) at (3,-1) {L};
  \node[label=below:{AB=01,11}] (B1) at (5,-1) {L};
  \draw (plus2.south west) -- (B0.north);
  \draw (plus2.south east) -- (B1.north);

  \node[label=above:{$\vlabel = \{A,B\}$}] (plus3) at (2,-4) {+};
  \node[label=below:{AB=00}] (L1) at (0,-6) {L};
  \node[label=below:{AB=01}] (L2) at (1.33,-6) {L};
  \node[label=below:{AB=10}] (L3) at (2.66,-6) {L};
  \node[label=below:{AB=11}] (L4) at (4,-6) {L};
  \draw (plus3.south west) -- (L1.north);
  \draw (plus3.255) -- (L2.north);
  \draw (plus3.285) -- (L3.north);
  \draw (plus3.south east) -- (L4.north);

    \node[] (dummy) at (0, -7.75) {\vphantom{L}};

\end{tikzpicture}

%% file: algs/prod.tex
\begin{algorithm}[t]
\SetAlgoLined
\KwInput{Input circuits $\circuit^{(1)} = (\vtree^{(1)} = (\vnodes^{(1)}, \vedges^{(1)}, \scope^{(1)}), \vlabel^{(1)}, \vparamfn^{(1)}), \circuit^{(2)} = (\vtree^{(2)} = (\vnodes^{(2)}, \vedges^{(2)}, \scope^{(2)}), \vlabel^{(2)}, \vparamfn^{(2)})$
}
\KwResult{Output circuit $\circuit' = (\vtree', \vlabel', \vparamfn')$}

$\vnode^{(1)}, \vnode^{(2)} \gets \texttt{root}(\vtree^{(1)}), \texttt{root}(\vtree^{(2)})$;

$(\vnode^{(1)}_l, \vnode^{(1)}_r), (\vnode^{(2)}_l, \vnode^{(2)}_r) \gets \texttt{children}(\vnode^{(1)}), \texttt{children}(\vnode^{(2)})$; \tcp*{null if $\vnode^{(1)}$/$\vnode^{(2)}$ are leaf}

$\vnode' \gets \texttt{newvtreenode}()$; 


$\newcommonvars \gets \scope^{(1)}(\vnode^{(1)}) \cap \scope^{(2)}(\vnode^{(2)})$;

\If{$\newcommonvars = \emptyset$} {
    $\vtree'_l, \vlabel'_l, \vparamfn'_l \gets \vtree^{(1)}, \vlabel^{(1)}, \vparamfn^{(1)}$;
    
    $\vtree'_r, \vlabel'_r, \vparamfn'_r \gets \vtree^{(2)}, \vlabel^{(2)}, \vparamfn^{(2)}$;
 

    $\vtree', \vlabel', \vparamfn' \gets \vtree'_l \cup \vtree'_r, \vlabel'_l \cup \vlabel'_r, \vparamfn'_l \cup \vparamfn'_r$; \tcp*{combine the vtrees/labelling fn/param fn}


    $\vlabel'(\vnode') \gets \emptyset$; \tcp*{Update label function}
    
    $\sumunits_{\vnode'} \gets \left(\sumunit_{\vnode', i^{(1)}i^{(2)}} \textbf{ for } i^{(1)} = 1, ..., |\sumunits_{\vnode^{(1)}}|, i^{(2)} = 1, ... |\sumunits_{\vnode^{(2)}}|\right)$;

    $\pcparam_{\vnode', i^{(1)}i^{(2)}jk} \gets \mathds{1}_{i^{(1)}=j, i^{(2)}=k}$;

    $\vparamfn'(\vnode') \gets (\sumunits_{\vnode'}, \pcparam_{\vnode'})$ \tcp*{Update parameter function}

    
        
}
\ElseIf{$\vnode^{(1)}$ and $\vnode^{(2)}$ are leaves} {
    $\vtree', \vlabel', \vparamfn' \gets \texttt{createemptyvtree()}$;

    $\vlabel'(\vnode') = \vlabel^{(1)}(\vnode^{(1)}) \cup \vlabel^{(2)}(\vnode^{(2)})$;

    $\sumunits_{\vnode'} \gets \left(\sumunit_{\vnode', i^{(1)}i^{(2)}} \textbf{ for } i^{(1)} = 1, ..., |\sumunits_{\vnode^{(1)}}|, i^{(2)} = 1, ... |\sumunits_{\vnode^{(2)}}|\right)$;

    $\pcparam_{\vnode', i^{(1)}i^{(2)}j^{(1)}j^{(2)}k^{(1)}k^{(2)}} \gets \pcparam_{\vnode^{(1)}, i^{(1)}j^{(1)}k^{(1)}} \pcparam_{\vnode^{(1)}, i^{(2)}j^{(2)}k^{(2)}}$;

    $\vparamfn'(\vnode') \gets (\sumunits_{\vnode'}, \pcparam_{\vnode'})$ \tcp*{Update parameter function}
}
\ElseIf{$\scope^{(1)}_{\commonvars}(\vnode^{(2)}) = \scope^{(1)}_{\commonvars}(\vnode^{(1)}_r)$}{
    $\vtree'_l, \vlabel'_l, \vparamfn'_l \gets \vtree^{(1)}_{\vnode^{(1)}_l}, \vlabel^{(1)}, \vparamfn^{(1)}$;
    
    $\vtree'_r, \vlabel'_r, \vparamfn'_r \gets \product(\vnode^{(1)}_r, \vnode^{(2)})$;

    $\vtree', \vlabel', \vparamfn' \gets \vtree'_l \cup \vtree'_r, \vlabel'_l \cup \vlabel'_r, \vparamfn'_l \cup \vparamfn'_r$; \tcp*{combine the vtrees/labelling fn/param fn}

    $\vlabel'(\vnode') = \vlabel^{(1)}(\vnode^{(1)})$; \tcp*{Update label function}
 
    $\sumunits_{\vnode'} \gets \left(\sumunit_{\vnode', i^{(1)}i^{(2)}} \textbf{ for } i^{(1)} = 1, ..., |\sumunits_{\vnode^{(1)}}|, i^{(2)} = 1, ... |\sumunits_{\vnode^{(2)}}|\right)$;

    $\pcparam_{\vnode', i^{(1)}i^{(2)}j^{(1)}k^{(1)}k^{(2)}} \gets \theta_{\vnode^{(1)}, i^{(1)}j^{(1)}k^{(1)}} \mathds{1}_{i^{(2)} = k^{(2)}}$;

    $\vparamfn'(\vnode') \gets (\sumunits_{\vnode'}, \pcparam_{\vnode'})$  \tcp*{Update parameter function}



    
    }
\ElseIf{$\scope^{(1)}_{\newcommonvars}(\vnode^{(1)}_l) = \scope^{(2)}_{\newcommonvars}(\vnode^{(2)}_l)$ and $\scope^{(1)}_{\newcommonvars}(\vnode^{(2)}_r) = \scope^{(1)}_{\newcommonvars}(\vnode^{(2)}_r)$} {
    $\vtree'_l, \vlabel'_l, \vparamfn'_l \gets \product(\vnode^{(1)}_l, \vnode^{(2)}_l)$;
    
    $\vtree'_r, \vlabel'_r, \vparamfn'_r \gets \product(\vnode^{(1)}_r, \vnode^{(2)}_r)$;

    $\vtree', \vlabel', \vparamfn' \gets \vtree'_l \cup \vtree'_r, \vlabel'_l \cup \vlabel'_r, \vparamfn'_l \cup \vparamfn'_r$; \tcp*{combine the vtrees/labelling fn/param fn}

    $\vlabel'(\vnode') = \vlabel^{(1)}(\vnode^{(1)}) \cup \vlabel^{(2)}(\vnode^{(2)})$;

    $\sumunits_{\vnode'} \gets \left(\sumunit_{\vnode', i^{(1)}i^{(2)}} \textbf{ for } i^{(1)} = 1, ..., |\sumunits_{\vnode^{(1)}}|, i^{(2)} = 1, ... |\sumunits_{\vnode^{(2)}}|\right)$;

    $\pcparam_{\vnode', i^{(1)}i^{(2)}j^{(1)}j^{(2)}k^{(1)}k^{(2)}} \gets \pcparam_{\vnode^{(1)}, i^{(1)}j^{(1)}k^{(1)}} \pcparam_{\vnode^{(1)}, i^{(2)}j^{(2)}k^{(2)}}$;

    $\vparamfn'(\vnode') \gets (\sumunits_{\vnode'}, \pcparam_{\vnode'})$ \tcp*{Update parameter function}
}
    \Else {
        \textbf{Return} fail (not compatible)
    }


$\scope'(\vnode') \gets \scope^{(1)}(\vnode^{(1)}) \cup \scope^{(2)}(\vnode^{(2)})$; \tcp*{Update scope function}

$\vtree' \gets \texttt{addnode}(\vtree'; \vnode')$; 

$\vtree' \gets \texttt{addchildren}(\vtree'; \vnode', \texttt{root}(\vtree'_l), \texttt{root}(\vtree'_r))$;

\textbf{Return} $(\vtree', \vlabel', \vparamfn')$

\caption{$\product(\circuit^{(1)}, \circuit^{(2)})$}
\label{alg:prod}
\end{algorithm}

%% file: algs/pow.tex






    
        
    
    
    
    



\begin{algorithm}[t]
\SetAlgoLined
\KwInput{Input circuit $\circuit = (\vtree = (\vnodes, \vedges, \scope), \vlabel, \vparamfn)$; power $\power$
}
\KwResult{Output circuit $\circuit' = (\vtree', \vlabel', \vparamfn')$}

$\vnode \gets \texttt{root}(\vtree)$;

$\vnode' \gets \texttt{newnode}()$;
 
\If(\tcp*[f]{Update vtree structure and parameter function (leaf)}){$\vnode$ is leaf} {  \label{algline:margleafstart}
    $\vtree' \gets \texttt{createvtree}(\vnode')$; \tcp*{create vtree with single node}


    $\vparamfn'(\vnode') \gets (\powop(\leafunit; \power) \textbf{ for } \leafunit \in \sumunits_{\vnode}, \pcparam_{\vnode})$; \label{algline:margleafend} \tcp*{apply power to leaf PC nodes}
    
        
}
\Else(\tcp*[f]{Update vtree structure and parameter function (non-leaf)}){ \label{algline:margnonleafstart}
    $\vnode_l, \vnode_r \gets \texttt{children}(\vnode)$; 
    
    $\vtree'_l, \vlabel'_l, \vparamfn'_l \gets \powop((\vtree_{\vnode_l}, \vlabel, \vparamfn), \power)$; 
    
    $\vtree'_r, \vlabel'_r, \vparamfn'_r \gets \powop((\vtree_{\vnode_r}, \vlabel, \vparamfn), \power)$; 

    $\vtree', \vlabel', \vparamfn' \gets \vtree'_l \cup \vtree'_r, \vlabel'_l \cup \vlabel'_r, \vparamfn'_l \cup \vparamfn'_r$; \tcp*{combine the vtrees/labelling fn/param fn}

    $\vtree' \gets \texttt{addnode}(\vtree'; \vnode'); \vtree' \gets \texttt{addchildren}(\vtree'; \vnode', \texttt{root}(\vtree'_l), \texttt{root}(\vtree'_r))$;
    
    
    $\vparamfn'(\vnode') \gets \vparamfn(\vnode)$; \label{algline:margnonleafend}
}

$\scope'(\vnode') \gets \scope(\vnode)$;  \tcp*{Update scope function} \label{algline:margscope}

$\vlabel'(\vnode') \gets \vlabel(\vnode)$; \tcp*{Update labelling function}

\textbf{Return} $(\vtree', \vlabel', \vparamfn')$

\caption{$\powop(\circuit, \power)$}
\label{alg:powop}
\end{algorithm}

%% file: sections/appendix/causal.tex
\section{Causal Inference} \label{apx:causal_inference}

In this section, we provide further background on causal inference for interested readers, the proof of our hardness result for backdoor adjustment, and then provide the full derivations of our tractability results for the backdoor, frontdoor and extended napkin formulae.

\subsection{Background on Causal Inference}

We use the framework of structural causal models (SCMs) \citepSM{pearl09causality} to define the task of causal inference. SCMs provide a formal model of the underlying reality of a data-generating system over variables $\vars$. In particular, each variable has an associated causal mechanism, which is a deterministic function of other variables in $\vars$ and a set of \textit{exogenous variables} $\exovars$. The exogenous variables represent the ``external state of the world'', or in other words, the source of randomness in generated data.

\begin{definition}[Structural Causal Model]
A SCM $\scm$ is a tuple $(\exovars, \vars, \scmfns, p(\exovars))$ where:
\begin{itemize}
    \item $\exovars$ is a set of exogenous (i.e. outside the model) random variables, which are typically unobserved; 
    \item $\vars = \{\var_1, .., \var_{\dm}\}$ is a set of endogenous (i.e. inside the model) random variables;
    \item $\scmfns = \{\scmfn_1, ..., \scmfn_{\dm}\}$ is a set of causal mechanisms (functions). In particular, each endogenous variable $\var_i$ has an associated function $\scmfn_i$ which is a mapping from (the domain of) some subset $\parents(\var_i) \subseteq \exovars \cup (\vars \setminus \{\var\})$ to (the domain of) $\var_i$. Members of $\parents(\var_i)$ are referred to as parents of $\var$, and can be split into exogenous parents $\exoparents_i$ and endogenous parents $\edgparents_i$.
    \item $\dist(\exovars)$ is a probability distribution over the exogenous variables.
\end{itemize}
\end{definition}

\begin{definition}[Semi-Markovian SCM]
    A SCM is said to be semi-Markovian if the causal diagram induced by the SCM does not contain directed cycles. In such a case, for each variable $\var \in \vars$ we write $\var(\exovarsval)$ to denote the unique value of $\var$ given a particular value $\exovarsval$ of the exogenous variables $\exovars$.
\end{definition}

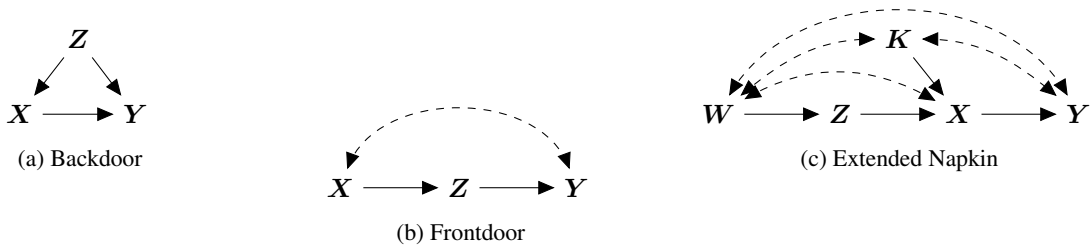
\begin{figure}
\begin{subfigure}[b]{0.25\linewidth}
\centering
\input{diags/backdoor_graph_compact}
\caption{Backdoor}
\label{fig:backdoor_apx}
\end{subfigure}
\begin{subfigure}{0.33\textwidth}
        \centering
        \input{diags/frontdoor_graph_compact}
        \caption{Frontdoor}
        \label{fig:frontdoor_apx}
    \end{subfigure}
\begin{subfigure}[b]{0.34\linewidth}
\centering
\input{diags/extended_napkin}
\caption{Extended Napkin}
\label{fig:extnapkin_apx}
\end{subfigure}
\caption{Examples of causal diagrams}
\label{fig:causal_diagrams_apx}
\end{figure}

For semi-Markovian SCMs, we define a causal graph/diagram $\graph$ over the endogenous variables $\edg$, as follows. For each variable $\edgm \in \edg$, we add a directed edge to this variable from any endogenous parents of $\edgm$. Then, for each exogenous variable which is a parent of two endogenous variables $\edgm_1, \edgm_2$, we add a bidirected edge between the variables. Intuitively, directed edges represent direct/functional relationships between observed variables, while bidirected arrows represent unobserved confounding. Examples of causal diagrams are shown in Figure \ref{fig:causal_diagrams_apx}.

A semi-Markovian SCM naturally induces a distribution $p(\edg)$ on the observed variables, through the distribution on exogenous variables $\dist(\exovars)$. However, we can also use it to define the semantics of interventions, or actions on a system:

\begin{definition}[Intervention]
Given a SCM $\scm$, and any subset $\treat \subseteq \edg$, and an instantiation $\treatval$ of $\treat$, we can define an intervention on $\scm$ as producing a submodel $\scm_{\treatval}$. $\scm_{\treatval}$ differs from $\scm$ in that the functional relationships $\scmfn_{\treat}$ are replaced with setting each variable to the fixed value $\treatval$, i.e. $\scmfn_{\treat} = \treatval$.

The interventional distribution $p_{\treatval}(\edg)$ (also written $p(\edg|do(\treatval))$) is defined to be the distribution of $\edg$ in the intervened model $\scm_{\treatval}$.
\end{definition}

While SCMs are powerful functional models, we rarely have access to the true underlying SCM for a system. Thus, we often make the much weaker assumption of knowledge of the causal graph encoding the \textit{qualitative} functional dependencies between the endogenous variables. In particular, an important question is, given just a causal graph $\graph$, and observational data $\dist(\pcvars)$ on the variables $\edg$, under what circumstances can we deduce (properties of) interventional distributions $\dist_{\treatval}(\vars)$? More formally, we say that an interventional distribution (query) $p_{\treatval}(\outcome)$ (for some $\outcome \subseteq \edg$) is \emph{identfiable} with respect to the causal graph $\graph$, if $p_{\treatval}(\outcome)$ is uniquely computable from $p(\edg)$ for any SCM $\scm$ that induces $\graph$. 

The problem of causal identification can be solved constructively using the do-calculus \citepSM{pearl1995do}, which comprises a set of rules for transforming a given interventional expression $\dist_{\treatval}(\outcome)$ into a function the observational distribution $\dist(\vars)$, given the causal graph $\graph$. In particular, it was later shown that the do-calculus is \textit{complete} \citepSM{shpitser2006complete}, i.e., if a query is identifiable, then the do-calculus can derive a formula (estimand), and there exists a polynomial time algorithm for finding such a formula \citepSM{shpitser2008complete}. For example, for the backdoor causal graph (Figure \ref{fig:backdoor}), we can obtain a formula $\dist_{\treatval}(\outcome) = \sum_{\adjust} \dist(\adjust) \dist(\outcome|\treatval, \adjust)$.

\subsection{Hardness of Causal Inference for Probabilistic Circuit Models}

Now suppose we have a model $\dist(\pcvars)$ of the observational distribution, perhaps learned from data, and we would like to compute $\dist_{\treatval}(\outcome)$ in the identifiable backdoor case. Unfortunately, in high dimensions, na\"ive computation of the do-calculus formula $\dist_{\treatval}(\outcome) = \sum_{\adjust} \dist(\adjust) \dist(\outcome|\treatval, \adjust)$ is computationally intractable, as it involves a summation that is exponential in the dimension $|\adjust|$. The natural question is then whether there exist probabilistic models in which we can compute the backdoor query more efficiently. Unfortunately, despite the tractability of PCs for most probabilistic inference tasks, we show that, if the observed data distribution is modelled by a probabilistic circuit, current structural and support properties are not sufficient for exact causal inference:

\backdoorhard*

\begin{proof}

We prove this in the case of binary variables for brevity of presentation, though the proof can be easily extended to non-binary discrete variables. Our proof is based on a reduction from the problem of computing the expectation of a logistic regression model, which was defined and shown to be \textsc{\#P}-hard in \citetSM{vandenbroeck2022explr} and which we refer to as the \texttt{EXPLR} problem. In particular, for any \texttt{EXPLR} problem over variables $\adjust$, with input size $n_{\adjust} = |\adjust|$, we construct a circuit in time and with size linear in $\adjust$ and where computing the backdoor query $\dist_{\treatval}(\outcomeval) = \sum_{\adjust} \dist(\adjust) \dist(\outcomeval|\treatval, \adjust)$ is equivalent to solving the \texttt{EXPLR} problem.

The \texttt{EXPLR} problem is defined as computing the following quantity (where $w_i \in \mathbb{R}$):
\begin{equation}
    \texttt{EXPLR}(\bm{w}) = \sum_{\adjustval} \frac{1}{1 + e^{-(w_0 + \sum_i w_i z_i)}}
\end{equation}

We will construct a circuit over variables $\pcvars = \{\treat, \outcome, \adjust\}$, where the sets $\treat = \{\treatm\}$ and $\outcome = \{\outcomem\}$ each consist of a single variable, and consider the backdoor query for instantiations $\treatmval, \outcomemval$ of $\treatm, \outcomem$. We begin by defining a number of auxiliary circuits/nodes for $\treat, \outcome$ and $\adjust$ individually, all structured decomposable, smooth and deterministic, which will be part of the construction of the main circuit. 

First, for $\outcome$ we define the leaf nodes $\mathds{1}_{\outcomemval}, \mathds{1}_{\neg \outcomemval}$ to encode the functions $\dist_{\mathds{1}_{\outcomemval}}(\outcomem) := \mathds{1}(\outcomem = \outcomemval), \dist_{\mathds{1}_{\neg \outcomemval}}(\outcomem) := \mathds{1}(\outcomem = \neg \outcomemval)$ respectively. For $\treat$, we define $\mathds{1}_{\treatmval}, \mathds{1}_{\neg \treatmval}$ to encode $\dist_{\mathds{1}_{\treatmval}}(\treatm) := \mathds{1}(\treatm = \treatmval), \dist_{\mathds{1}_{\neg \treatmval}}(\treatm) := \mathds{1}(\treatm = \neg \treatmval)$ (respectively) in a similar manner. Finally, for $\adjust$, we define two circuits, $\mathds{1}_{\adjust}$ and $\circuit_{\adjust}$, as follows. 
Let $\adjust := \{\adjustm_1, ..., \adjustm_{n_{\adjust}}\}$ be an arbitrary ordering of the variables in $\adjust$, and let $\adjust_{\geq i}$ denote $\{\treatm_i, ..., \treatm_{n_{\adjust}}\}$ for any $1 \leq i \leq n_{\adjust}$. Then we define the circuit $\mathds{1}_{\adjust}$ recursively as follows, where $\mathds{1}_{\adjust} := \mathds{1}_{\adjust_{\geq 1}}$ (where $\times$ indicates a product node with its arguments as children):
\begin{equation}
    \mathds{1}_{\adjust_{\geq i}} := 
    \begin{cases}
        \mathds{1}_{\adjustm_{i}} \times \mathds{1}_{\adjust_{\geq i+1}} & 1 \leq i < n_{\adjust}\\
        \mathds{1}_{\adjustm_{i}} & i = n_{\adjust} 
    \end{cases}
\end{equation}
This circuit consists of a series of product units, and leaf units $\mathds{1}_{\adjustm_i}$ for each $\adjustm_i \in \adjustm$ which we define to encode the function $\dist_{\mathds{1}_{\adjustm_i}}(\adjustm_i) \equiv 1$ (for all values of $\adjustm_i$). Thus, the circuit as a whole encodes $\dist_{\mathds{1}_{\adjust}}(\adjust) \equiv 1$ for all values of $\adjust$. In terms of structural and support properties, the circuit is trivially deterministic and smooth as it does not contain any sum nodes, and is clearly also structured decomposable. Finally, it can also be seen that the size $|\mathds{1}_{\adjust}|$ (number of edges) of $\mathds{1}_{\adjust}$ is $O(n_{\adjust})$.

We now design a circuit $\circuit_{\adjust}$ to encode the function $e^{-(w_0 + \sum_i w_i z_i)}$ as follows, where $\circuit_{\adjust} := \circuit_{\adjust_{\geq 1}}$:
\begin{equation}
    \circuit_{\adjust_{\geq i}} := 
    \begin{cases}
        \circuit_{\adjustm_i} \times \circuit_{\adjust_{\geq i + 1}} & 1 \leq i < n_{\adjust}\\
        \circuit_{\adjustm_i}  & i = n_{\adjust} 
    \end{cases}
\end{equation}
where we define leaf nodes $\circuit_{\adjustm_i}$ to encode $\dist_{\circuit_{\adjustm_i}}(\adjustm_i) := e^{-w_i \adjustm_i}$ for $1 \leq i < n_{\adjust}$ and $\dist_{\circuit_{\adjustm_i}}(\adjustm_i) := e^{-(w_0 + w_i \adjustm_i)}$ for $i = n_{\adjust}$. By recursion it can be seen that this circuit does indeed encode $\dist_{\circuit_{\adjust}}(\adjust) = e^{-(w_0 + \sum_i w_i z_i)}$. This circuit is deterministic and smooth, and also decomposes in the same way as $\mathds{1}_{\adjust}$, i.e. they are structured decomposable with the same vtree. It can also be seen that the size $|\circuit_{\adjust}|$ of $\circuit_{\adjust}$ is $O(n_{\adjust})$.

Now, consider the following probabilistic circuit over $\pcvars = \treat \cup \outcome \cup \adjust$ (where $\times, +$ represent product, sum nodes respectively):
\begin{equation} \label{eqn:hardcircuitdefn}
\circuit := \mathds{1}_{\outcomemval}\times (\mathds{1}_{\treatmval} \times \mathds{1}_{\adjust} + \mathds{1}_{\neg \treatmval}  \times \mathds{1}_{\adjust}) +  \mathds{1}_{\neg \outcomemval} \times (\mathds{1}_{\treatmval} \times \circuit_{\adjust} +  \mathds{1}_{\neg \treatmval}  \times \mathds{1}_{\adjust})
\end{equation}
$\circuit$ is structured decomposable as all of the product units with the same scope in the equation above decompose in the same way, and we have seen that $\mathds{1}_{\adjust}$ and $\circuit_{\adjust}$ are structured decomposable with respect to the same vtree. It is also smooth and deterministic as the individual circuits $\mathds{1}_{\adjust}$ and $\circuit_{\adjust}$ are smooth and deterministic, and the sum nodes in the equation satisfy determinism by the fact that $(\mathds{1}_{\outcomemval}, \mathds{1}_{\neg \outcomemval})$ and $(\mathds{1}_{\treatmval}, \mathds{1}_{\neg \treatmval})$ have disjoint support. Finally, as the sizes of $\mathds{1}_{\adjust}$ and $\circuit_{\adjust}$ are $O(n_{\adjust})$, $|\circuit(\pcvars))|$ is also $O(n_{\adjust})$.

Now, we show that the backdoor query on $\circuit$ is equivalent to solving the corresponding $\texttt{EXPLR}$ problem. First, we derive expressions for all of the individual components of the backdoor formula on the circuit $\circuit$, by evaluating according to Equation \ref{eqn:hardcircuitdefn}:
\begin{align*}
\dist_{\circuit}(\treatmval, \outcomemval, \adjustval) &= \dist_{\mathds{1}_{\outcomemval}}(\outcomemval) \times (\dist_{\mathds{1}_{\treatmval}}(\treatmval) \times \dist_{\mathds{1}_{\adjust}}(\adjustval) + \dist_{\mathds{1}_{\neg \treatmval}}(\treatmval) \times \dist_{\mathds{1}_{\adjust}}(\adjustval)) + \dist_{\mathds{1}_{\neg \outcomemval}}(\outcomemval) \times (\dist_{\mathds{1}_{\treatmval}}(\treatmval) \times \dist_{\circuit_{\adjust}}(\adjustval) +  \dist_{\mathds{1}_{\neg \treatmval}}(\treatmval)  \times \dist_{\mathds{1}_{\adjust}}(\adjustval))\\ 
&= 1 \times (1 \times 1 + 0 \times 1) +  0 \times (1  \times 1 + 0 \times 1)\\
&= 1 \\ \\
\dist_{\circuit}(\treatmval, \adjustval) 
&= \dist_{\circuit}(\treatmval, \outcomemval, \adjustval) + \dist_{\circuit}(\treatmval, \neg \outcomemval, \adjustval)   \\
&= 1 + \circuit_{\adjust}(\adjustval)\\ \\
\dist_{\circuit}(\adjustval) 
&= \dist_{\circuit}(\treatmval, \adjustval) + \dist_{\circuit}(\neg \treatmval, \outcomemval, \adjustval) + \dist_{\circuit}(\neg \treatmval, \neg \outcomemval, \adjustval) \\
&= \dist_{\circuit}(\treatmval, \adjustval) + 1 + 1 \\
&= \dist_{\circuit}(\treatmval, \adjustval) + 2 \\
\end{align*}
The backdoor query for $\circuit$ can then be expressed as
\begin{align*}
    \dist_{\treatval}(\outcomeval) &= \sum_{\adjustval} \dist_{\circuit}(\adjustval) \dist_{\circuit}(\outcomemval|\treatmval, \adjustval)
    = \sum_{\adjustval}(\dist_{\circuit}(\treatmval, \adjustval) + 2) \frac{\dist_{\circuit}(\treatmval, \outcomemval, \adjustval)}{\dist_{\circuit}(\treatmval, \adjustval)} 
    = \sum_{\adjustval} \left[ 1 + \frac{2}{1 + \dist_{\circuit_{\adjust}}(\adjustval)} \right]  \\
    &= 2^{n_{\adjust}} + 2 \sum_{\adjustval} \frac{1}{1 + e^{-(w_0 + \sum_i w_i z_i)}}
\end{align*}

Thus, if we can compute the backdoor query for $\circuit$, then we can compute the given $\texttt{EXPLR}$ problem, completing the reduction.
\end{proof}

The significance of this result is that it implies hardness of causal inference for structured decomposable and deterministic PCs whenever there is a \textit{valid backdoor adjustment} (whether we use the backdoor formula or not), one of the simplest and most common cases where the causal effect is identifiable.

\begin{restatable}{corollary}{backdoorcor}
For any interventional query $\p_{\treatval}(\outcomeval)$ and causal diagram $\graph$ such that the query is identifiable through a backdoor adjustment, and the observational distribution $\p(\edg)$ given as  a decomposable and smooth circuit $\circuit$ encoding $\dist$, computing $\p_{\treatval}(\outcomeval)$ is \textsc{\#P}-hard, even if the circuit is structured decomposable and deterministic.
\end{restatable}

\begin{proof}
By the identifiability condition, we have that $p_{\treatval}(\outcomeval) = \sum_{\adjustval} p(\outcomeval|\treatval, \adjustval) p(\adjustval) = \sum_{\adjustval} \circuit(\adjustval) \circuit(\outcomeval|\treatval, \adjustval)$, which is the backdoor query. Hardness of computing the causal effect then follows from hardness of backdoor queries for the probabilistic circuit.
\end{proof}

\subsection{MD-calculus and Causal Inference}

As sketched in the main paper, MD-calculus provides the tools for us to analyze what properties we need (to add to structured decomposability and determinism) to enable tractable computation of causal queries. We now provide the full derivations for the backdoor, frontdoor, and extended napkin cases. For convenience, in the following we will refer to the intermediate circuits in a computation by the functions they encode, e.g. $\circuit(\dist_{\circuit}(\treat, \adjust))$ for the circuit obtained from applying the $\marg(\cdot; \pcvars \setminus (\treat \cup \adjust))$ operation to $\circuit$.

\subsubsection{Backdoor}

We begin with the backdoor query. In cases where there is a valid backdoor adjustment set, such as in Figure \ref{fig:backdoor}, we have the following formula for the interventional distribution:
\begin{align*}
     \dist_{\circuit, \treat}(\outcome) = \sum_{\adjust} \dist_{\circuit}(\adjust) \dist_{\circuit}(\outcome|\treat, \adjust) = \sum_{\adjust} \dist_{\circuit}(\adjust) \frac{\dist_{\circuit}(\outcome, \treat, \adjust)}{\dist_{\circuit}(\treat, \adjust)}
\end{align*}

Before continuing, it is worth noting that the expression above is valid for any value of $\treat, \outcome$, though in causal inference it is more typical that we are interested in evaluating $\dist_{\circuit, \treat}(\outcome)$ for specific values $\treatval$ of $\treat$, i.e. a specific intervention, or a small set of interventions. This distinction is important as we will see, interestingly, that instantiating $\treat$ makes the query more tractable in the sense that the marginal determinism requirements for the circuit $\circuit$ are more relaxed.

To apply the MD-calculus, we first identify the deterministic operations in the pipeline. For the backdoor query, the only deterministic operation is the $\powop(\cdot; -1)$ operation, applied to $\circuit(\dist_{\circuit}(\treat, \adjust))$ which requires a deterministic input circuit. Then, we can work \textit{backwards} through the pipeline from $\powop$ in order to derive tractability conditions on $\circuit$. 

\begin{enumerate}
    \item \textbf{Requirement}: $\circuit(\dist_{\circuit}(\treat, \adjust))$ is $(\treat \cup \adjust)$-deterministic.
    \item $\marg(\cdot; \outcome)$: $\circuit(\dist_{\circuit}(\treat, \outcome, \adjust))$ is $(\treat \cup \adjust)$-deterministic $\implies$ $\circuit(\dist_{\circuit}(\treat, \adjust))$ is $(\treat \cup \adjust)$-deterministic
    \item $\marg(\cdot; \pcvars \setminus (\treat \cup \outcome \cup \adjust))$: $\circuit(\dist_{\circuit}(\pcvars))$ is $(\treat \cup \adjust)$-deterministic$\implies$ $\circuit(\dist_{\circuit}(\treat, \outcome, \adjust))$ is  $(\treat \cup \adjust)$-deterministic 
    \item \textbf{Sufficient Condition}: $\circuit = \circuit(\dist_{\circuit}(\pcvars))$ is $(\treat \cup \adjust)$-deterministic.
\end{enumerate}

This simple derivation shows that it suffices for $\circuit$ to be $(\treat \cup \adjust)$-deterministic to compute the backdoor query. Now, let us consider the case in which we instantiate $\treat$ with a value $\treatval$:
\begin{align*}
     \dist_{\circuit, \treatval}(\outcome) = \sum_{\adjust} \dist_{\circuit}(\adjust) \dist_{\circuit}(\outcome|\treatval, \adjust) = \sum_{\adjust} \dist_{\circuit}(\adjust) \frac{\dist_{\circuit}(\outcome, \treatval, \adjust)}{\dist_{\circuit}(\treatval, \adjust)}
\end{align*}

Here, employing the MD-calculus gives us the following:

\begin{enumerate}
    \item \textbf{Requirement}: $\circuit(\dist_{\circuit}(\treatval, \adjust)$) is $\adjust$-deterministic.
    \item $\marg(\cdot; \outcome)$: $\circuit(\dist_{\circuit}(\treatval, \outcome, \adjust))$ is $\adjust$-deterministic $\implies$ $\circuit(\dist_{\circuit}(\treatval, \adjust))$ is $\adjust$-deterministic
    \item $\marg(\cdot; \pcvars \setminus (\treat \cup \outcome \cup \adjust))$: $\circuit(\dist_{\circuit}(\treatval, \pcvars \setminus \treat))$ is $\adjust$-deterministic $\implies$ $\circuit(\dist_{\circuit}(\treatval, \outcome, \adjust))$ is  $\adjust$-deterministic 
    \item $\inst(\cdot; \treatval)$: $\circuit(\dist_{\circuit}(\pcvars))$ is $(\treat' \cup \adjust)$-deterministic for some $\treat' \subseteq \treat$ $\implies$ $\circuit(\dist_{\circuit}(\treatval, \pcvars \setminus \treat))$ is $\adjust$-deterministic
    \item \textbf{Sufficient Condition}: $\circuit$ is $(\treat' \cup \adjust)$-deterministic for some $\treat' \subseteq \treat$.
\end{enumerate}

Notice that in the requirement, due to the instantiation, the input to the $\powop$ operation $\circuit(\dist_{\circuit}(\treatval, \adjust))$ has scope $\adjust$, meaning that we require it to be $\adjust$-deterministic rather than $(\treat \cup \adjust)$-deterministic. In the final step, we use the MD-calculus rule for instantiation. This shows that the instantiated backdoor adjustment is tractable for a wider range of circuits than if we insisted on a circuit encoding $\dist_{\circuit, \treat}(\outcome)$ as a function of $\treat$ (and $\outcome$).

\subsubsection{Frontdoor}

Another common case where the causal effect is identifiable is the frontdoor causal diagram, shown in Figure \ref{fig:frontdoor_apx}. Unlike the backdoor case, there is unobserved confounding of $\treat$ and $\outcome$, represented by the dashed bidirectional arrow. However, the existence of the observed mediator $\adjust$ nonetheless allows for identifiability, via the formula
\begin{equation}
\dist_{\circuit, \treat}(\outcome) = \sum_{\adjust} \dist_{\circuit}(\adjust|\treat) \sum_{\treat'} \dist_{\circuit}(\treat') \dist_{\circuit}(\outcome|\treat', \adjust)
\end{equation}
We now employ the MD-calculus to derive tractability conditions. This time, looking at the conditionals, there are two deterministic ($\powop$) operations, as well as an auxiliary variable $\treat'$ (that is summed out in the end), which has the same joint distribution with $\pcvars \setminus \treat$ as $\treat$. 

\begin{enumerate}
    \item \textbf{Requirement}: $\circuit(\dist_{\circuit}(\treat))$ is $\treat$-deterministic and $\circuit(\dist_{\circuit}(\treat,\adjust))$ is $(\treat \cup \adjust)$-deterministic.
    \item $\marg(\cdot; \pcvars \setminus \treat)$: $\circuit(\dist_{\circuit}(\pcvars))$ is $\treat$-deterministic $\implies$ $\circuit(\dist_{\circuit}(\treat))$ is $\treat$-deterministic.
    \item $\marg(\cdot; \pcvars \setminus (\treat \cup \adjust))$: $\circuit(\dist_{\circuit}(\pcvars))$ is $(\treat \cup \adjust)$-deterministic $\implies$ $\circuit(\dist_{\circuit}(\treat \cup \adjust))$ is $(\treat \cup \adjust)$-deterministic.
    \item \textbf{Sufficient Condition}: $\circuit$ is $\treat$-deterministic and $(\treat \cup \adjust)$-deterministic.
\end{enumerate}

Now, let us consider instantiating $\treat$ with value $\treatval$. In this case, we have:
\begin{equation}
\dist_{\circuit, \treatval}(\outcome) = \sum_{\adjust} \dist_{\circuit}(\adjust|\treatval) \sum_{\treat'} \dist_{\circuit}(\treat') \dist_{\circuit}(\outcome|\treat', \adjust)
\end{equation}
Note that in this case, the conditional $\dist_{\circuit}(\adjust|\treatval)$ does not impose any determinism requirements, since the input to the $\powop$ operation is a scalar $\circuit(\dist_{\circuit}(\treatval))$. However, the requirements for the other conditional remain the same, as $\treat'$ is an auxiliary variable that is not tied to the intervention value $\treatval$. Overall, we can conclude that $\circuit$ being $(\treat \cup \adjust)$-deterministic is sufficient for the instantiated frontdoor formula, which is again weaker than in the non-instantiated case.

\subsubsection{Extended Napkin}

The extended napkin causal graph in Figure \ref{fig:extnapkin} is an extension of the so-called \textit{napkin} causal graph \citepSM{pearl09causality}, which is obtained by removing $\bm{K}$ from the graph.
For this diagram, the do-calculus gives us the following formula for the interventional distribution:
\begin{align*}
    \centering
    \dist_{\circuit, \treat}(\outcome) = \sum_{\bm{K}} &\left(\sum_{\bm{W}, \treat', \outcome'}\dist_{\circuit}(\treat', \outcome'|\bm{K}, \bm{z}, \bm{W}) \dist_{\circuit}(\bm{W}, \bm{K})\right) \frac{\sum_{\bm{W}} \dist_{\circuit}(\treat, \outcome| \bm{K}, \adjustval, \bm{W}) \dist_{\circuit}(\bm{W}, \bm{K})}{\sum_{\bm{W}} \dist_{\circuit}(\treat| \bm{K}, \adjustval, \bm{W}) \dist_{\circuit}(\bm{W}, \bm{K})}
\end{align*}

As described in the main paper, this is a case where we need to instantiate $\treat$ for tractability, leading to the following formula:
\begin{align*}
    \centering
    \dist_{\circuit, \treatval}(\outcome) = \sum_{\bm{K}} &\left(\sum_{\bm{W}, \treat', \outcome'}\dist_{\circuit}(\treat', \outcome'|\bm{K}, \bm{z}, \bm{W}) \dist_{\circuit}(\bm{W}, \bm{K})\right) \frac{\sum_{\bm{W}} \dist_{\circuit}(\treatval, \outcome| \bm{K}, \adjustval, \bm{W}) \dist_{\circuit}(\bm{W}, \bm{K})}{\sum_{\bm{W}} \dist_{\circuit}(\treatval| \bm{K}, \adjustval, \bm{W}) \dist_{\circuit}(\bm{W}, \bm{K})}
\end{align*}

In this formula, there are four deterministic operations. Three of them relate to the conditionals in the formula, while the final one is the $\powop(\cdot; -1)$ operation applied to $\circuit(\sum_{\bm{W}} \dist_{\circuit}(\treatval| \bm{K}, \adjustval, \bm{W}) \dist_{\circuit}(\bm{W}, \bm{K}))$. To derive a sufficient tractability condition on $\circuit$, we work backward through the computation from each of these operations.

We first consider the conditionals in the formula. Through similar reasoning to the instantiated backdoor formula, all of these can be computed efficiently as long as $\circuit$ is $(\bm{K} \cup \bm{W} \cup \adjust')$-deterministic, for some $\adjust' \subseteq \adjust$. 

For the final deterministic operation $\powop(\cdot; -1)$ applied to $\circuit(\sum_{\bm{W}} \dist_{\circuit}(\treatval| \bm{K}, \adjustval, \bm{W}) \dist_{\circuit}(\bm{W}, \bm{K}))$, note that the input circuit has scope $\bm{K}$, so the requirement is that it is $\bm{K}$-deterministic. At this point, we can apply the rules for $\marg$ and $\product$:

\begin{enumerate}
    \item \textbf{Requirement}: $\circuit(\sum_{\bm{W}} \dist_{\circuit}(\treatval| \bm{K}, \adjustval, \bm{W}) \dist_{\circuit}(\bm{W}, \bm{K}))$ is $\bm{K}$-deterministic.
    \item $\marg$: $\circuit(\dist_{\circuit}(\treatval| \bm{K}, \adjustval, \bm{W}) \dist_{\circuit}(\bm{W}, \bm{K}))$ is $\bm{K}$-deterministic $\implies$ $\circuit(\sum_{\bm{W}} \dist_{\circuit}(\treatval| \bm{K}, \adjustval, \bm{W}) \dist_{\circuit}(\bm{W}, \bm{K}))$ is $\bm{K}$-deterministic.
    \item $\product$: $\circuit(\dist_{\circuit}(\treatval| \bm{K}, \adjustval, \bm{W}))$, $\circuit(\dist_{\circuit}(\bm{W}, \bm{K}))$ both $\bm{K}$-deterministic $\implies$ $\circuit(\dist_{\circuit}(\treatval| \bm{K}, \adjustval, \bm{W}) \dist_{\circuit}(\bm{W}, \bm{K}))$ is $\bm{K}$-deterministic, by rule (a), where $\margdetvars = \bm{K}$.
    \item $\marg$: \underline{$\circuit(\dist_{\circuit}(\pcvars))$ is $\bm{K}$-deterministic} $\implies$ $\circuit(\dist_{\circuit}(\bm{W}, \bm{K}))$ is $\bm{K}$-deterministic
    \item $\product$: $\circuit(\dist_{\circuit}(\treatval, \bm{K}, \adjustval, \bm{W})), \circuit(\dist_{\circuit}(\bm{K}, \adjustval, \bm{W})^{-1})$ both $\bm{K}$-deterministic $\implies$ $\circuit(\dist_{\circuit}(\treatval| \bm{K}, \adjustval, \bm{W}))$ is $\bm{K}$-deterministic, by rule (a), where $\margdetvars = \bm{K}$.
    \item $\marg$: $\circuit(\dist_{\circuit}(\treatval, \adjustval, \pcvars \setminus (\treat \cup \adjust)))$ is $\bm{K}$-deterministic $\implies$ $\circuit(\dist_{\circuit}(\treatval, \bm{K}, \adjustval, \bm{W}))$ is $\bm{K}$-deterministic
    \item $\inst$: \underline{$\circuit(\dist_{\circuit}(\pcvars))$ is $(\bm{K} \cup \treat' \cup \adjust')$-deterministic for some $\treat' \subseteq \treat$, $\adjust' \subseteq \adjust$} $\implies$ $\circuit(\dist_{\circuit}(\treatval, \adjustval, \pcvars \setminus (\treat \cup \adjust)))$ is $\bm{K}$-deterministic
    \item $\powop$: $\circuit(\dist_{\circuit}(\bm{K}, \adjustval, \bm{W}))$ is $\bm{K}$-deterministic $\implies$ $\circuit(\dist_{\circuit}(\bm{K}, \adjustval, \bm{W})^{-1})$ is $\bm{K}$-deterministic
    \item $\marg$: $\circuit(\dist_{\circuit}(\adjustval, \pcvars \setminus \adjust))$ is $\bm{K}$-deterministic $\implies$ $\circuit(\dist_{\circuit}(\bm{K}, \adjustval, \bm{W})$ is $\bm{K}$-deterministic
    \item $\inst$: \underline{$\circuit(\dist_{\circuit}(\pcvars))$ is $(\bm{K} \cup \adjust')$-deterministic for some $\adjust' \subseteq \adjust$} $\implies$ $\circuit(\dist_{\circuit}(\adjustval, \pcvars \setminus \adjust))$ is $\bm{K}$-deterministic
    \item \textbf{Sufficient Condition}: $\circuit$ is $\bm{K}$-deterministic
\end{enumerate}

We have underlined the individual conditions on $\circuit$ that have been derived. For the reciprocal to be tractable, we need $\circuit$ to satisfy all of these. However, it can be seen that the first condition implies the other two (by taking $\treat' = \emptyset$, $\adjust' = \emptyset$), giving the condition at the bottom of the derivation. Now, combining with the previous conditions due to the conditional distributions, the overall sufficient condition for tractability of the (instantiated)extended napkin query is that $\circuit$ is $\bm{K}$-deterministic and $(\bm{K} \cup \bm{W} \cup \adjust')$-deterministic, for some $\adjust' \subseteq \adjust$.

%% file: diags/frontdoor_graph_compact.tex
\begin{tikzpicture}

\node[] (X) {$\treat$};
\node[right=of X] (Z) {$\adjust$};
\node[right=of Z] (Y) {$\outcome$};

\edge {X} {Z};
\edge {Z} {Y};

\draw[<->, dashed] (X) to [out=70, in=110] (Y);

\end{tikzpicture}